\documentclass{article}

\usepackage{preprint}
\setcitestyle{numbers,square,comma,sort&compress}
\usepackage{notoccite}

\usepackage[utf8]{inputenc}
\usepackage[T1]{fontenc}
\usepackage{hyperref}
\usepackage{url}
\usepackage{booktabs}
\usepackage{amsfonts}
\usepackage{nicefrac}
\usepackage{microtype}
\usepackage{xcolor}
\usepackage{amsthm}
\usepackage{thmtools}
\usepackage{thm-restate}
\definecolor{linkblue}{HTML}{00006A}

\hypersetup{
  colorlinks=true,
  linkcolor=linkblue,
  citecolor=linkblue,
  urlcolor=linkblue,
  filecolor=linkblue,
}

\usepackage{titletoc}
\usepackage{algorithm}
\usepackage{algorithmic}
\usepackage{subcaption}
\usepackage{colortbl}
\usepackage[normalem]{ulem}

\usepackage{graphicx}

\usepackage{tikz}
\usepackage{amsmath,amssymb}
\usetikzlibrary{arrows.meta,decorations.pathreplacing,calc,positioning,fit,
                 backgrounds,shapes.geometric,shapes.misc}
\usetikzlibrary{patterns,shapes.symbols}
\usepackage{paralist}
\usepackage[inline]{enumitem}
\usepackage{placeins}

\newcommand{\round}[1]{\left\lfloor #1 \right\rceil}

\newcommand{\R}[0]{\mathbb{R}}
\newcommand{\G}[0]{\mathbb{G}}
\DeclareMathOperator*{\argmin}{arg\,min}
\usepackage{thmtools, thm-restate}
\theoremstyle{plain}
\newtheorem{theorem}{Theorem}[section]
\newtheorem{proposition}[theorem]{Proposition}
\newtheorem{lemma}[theorem]{Lemma}
\newtheorem{corollary}[theorem]{Corollary}
\theoremstyle{remark}
\newtheorem{remark}[theorem]{Remark}
\newcommand{\ScaleAlgorithm}{PiSO}

\newcommand{\ScaleAlgorithmUpperFull}{PiSO}
\newcommand{\ScaleAlgorithmShort}{PiSO}
\newcommand{\DataFreeAlgorithm}{OptQuant}

\definecolor{darkred}{HTML}{C00000}

\title{Optimal Post-Training Quantization Scales\par and Where to Find Them}

\author{
  Juan Amboage
  \thanks{Equal contribution}
  \hspace{0.05em}
  \thanks{Current affiliation: AIDOS Lab, University of Fribourg. This work was conducted while employed at AMD.}
  \hspace{-0.05em} \\
  AMD \\
  \texttt{juan.garciaamboage@unifr.ch}
  \And
  Pablo Monteagudo-Lago\footnotemark[1] \\
  AMD \\
  \texttt{Pablo.MonteagudoLago@amd.com}
  \And
  Ian Colbert \\
  AMD \\
  \texttt{Ian.Colbert@amd.com}
  \And
  Giuseppe Franco \\
  AMD \\
  \texttt{Giuseppe.Franco@amd.com}
  \And
  Nicholas Fraser \\
  AMD \\
  \texttt{nicholas.fraser@amd.com}
}

\begin{document}

\maketitle

\begin{abstract}
Post-training quantization (PTQ) compresses large language models by mapping weights to low-bit representations. While much work has focused on how weights are rounded onto the quantization grid, the scaling factor that defines the grid is typically chosen with simple, data-free heuristics. We observe that, under round-to-nearest (RTN) quantization, the grid assignments are piecewise constant in the scale, making the output reconstruction error piecewise quadratic over finitely many intervals, each admitting a closed-form optimization. Building on this insight, we present PiSO (Piecewise Scale Optimization), an algorithm that leverages calibration data to efficiently select channel-wise scales to achieve globally optimal output reconstruction error to arbitrary accuracy. We extend PiSO to group-wise quantization via tractable objectives. We also propose strategies for integrating it with error-correction methods such as GPTQ and Qronos, which greedily adjust weight assignments to compensate for the rounding errors of previously quantized weights. Experiments on Llama and Qwen models across multiple sizes and bit-widths show consistent improvements in perplexity and zero-shot accuracy, both standalone and combined with error correction, with gains increasing as the bit-width narrows. For instance, on Llama-3, PiSO reduces the 2-bit perplexity of Qronos by up to $7\times$, and at 3 bits, PiSO without error correction outperforms GPTQ with absmax scales.
\end{abstract}

\section{Introduction}
Large language models (LLMs) incur significant memory and compute costs that hinder deployment on resource-constrained hardware. Post-training quantization (PTQ) addresses this by mapping weights, and sometimes activations, to low-precision representations without retraining.The target representation is a discrete grid defined by a bit-width, scaling factor, and (optionally) a zero point. Greedy error correction algorithms such as GPTQ~\citep{frantar2022gptq}, and more recently GPTAQ~\citep{li2025gptaq} or Qronos~\citep{zhang2026qronos}, are widely used to minimize accuracy degradation when mapping to such grids. Crucially, these methods assume a fixed grid during error correction and, in practice, select parameters such as scaling factors beforehand.

Common scale-selection rules either map the maximum absolute weight to the extreme grid point, often referred to as the "absmax" grid, or minimize weight reconstruction error over a finite set of candidate scales between zero and absmax \citep{zhang2026qronos,egiazarian2026bridging,sanjeet2026mixquant}. Both approaches are \emph{data-free}: they ignore how input
activations amplify different weight errors at the layer output. Previous work on quantization of convolutional neural networks by \citet{idelbayev2021optimal} introduced an exact, data-free scale optimizer to minimize weight reconstruction error under round-to-nearest (RTN) quantization, which, to the best of our knowledge, has not been adopted in LLM PTQ pipelines.
Recent methods \citep{nair2024cdquant, zhang2025comq, zhang2026beacon} address data-aware scale selection in LLM PTQ through grid search or coupled with coordinate descent procedures, without global-optimality guarantees (see Section~\ref{background}). More broadly,  scale selection for LLMs has received little formal study, particularly relative to weight-to-grid assignment. In this work, we show for the first time that the optimal channel-wise scales for an output reconstruction objective can be computed exactly and efficiently under RTN quantization by exploiting the piecewise structure of the weight-to-grid assignment. We then extend our algorithm to handle group-wise scaling, and propose several levels of integration with error correction algorithms.

\paragraph{Contributions.}
Our main contributions are:
\begin{inparaenum}[(1)]
\item \ScaleAlgorithmUpperFull, an algorithm that efficiently computes \emph{optimal} scales for channel-wise weight quantization with respect to a \emph{data-aware} layer output reconstruction objective under RTN quantization;
\item two tractable objectives for extending PiSO to group-wise quantization, one of which accounts for cross-group interactions; and
\item three principled strategies for integrating PiSO with GPTQ and Qronos that empirically improve the performance of the resulting quantized models.
\end{inparaenum}
We evaluate \ScaleAlgorithm{} on Llama and Qwen models across multiple sizes, bit-widths, and quantization granularities. It consistently improves perplexity and zero-shot accuracy, both standalone and combined with GPTQ or Qronos, with larger gains at lower bit-widths. At 3 bits, the best results on Llama-3 are obtained by combining \ScaleAlgorithm{} with error correction, and even RTN$+$\ScaleAlgorithm{} alone outperforms GPTQ with the absmax grid (Table~\ref{tab:channel_llama_3-bit4-bit}).
Furthermore, our calibration size study (Figures~\ref{fig:cal-sweep},\ref{fig:cal-sweep-appendix}) also points to improved data efficiency: on Llama-3.2-1B with 3-bit channel-wise quantization, RTN\,$+$\,\ScaleAlgorithm{} needed only a few calibration sequences, and GPTQ with \ScaleAlgorithm{} scales outperformed GPTQ with absmax scales using $32\times$ less calibration data.

\section{Background and Related Work}
\label{background}
\paragraph{Notation.}\label{notation}
Throughout the paper, $\mathbf{W}\in\R^{D\times M}$ denotes the weight matrix of a single layer with $D$ input dimensions and $M$ output channels, and $\mathbf{w}_j := \mathbf{W}_{:,j}\in\R^{D}$ is the weight vector of channel~$j$. $\tilde{\mathbf{W}}$ denotes the quantized (and possibly error-corrected) weight matrix. $\mathbf{X}\in\R^{N\times D}$ denotes the input activation matrix of a given layer of the high-precision model, while $\tilde{\mathbf{X}}$ is the input activation matrix produced by the (partially) quantized model.

For finite grid \(\G=\{g_1<\cdots<g_L\}\subset\mathbb R\) and {scale} $s>0$, we define the symmetric quantizer as
\begin{equation}\label{eq:quantizer}
    \mathcal{Q}(x;\,s) \;:=\; s \cdot q(x;\,s), \qquad\text{where}\quad q(x;\,s) := \round{\tfrac{x}{s}}_{\G}
\end{equation}
is the \emph{unscaled grid assignment}.

Here $\round{\cdot}_\G$ denotes the round-to-nearest (RTN) operator on the grid: $\round{y}_\G := \argmin_{g \in \G} |g - y|$, with a fixed deterministic tie-breaking rule.
For signed $B$-bit integers, $\G = \{-2^{B-1},\, -2^{B-1}+1,\, \ldots,\, 2^{B-1}-1\}$.
Channel-wise quantization uses one scale per output channel, whereas group-wise quantization partitions each channel into contiguous groups of weights, each with its own scale. Concretely, given a channel $\mathbf{w} \in \mathbb{R}^{D}$ and a group size $G$, $\mathbf{w}$ is partitioned into $K = \lceil D / G \rceil$ groups $\mathbf{w}_{\{1\}}, \ldots, \mathbf{w}_{\{K\}} \in \R^G$.
The activation matrices are partitioned column-wise according to the same grouping. We write $\mathbf{X}_{\{k\}}$
for the columns of $\mathbf{X}$ associated with group~$k$. Subscripts $\{\leq k\}$ and $\{< k\}$ denote prefix concatenation over groups, i.e., $\mathbf{w}_{\{\leq k\}} := [\mathbf{w}_{\{1\}}, \ldots, \mathbf{w}_{\{k\}}]^\top$
where, unless stated otherwise, concatenation is along the column dimension.

\paragraph{Post-training quantization.}
Following \citet{zhang2026qronos}, we distinguish two stages of the PTQ pipeline: (1) \emph{transform}, and (2) \emph{round}. Transform methods modify weights and/or activations to make them more amenable to quantization (e.g., by applying rotations to reduce outliers~\citep{ashkboos2024quarot,liu2025spinquant}), whereas round methods map (possibly transformed) weights onto a quantization grid, aiming to minimize a layer-wise reconstruction objective such as:
\begin{equation}\label{eq:layer-obj}
    E_{\mathrm{GPTQ}}(\mathbf{\tilde{W}}) = \bigl\|\tilde{\mathbf{X}}\mathbf{W} - \tilde{\mathbf{X}}\mathbf{\tilde{W}}\bigr\|_F^2 \quad\text{or}\quad E(\mathbf{\tilde{W}}) = \bigl\|\mathbf{X}\mathbf{W} - \tilde{\mathbf{X}}\,\mathbf{\tilde{W}}\bigr\|_F^2.
\end{equation}

The second objective explicitly accounts for activation drift introduced by quantizing previous layers.\footnote{The first objective is widely adopted as it is used in GPTQ and many others \citep{nair2024cdquant, zhang2025comq, zhang2026beacon, frantar2022optimal}, while methods like GPFQ \citep{lybrand2021greedy}, GPTAQ and Qronos adopt the second.} \emph{Round} algorithms can be further subdivided based on their optimization approach:

\paragraph{Gradient-based rounding.}These algorithms use backpropagation to learn rounding decisions and, in some cases, quantization parameters~\citep{nagel2020up, li2021brecq, hubara2021accurate, cheng2024optimize, shao2024omniquant}. However, their computational cost has limited their applicability to LLMs until recently~\citep{cheng2024optimize}, they require careful hyperparameter tuning, and risk overfitting to the calibration set~\citep{li2021brecq, hubara2021accurate}.

\paragraph{Greedy rounding.} These algorithms typically assume a fixed quantization grid and process weights sequentially in a given layer: each weight is assigned to a grid point, and the resulting quantization error is compensated for in the remaining unquantized weights so as to minimize the layer output error. OBC~\citep{frantar2022optimal} introduced this framework by adapting Optimal Brain Surgeon~\citep{hassibi1993optimal} to layer-wise PTQ, and GPTQ~\citep{frantar2022gptq} scaled it to LLMs through an efficient implementation. Qronos~\citep{zhang2026qronos} and GPTAQ~\citep{li2025gptaq} extend GPTQ to account for the quantization error propagated from previously quantized layers.

\paragraph{Scale selection.} The choice of
scale determines the rounding errors that greedy methods must compensate.
The dominant choice in PTQ pipelines remains \emph{data-free}: either the absmax scale, or weight-error
grid-search over a finite set of candidates in $(0, s_{\mathrm{absmax}}]$ \citep{zhang2026qronos,egiazarian2026bridging,sanjeet2026mixquant}. Both ignore how input
activations amplify weight errors at the layer output, and both implicitly cap
$s$ at $s_{\mathrm{absmax}}$, a restriction that is not always benign (as exemplified in Figure~\ref{fig:scale-dist} in Appendix~\ref{ap:scale-dist}). Previous work on CNN quantization by \citet{idelbayev2021optimal} derived a globally optimal, data-free scale optimizer for weights quantized with RTN. The method minimizes weight reconstruction error, so it ignores how activations amplify and correlate errors at the layer output. EasyQuant~\citep{wu2020easyquant} proposes a data-aware quantization procedure for CNNs that alternates grid searches over weight and activation scales using convolution-output cosine similarity. In the setting most relevant to LLM PTQ, CDQuant~\citep{nair2024cdquant}, COMQ~\citep{zhang2025comq}, and Beacon~\citep{zhang2026beacon} propose different data-aware methods for weight scale selection together with their own coordinate-descent procedures for optimizing the weight-to-grid assignments. CDQuant selects scales before rounding via grid-search on the layer output error; this inherits the suboptimal $s \leq s_{\mathrm{absmax}}$ constraint and the discrete granularity of the search. COMQ and Beacon jointly refine scales and grid assignments, an NP-hard problem where they do not offer global-optimality guarantees. Moreover, COMQ and Beacon are restricted to channel-wise quantization, while CDQuant's group-wise variant treats groups independently, ignoring cross-group interactions through the activations.  A complementary line of work focuses on scale selection in quantization-aware training (QAT)~\citep{Esser2020LEARNED, liu2026paretoq} or gradient-based PTQ~\citep{shao2024omniquant, cheng2024optimize}. These approaches require backpropagation, are often compute-intensive and need careful hyperparameter tuning, whereas our method is gradient-free and does not require hyperparameter tuning.

\section{Methods}
\label{sec:channel-wise}
We first derive PiSO for channel-wise RTN scale optimization. We then extend it to group-wise quantization, via two different tractable objectives; and finally propose three strategies for integration with greedy error correction methods such as GPTQ and Qronos.

We start by restricting ourselves to the RTN quantizer (Equation~\ref{eq:quantizer}). Therefore, we express the layer output reconstruction objective (Equation~\ref{eq:layer-obj}) as a function of the scaling factors:
\begin{equation}\label{eq:layer-obj-param}
    E(\mathbf{s}) = \bigl\|\mathbf{X}\mathbf{W} - \tilde{\mathbf{X}}\,\mathcal{Q}(\mathbf{W;s})\bigr\|_F^2, \quad \mathcal{Q}(\mathbf{W;s}):=\mathbf{s} \cdot \left\lfloor\frac{\mathbf{W}}{\mathbf{s}}\right\rceil_{\G}.
\end{equation}
This restriction is principled: to minimize the weight reconstruction error (data-free), the optimal grid assignments are directly determined by the scales and achieved by RTN (Proposition~\ref{prop:optimal_mse}). It also makes the  data-aware problem tractable: jointly optimizing scales and assignments is an NP-hard integer least-squares problem (Remark~\ref{rem:subopt_data_aware}), whereas under RTN the scales minimizing Equation~\ref{eq:layer-obj-param} can be computed \emph{exactly} in polynomial time despite its non-convexity. PiSO rests on three observations: (1) separability across output channels~\citep{zhang2026beacon}; (2) a closed-form optimum for a fixed grid assignment~\citep{zhang2025comq, zhang2026beacon}; and (3) \textbf{the key insight}: since $q(\mathbf{w}; s)$ is \emph{piecewise constant} in $s$, the real line (search space of $s$) splits into finitely many intervals, each with a closed-form minimizer of the error.

\begin{figure}[t]
  \centering
  \includegraphics[width=0.95\textwidth]{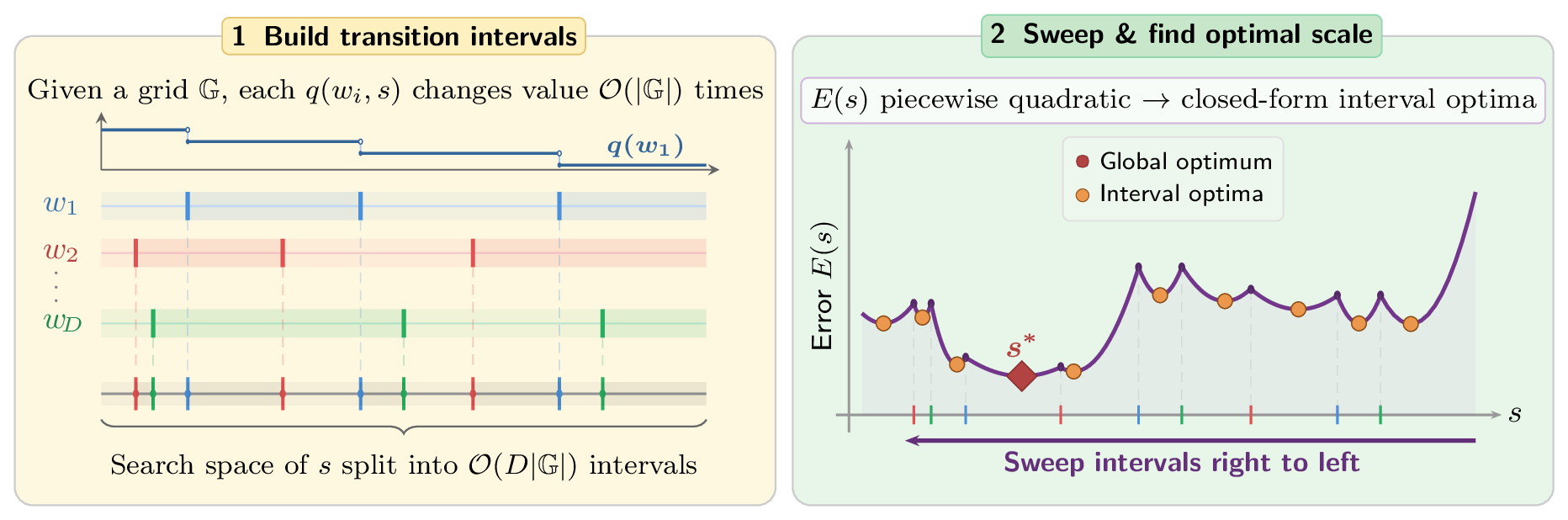}
  \caption{Overview of~\ScaleAlgorithm. The grid assignment $q(\mathbf{w}; s)$ is piecewise constant in the scale, partitioning the real line into intervals where the objective in Equation~\ref{eq:layer-obj-param} simplifies to a quadratic. PiSO sweeps through these intervals, and returns the scale achieving the lowest global error.}
  \label{fig:algorithm}
\end{figure}

\subsection{Optimal Scale Algorithm for Channel-wise Quantization}
We now formalize these observations and combine them to propose the scale optimization algorithm \ScaleAlgorithm{} (Algorithm~\ref{alg:scale-sweep} in Appendix~\ref{Implementation-det}). For intuition, Appendix~\ref{app:low-dimensional-example} works through a two-dimensional example that contrasts data-aware and data-free scale selection.

\paragraph{Channel-wise separability.}
Writing $\mathbf{W}=[\mathbf{w}_1,\ldots,\mathbf{w}_M]$ as a collection of column vectors, each with its own scale, the objective in Equation~\ref{eq:layer-obj-param} separates into $M$ independent problems (Lemma~\ref{lemma:separable}):
\begin{equation}\label{eq:separable}
    E(\mathbf{s})
    = \sum_{j=1}^M
      \bigl\|
        \mathbf{X}\mathbf{w}_j
        - \tilde{\mathbf{X}}\,\mathcal{Q}(\mathbf{w}_j;\, s_j)
      \bigr\|^2
    \;\;=\;\; \sum_{j=1}^M E_j(s_j).
\end{equation}
Each $s_j$ can therefore be optimized independently. In what follows, we drop the subscript $j$ and focus on finding the optimal scale $s$ for a single output channel with weight vector $\mathbf{w}\in\R^D$.

\paragraph{Closed-form solution for fixed grid assignments.}
Using $\mathcal{Q}(\mathbf{w};\,s) = s\, q(\mathbf{w};\,s)$ from Equation~\ref{eq:quantizer}, the channel-wise error reads
\begin{equation}\label{eq:error-decomp}
E(s) = \|\mathbf{X}\mathbf{w} - s\,\tilde{\mathbf{X}}\, q(\mathbf{w};\,s)\|^2,
\end{equation}
which is non-convex through the rounding operator in $q(\cdot)$. Yet, for any \emph{fixed} grid assignment $\mathbf{q}$, $E(s)$ is convex in $s$ with a closed-form minimizer (Lemma~\ref{lemma:closed_form_scale}, Appendix~\ref{appendix:proofs}). Defining
\begin{equation}\label{eq:HG-def}
\mathbf{H} := \tilde{\mathbf{X}}^\top\tilde{\mathbf{X}}, \quad \mathbf{G} := \tilde{\mathbf{X}}^\top\mathbf{X} \;\in \R^{D \times D}, \qquad \alpha(\mathbf{q}) := \mathbf{q}^\top \mathbf{G}\, \mathbf{w}, \quad \beta(\mathbf{q}) := \mathbf{q}^\top \mathbf{H}\, \mathbf{q},
\end{equation}
the optimal scale for a fixed $\mathbf{q}$ and its associated error (up to a $s$-independent constant) are
\begin{equation}\label{eq:closed-form}
s^\ast(\mathbf{q}) = \alpha(\mathbf{q})/\beta(\mathbf{q}), \qquad E(s \mid \mathbf{q}) = \mathrm{const} - 2s\,\alpha(\mathbf{q}) + s^2\,\beta(\mathbf{q})
\end{equation}
as shown in Appendix~\ref{appendix:proofs} (Remark~\ref{remark:alpha_beta_equivalence}). Equation~\ref{eq:error-decomp} accommodates several layer-wise reconstruction objectives:
\begin{inparaenum}[(1)]
    \item \emph{cross-activation} ($\mathbf{X}\neq\tilde{\mathbf{X}}$) as in the objective of Qronos;
    \item \emph{self-activation} (replacing $\mathbf{X}$ by $\tilde{\mathbf{X}}$), as in the GPTQ objective; and
    \item \emph{data-free} (setting $\mathbf{X}$ and $\tilde{\mathbf{X}}$ to the identity $\mathbf{I}$).
\end{inparaenum} Our algorithm is applicable for any of these choices. The \emph{data-free} objective results in an algorithm equivalent to OptQuant by \citet{idelbayev2021optimal}.

\paragraph{Piecewise-constant structure.}
Equation~\ref{eq:closed-form} assumes a fixed grid assignment $\mathbf{q}$, but $q(\mathbf{w};\,s)$ is not constant in $s$. However, $q(\mathbf{w};\,s)$ is \emph{piecewise constant} in $s$, with the breakpoints admitting a closed form. For example, for a positive scalar weight $w_i$, as $s$ decreases from $+\infty$ toward $0$ the ratio $w_i / s$ grows and the assignment $q_i = \round{w_i / s}_\G$ walks through $\G$, jumping from $g_\ell$ to $g_{\ell+1}$ exactly when $w_i / s$ crosses their midpoint (Corollary~\ref{cor:full_vector}, Appendix~\ref{appendix:proofs}). This yields at most $|\G| - 1$ \emph{transition scales} per weight, and $D(|\G|-1)$ for the full vector $\mathbf{w}\in\R^D$ that follow the form:
\begin{equation}\label{eq:transition}
    t_{i,\ell} = \frac{2\, w_i}{g_\ell + g_{\ell+1}}, \qquad i = 1,\ldots,D, \quad \ell = 1,\ldots,|\G|-1.
\end{equation}
Crucially, each $t_{i,\ell}$ is associated with a single weight index $i$: when $s$ crosses $t_{i,\ell}$, only $q_i$ changes while all other entries stay put (Appendix~\ref{appendix:degenerate_intervals} covers coinciding transitions). Sorting the positive transitions as $\tau_1 > \cdots > \tau_T$, with $T\le D(|\G|-1)$, partitions $(0,+\infty)$ into $T+1$ intervals on which $\mathbf{q}$ is constant. We focus on positive scales, the signed extension is detailed in Appendix~\ref{appendix:sign_change}.

\paragraph{Sweeping with incremental updates.}
Within each interval $(\tau_j, \tau_{j-1})$, $\mathbf{q}$ is constant and $E(s)$ is quadratic in $s$, with minimizer $s_j^\ast = \mathrm{clamp}(\alpha/\beta,\, \tau_j,\, \tau_{j-1})$, clamped since $\alpha/\beta$ may fall outside the interval, and associated error $\phi_j = (s_j^\ast)^2\,\beta - 2\,s_j^\ast\,\alpha$ (Equations~\ref{eq:HG-def} and~\ref{eq:closed-form}). The lowest $\phi_j$ across intervals gives the optimal output reconstruction error, defined as the global infimum of the objective. When a minimizing candidate is admissible, this value is attained exactly. Otherwise, an admissible scale can approach it to any prescribed objective tolerance (Proposition~\ref{prop:sweep_optimality}). Recomputing $\mathbf{q}$, $\alpha$, and $\beta$ from scratch at each interval would cost $\mathcal{O}(D^2)$ per interval and $\mathcal{O}(D^3 |\G|)$ in total. Instead,  we exploit that at each transition boundary one known entry of $\mathbf{q}$ changes: letting $i$ be that index and $\delta := q_i^{\mathrm{new}} - q_i^{\mathrm{old}}$, we maintain the running scalars $\alpha$, $\beta$ and the auxiliary vector $\mathbf{h} := \mathbf{H}\, \mathbf{q}$ via
\begin{equation}\label{eq:updates}
    \alpha \leftarrow \alpha + \delta \cdot (\mathbf{G}\mathbf{w})_i, \qquad
    \beta  \leftarrow \beta + 2\delta \cdot {h}_i + \delta^2 \cdot H_{ii}, \qquad
    \mathbf{h} \leftarrow \mathbf{h} + \delta \cdot \mathbf{H}_{i,:}.
\end{equation}
The vector $\mathbf{G}\mathbf{w}$ is precomputed once. Updating $\alpha$ and $\beta$ costs $\mathcal{O}(1)$, while $\mathbf{h}$ costs $\mathcal{O}(D)$, reducing the overall per-channel complexity of PiSO to $\mathcal{O}(D^2 |\G|)$ (Proposition~\ref{prop:incremental_updates}, Appendix~\ref{appendix:proofs}).

\paragraph{Algorithm analysis.}
When $\mathbf{H}$ is diagonal, $\mathbf{h} = \mathbf{H}\mathbf{q}$ does not need to be maintained explicitly, and the update of $\beta$ becomes $\beta \leftarrow \beta + H_{ii}(2\, q_i^{\mathrm{old}} \delta + \delta^2)$. This reduces the total complexity of sweeping the intervals to $\mathcal{O}(D|\G|)$ per channel, \emph{linear} in weight dimension. This covers the \emph{data-free} regime and, also, using a diagonal approximation of $\mathbf{H}$ and $\mathbf{G}$, trading fidelity for speed. In both the dense, and diagonal case, the $M$ channel problems are embarrassingly parallel.

\subsection{Group-wise Quantization}
\label{sec:group-wise}
In group-wise quantization with group size $G$ and $K$  groups Equation~\ref{eq:error-decomp} becomes
\begin{equation}\label{eq:error-group}
    E(s_1, \ldots, s_K) =
    \left\|
    \mathbf{X}\mathbf{w} - \tilde{\mathbf{X}}\left[s_1 q(\mathbf{w}_{\{1\}};s_1), \ldots, s_Kq(\mathbf{w}_{\{K\}};s_K)\right]^\top
    \right\|^{2}.
\end{equation}
Minimizing Equation~\ref{eq:error-group} can no longer be solved with a one-dimensional sweep. Instead of solving $\mathcal{O}(D|\G|)$ closed-form subproblems on intervals of the real line, it would require solving $\mathcal{O}((G|\G|)^K)$ subproblems on hyper-intervals of $\R^K$.\footnote{Each scale $s_k$ has its own set of transition points, and the joint assignment $(\mathbf{q}^1,\ldots,\mathbf{q}^K)$ is constant only on the intersection of intervals across all $K$ axes.} A joint grid search suffers from the same curse of dimensionality: $R$ candidates per scale require $R^K$ evaluations. We propose two tractable objectives for group-wise quantization, which PiSO can efficiently minimize. We note that they are independent of the 1D solver and could equally be paired with other optimization procedures.

\paragraph{Independent groups.}
Similarly to \citet{nair2024cdquant}, we discard all cross-group interactions and minimize the block-diagonal approximation of Equation~\ref{eq:error-group}:
\begin{equation}\label{eq:independent-group}
    E_{\mathrm{ind}}(s_1,\ldots,s_K)
    := \sum_{k=1}^{K}
      \bigl\|
        \mathbf{X}_{\{k\}}\mathbf{w}_{\{k\}}
        - s_k\, \tilde{\mathbf{X}}_{\{k\}}\, q(\mathbf{w}_{\{k\}};\, s_k)
      \bigr\|^2.
\end{equation}
$E_{\mathrm{ind}}$ is separable in the $K$ scales, so each group reduces to a channel-wise sweep on $G$ weights with $\mathbf{H}_{kk} = \tilde{\mathbf{X}}_{\{k\}}^\top\tilde{\mathbf{X}}_{\{k\}}$ and $\mathbf{G}_{kk} = \tilde{\mathbf{X}}_{\{k\}}^\top\mathbf{X}_{\{k\}}$. Each group sweep costs $\mathcal{O}(G^2|\G|)$, for a total per-channel cost of $\mathcal{O}(DG|\G|)$ that parallelizes across both groups and channels. Proposition~\ref{prop:block_diag_gap} (Appendix~\ref{appendix:proofs}) bounds the gap between optimizing the original objective and the independent one by the spectral norms of the discarded cross-group blocks of $\mathbf{H}$ and $\mathbf{G}$.

\paragraph{Sequential groups.}\label{sequential_groups}
To account for cross-group interactions, one can greedily process each group. After fixing the scales for the first $k-1$ groups, $s_{k}$ is chosen to minimize the partial reconstruction error in which the already quantized weights (and their scales) are fixed. More precisely,

\begin{equation}\label{eq:sequential-group}
E_{\mathrm{seq}}(s_k) := \bigl\| \mathbf{X}_{\{\leq k\}}\,\mathbf{w}_{\{\leq k\}} - \tilde{\mathbf{X}}_{\{<k\}}\,\tilde{\mathbf{w}}_{\{<k\}} - s_k\, \tilde{\mathbf{X}}_{\{k\}}\, q(\mathbf{w}_{\{k\}};s_k) \bigr\|^2,
\end{equation}

where $\tilde{\mathbf{w}}_{\{i\}} := s_i^\ast\,q(\mathbf{w}_{\{i\}};\,s_i^\ast)$ for $i < k$. The quantization error of preceding groups shifts the target against which group~$k$ is optimized. The procedure provides a middle ground between the independent variant and a full joint optimization.
Mirroring the column reordering commonly used in GPTQ implementations~\citep{gptq, pappalardo2025xilinx}, we process groups by aggregated Hessian importance (more details in Appendix~\ref{group-proc-order}). Each group sweep costs $\mathcal{O}(G^2|\G|)$, for a total per-channel cost of $\mathcal{O}(DG|\G|)$.

\subsection{Integration with Existing PTQ Methods}\label{sec:inte}
Integration with the \emph{transform} stage of the PTQ pipeline (Section~\ref{background}) is immediate: \ScaleAlgorithm{} can be applied to the transformed weights produced by methods such as rotation-based equalization~\citep{ashkboos2024quarot,liu2025spinquant} before the round stage. Composing with \emph{error correction} methods such as GPTQ is less direct. They keep the scale fixed but deliberately move weights away from their nearest grid point to compensate for earlier rounding errors, interfering with the RTN assumption. Note that jointly optimizing the scales and the grid assignments is NP-hard~\citep{zhang2026beacon}, as is the data-aware optimization of the grid assignments alone even for a fixed scale~\citep{zhang2026qronos, hassibi2002expected} (Remark~\ref{rem:subopt_data_aware}). Thus, error correction methods greedily assign weights to a fixed grid and do not offer global optimality guarantees~\citep{zhang2023post, zhang2025provable}. Moreover, their performance depends on the choice of the scale, as the quantization grid determines the rounding errors they must compensate.

\paragraph{Conditional optimization.}
Motivated by this, we approximate the joint problem through two complementary updates.
In the first, \ScaleAlgorithm{} computes the \emph{conditionally optimal} scale for the current state of the quantized model, using RTN as a surrogate for the grid assignments that error correction has not made yet. For group-wise quantization, optimality is with respect to the objectives of Section~\ref{sec:group-wise}.
In the second, error correction greedily assigns the weights to the resulting grid, diffusing error into the weights that have not yet been quantized. We propose three integration strategies that differ in the \emph{granularity of interleaving} these updates, that is, in how much of the model has already been quantized when \ScaleAlgorithm{} is applied, and hence in how much the surrogate has to cover. We describe them for GPTQ, but they also apply to Qronos and related layer-wise methods. Appendix~\ref{app:interleaved-group} elaborates on these integrations, providing their pseudocode.

\paragraph{Decoupled (GPTQ $+$ PiSO).}
A first pass over the model optimizes the scales of all layers, computing each layer's $\tilde{\mathbf{X}}$ by RTN-quantizing the preceding layers with their PiSO scales. A second pass then applies error correction with these scales held fixed. This fully separates the two stages, at the cost of an extra calibration pass. It also creates an activation mismatch, as error correction changes the inputs to which each layer's scales were fitted.

\paragraph{Layer-wise interleaving (GPTQ $\odot$ PiSO).}
Scale optimization and error correction alternate in a single pass, with each layer's scales optimized immediately before its error correction step. All preceding layers have been both scale-optimized and error-corrected, so the $\tilde{\mathbf{X}}$ used by PiSO matches the activations the layer receives during error correction and at inference. Error correction itself is unchanged and operates on the resulting fixed grid. The surrogate is thus only needed within the layer, for the assignments that error correction makes after the scales are selected.

\paragraph{Group-wise interleaving (GPTQ\,$\boxdot$\,PiSO).}
In group-wise quantization, the coupling can be tightened further by running PiSO at each group boundary inside the quantization loop of error correction, just before that group is quantized. This works with either objective of Section~\ref{sec:group-wise}, we adjust the sequential one to account for the already diffused error (Appendix~\ref{app:interleaved-group}). The scale is thus optimized for weights that already include the error-diffusion updates from all preceding groups, so the surrogate is only needed within the current group. This strategy requires the columns of each group to be processed contiguously, so we use a group-aware ordering: groups are sorted by aggregate Hessian importance, and columns within each group by their own importance (Appendix~\ref{group-proc-order}).

\section{Experiments}
\label{sec:experiments}
\begin{figure}[t]
    \centering
    \begin{subfigure}[b]{0.46\textwidth}
        \centering
        \includegraphics[width=\textwidth]{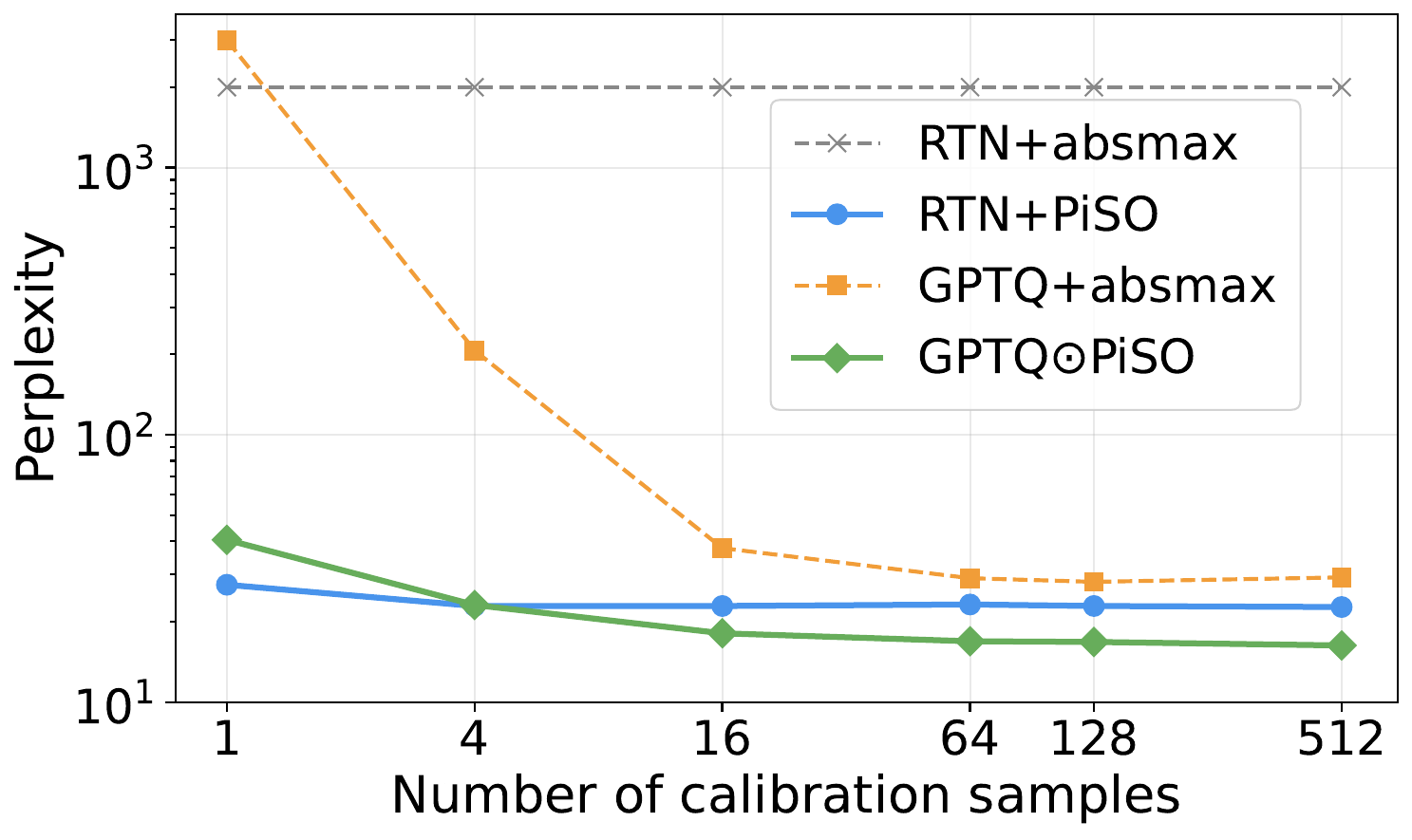}
        \caption{C4 calibration, WikiText2 perplexity ($\downarrow$)}
        \label{fig:cal-sweep-ppl}
    \end{subfigure}
    \hfill
    \begin{subfigure}[b]{0.45\textwidth}
        \centering
        \includegraphics[width=\textwidth]{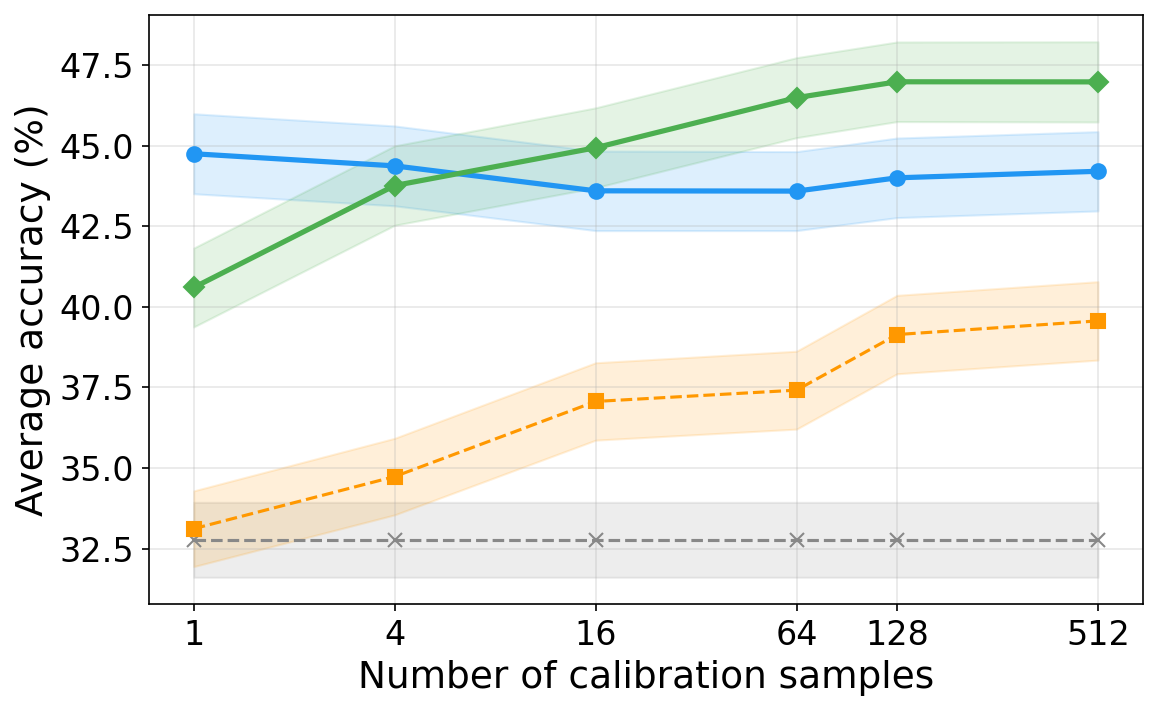}
        \caption{WikiText2 calibration, Zero-shot accuracy ($\uparrow$)}
        \label{fig:cal-sweep-acc}
    \end{subfigure}
\caption{Effect of calibration set size on Llama-3.2-1B for 3-bit channel-wise weight quantization. \ScaleAlgorithm{} is stable with few samples, whereas GPTQ with absmax requires $\geq 64$ samples to converge.}
\label{fig:cal-sweep}

\end{figure}
\begin{table}[t]
\centering
\caption{Channel-wise 2-bit weight-only quantization results on Llama-3 and Qwen-2.5 models, reporting WikiText2 perplexity ($\downarrow$) and average zero-shot accuracy ($\uparrow$).~\ScaleAlgorithmShort{} reduces absmax perplexity by orders of magnitude; the interleaved variants achieve the best results in most configurations.
}
\label{tab:channel_2-bit}
\resizebox{\textwidth}{!}{
\begin{tabular}{l ccc ccc ccc ccc}
\toprule
 & \multicolumn{6}{c}{\textbf{2-bit Llama}} & \multicolumn{6}{c}{\textbf{2-bit Qwen}} \\
\cmidrule(lr){2-7} \cmidrule(lr){8-13}
 & \multicolumn{3}{c}{WikiText2 ($\downarrow$)} & \multicolumn{3}{c}{0-shot ($\uparrow$)} & \multicolumn{3}{c}{WikiText2 ($\downarrow$)} & \multicolumn{3}{c}{0-shot ($\uparrow$)} \\
\cmidrule(lr){2-4} \cmidrule(lr){5-7} \cmidrule(lr){8-10} \cmidrule(lr){11-13}
 & 1B & 3B & 8B & 1B & 3B & 8B & 1.5B & 3B & 7B & 1.5B & 3B & 7B \\
\midrule
BF16 & 8.97 & 7.14 & 5.89 & 52.6 & 62.0 & 68.6 & 8.89 & 7.88 & 6.92 & 54.9 & 56.5 & 57.7 \\
\midrule
RTN + absmax & 2e5 & 4e5 & 4e5 & 33.0 & 32.5 & 33.4 & 7e6 & 6e6 & 1e5 & 31.8 & 32.9 & 32.3 \\
RTN + \DataFreeAlgorithm & 2e6 & 8e5 & 2e6 & 32.8 & 32.3 & 32.1 & 2e8 & 3e8 & 7e6 & 32.2 & 31.4 & 32.2 \\
RTN + \ScaleAlgorithmShort & 3e3 & 3e3 & 3e3 & 32.3 & 31.6 & 32.0 & 2e3 & 2e3 & 4e3 & 31.8 & 32.7 & 32.2 \\
\midrule
AWQ & 2e6 & 8e5 & 1e6 & 32.7 & 32.0 & 32.1 & 9e7 & 7e8 & 9e6 & 32.4 & 31.7 & 32.5 \\
Beacon & 52.8 & 72.4 & 25.5 & 33.8 & 33.3 & 36.8 & 1e3 & 33.2 & 21.9 & 32.5 & 32.4 & 33.9 \\
\midrule
GPTQ + absmax & 3e3 & 415 & 331 & 31.6 & 31.9 & 31.5 & 675 & 596 & 72.7 & 33.0 & 32.8 & 32.1 \\
GPTQ + \DataFreeAlgorithm & 3e3 & 9e3 & 1e6 & 31.7 & 32.3 & 31.8 & 353 & 1e5 & 185 & 32.1 & 31.1 & 34.0 \\
GPTQ + \ScaleAlgorithmShort & 92.5 & 47.5 & 32.5 & 33.1 & 33.2 & 34.7 & \textbf{39.2} & 32.7 & \textbf{15.6} & \textbf{34.4} & 32.9 & 37.1 \\
GPTQ $\odot$ \ScaleAlgorithmShort & 54.7 & 29.5 & 22.3 & 35.5 & 36.3 & 40.2 & 44.7 & 30.6 & 17.0 & 33.7 & 34.4 & 34.9 \\
\midrule
Qronos + absmax & 197 & 68.9 & 49.5 & 31.6 & 32.7 & 33.5 & 111 & 35.8 & 2e4 & 32.1 & 32.5 & 32.1 \\
Qronos + \DataFreeAlgorithm & 5e3 & 1e4 & 3e5 & 31.2 & 31.5 & 31.6 & 3e3 & 657 & 1e4 & 31.2 & 32.6 & 32.6 \\
Qronos + \ScaleAlgorithmShort & 216 & 84.8 & 76.0 & 32.1 & 33.3 & 33.0 & 196 & 27.4 & 24.5 & 32.3 & 32.9 & 36.1 \\
Qronos $\odot$ \ScaleAlgorithmShort & \textbf{27.9} & \textbf{18.9} & \textbf{16.6} & \textbf{37.5} & \textbf{41.4} & \textbf{44.0} & 131 & \textbf{17.2} & 16.3 & 32.7 & \textbf{34.8} & \textbf{38.1} \\
\bottomrule
\end{tabular}
}
\end{table}

\begin{table}[t]
\centering
\caption{Channel-wise 3-bit and 4-bit quantization results on Llama-3 models, reporting WikiText2 perplexity ($\downarrow$) and average zero-shot accuracy ($\uparrow$). The gap between absmax and~\ScaleAlgorithmShort-optimized scales widens at lower bit-widths, and \ScaleAlgorithm~consistently achieves the best results.}
\label{tab:channel_llama_3-bit4-bit}
\resizebox{\textwidth}{!}{
\begin{tabular}{l ccc ccc ccc ccc}
\toprule
 & \multicolumn{6}{c}{\textbf{3-bit}} & \multicolumn{6}{c}{\textbf{4-bit}} \\
\cmidrule(lr){2-7} \cmidrule(lr){8-13}
 & \multicolumn{3}{c}{WikiText2 ($\downarrow$)} & \multicolumn{3}{c}{0-shot ($\uparrow$)} & \multicolumn{3}{c}{WikiText2 ($\downarrow$)} & \multicolumn{3}{c}{0-shot ($\uparrow$)} \\
\cmidrule(lr){2-4} \cmidrule(lr){5-7} \cmidrule(lr){8-10} \cmidrule(lr){11-13}
 & 1B & 3B & 8B & 1B & 3B & 8B & 1B & 3B & 8B & 1B & 3B & 8B \\
\midrule
BF16 & 8.97 & 7.14 & 5.89 & 52.6 & 62.0 & 68.6 & 8.97 & 7.14 & 5.89 & 52.6 & 62.0 & 68.6 \\
\midrule
RTN + absmax & 2e3 & 436 & 1e3 & 32.8 & 35.7 & 38.0 & 22.3 & 9.61 & 7.83 & 46.3 & 55.6 & 64.7 \\
RTN + \DataFreeAlgorithm & 9e3 & 6e4 & 5e4 & 31.9 & 32.4 & 32.2 & 136 & 8e3 & 9e4 & 33.8 & 31.8 & 32.2 \\
RTN + \ScaleAlgorithmShort & 19.5 & 10.6 & 8.99 & 44.2 & 52.1 & 59.9 & 11.0 & 7.92 & 6.64 & 50.0 & 60.4 & 67.3 \\
\midrule
AWQ & 102 & 50.1 & 11.8 & 38.2 & 52.5 & 58.7 & 11.9 & 8.02 & 6.76 & 50.4 & 59.8 & 66.6 \\
Beacon & 19.1 & 33.7 & 9.35 & 42.4 & 40.1 & 53.5 & 13.6 & 25.1 & 7.53 & 46.1 & 46.3 & 59.7 \\
\midrule
GPTQ + absmax & 19.6 & 11.3 & 9.23 & 39.6 & 49.8 & 55.1 & 10.4 & 7.79 & 6.49 & 49.3 & 59.4 & 65.6 \\
GPTQ + \DataFreeAlgorithm & 546 & 9e3 & 8e5 & 33.1 & 32.0 & 31.7 & 255 & 9e3 & 4e5 & 35.1 & 31.5 & 31.7 \\
GPTQ + \ScaleAlgorithmShort & 12.9 & 9.19 & 7.42 & 47.5 & 54.6 & \textbf{63.0} & 9.94 & 7.63 & 6.32 & \textbf{52.0} & 60.2 & \textbf{68.3} \\
GPTQ $\odot$ \ScaleAlgorithmShort & 12.8 & 9.08 & 7.41 & 47.0 & \textbf{55.9} & 62.6 & 9.91 & 7.65 & 6.32 & 51.3 & 60.0 & 68.1 \\
\midrule
Qronos + absmax & 15.0 & 9.83 & 8.27 & 40.6 & 51.4 & 57.9 & 10.0 & 7.64 & 6.36 & 50.2 & 60.1 & 65.2 \\
Qronos + \DataFreeAlgorithm & 2e3 & 9e3 & 8e5 & 32.8 & 31.9 & 32.3 & 335 & 9e3 & 4e5 & 34.6 & 31.8 & 31.1 \\
Qronos + \ScaleAlgorithmShort & \textbf{11.6} & 8.42 & 7.01 & \textbf{48.6} & 55.5 & 61.4 & 9.61 & \textbf{7.44} & 6.19 & 51.5 & \textbf{61.2} & 67.4 \\
Qronos $\odot$ \ScaleAlgorithmShort & 11.6 & \textbf{8.41} & \textbf{6.98} & 48.0 & 55.6 & 62.7 & \textbf{9.59} & 7.44 & \textbf{6.18} & 51.1 & 60.5 & 67.0 \\
\bottomrule
\end{tabular}
}
\end{table}

\begin{table}[t]
\centering
\caption{Group-wise quantization results on Llama-3 3B and Qwen-2.5 3B with group sizes 16 and 32 (G16/G32), reporting WikiText2 perplexity ($\downarrow$). The subscripts $\mathrm{I}/\mathrm{S}$ denote the independent/sequential objectives (Section~\ref{sec:group-wise}); entries without subscript report the best of the two.}
\label{tab:per_group_3B_ppl}
\resizebox{\textwidth}{!}{
\begin{tabular}{l cccc cccc cccc}
\toprule
 & \multicolumn{4}{c}{\textbf{2-bit}} & \multicolumn{4}{c}{\textbf{3-bit}} & \multicolumn{4}{c}{\textbf{4-bit}} \\
\cmidrule(lr){2-5} \cmidrule(lr){6-9} \cmidrule(lr){10-13}
 & \multicolumn{2}{c}{Llama} & \multicolumn{2}{c}{Qwen} & \multicolumn{2}{c}{Llama} & \multicolumn{2}{c}{Qwen} & \multicolumn{2}{c}{Llama} & \multicolumn{2}{c}{Qwen} \\
\cmidrule(lr){2-3} \cmidrule(lr){4-5} \cmidrule(lr){6-7} \cmidrule(lr){8-9} \cmidrule(lr){10-11} \cmidrule(lr){12-13}
 & G16 & G32 & G16 & G32 & G16 & G32 & G16 & G32 & G16 & G32 & G16 & G32 \\
\midrule
BF16 & 7.14 & 7.14 & 7.88 & 7.88 & 7.14 & 7.14 & 7.88 & 7.88 & 7.14 & 7.14 & 7.88 & 7.88 \\
\midrule
RTN + absmax & 8e5 & 8e5 & 3e8 & 2e8 & 4e3 & 11.7 & 21.0 & 102 & 7.69 & 7.72 & 8.78 & 8.91 \\
RTN + \DataFreeAlgorithm & 2e6 & 2e6 & 33.0 & 32.9 & 16.4 & 23.3 & 45.7 & 43.0 & 9.81 & 10.2 & 51.3 & 52.0 \\
RTN + \ScaleAlgorithmShort$_{\mathrm{I}}$ & 2e4 & 2e4 & 388 & 758 & 8.21 & 8.62 & 8.85 & 9.13 & 7.33 & 7.42 & 8.08 & 8.18 \\
RTN + \ScaleAlgorithmShort$_{\mathrm{S}}$ & 19.4 & 28.7 & 17.1 & 25.9 & 7.86 & 8.17 & 8.58 & 8.92 & 7.26 & 7.33 & 8.00 & 8.04 \\
\midrule
AWQ & 8e5 & 8e5 & 2e8 & 2e8 & 9.61 & 9.57 & 11.4 & 11.4 & 7.50 & 7.55 & 8.37 & 8.46 \\
\midrule
GPTQ + absmax & 30.3 & 26.4 & 34.8 & 30.4 & 8.41 & 8.29 & 8.86 & 8.91 & 7.35 & 7.38 & 8.06 & 8.03 \\
GPTQ + \DataFreeAlgorithm & 34.8 & 47.0 & 35.8 & 34.5 & 10.7 & 11.1 & 49.4 & 48.6 & 9.25 & 9.34 & 52.5 & 52.7 \\
GPTQ + \ScaleAlgorithmShort & 21.5 & 25.2 & 15.5 & 18.1 & 8.12 & 8.50 & 8.55 & 8.79 & 7.30 & 7.37 & 8.01 & 8.03 \\
GPTQ $\odot$ \ScaleAlgorithmShort & 20.7 & 24.0 & 15.3 & 18.1 & 8.11 & 8.44 & 8.52 & 8.72 & 7.30 & 7.36 & 8.01 & 8.03 \\
GPTQ $\boxdot$ \ScaleAlgorithmShort & 17.0 & 21.3 & 14.3 & 16.6 & 7.92 & 8.29 & 8.48 & 8.68 & 7.27 & 7.34 & \textbf{7.98} & 8.02 \\
\midrule
Qronos + absmax & 22.6 & 24.5 & 77.3 & 21.1 & 8.07 & 8.08 & 8.67 & 8.66 & 7.32 & 7.31 & 8.03 & 8.03 \\
Qronos + \DataFreeAlgorithm & 15.1 & 18.1 & 17.7 & 35.9 & 7.74 & \textbf{7.91} & 8.43 & 8.64 & 7.24 & \textbf{7.27} & 7.99 & 8.02 \\
Qronos + \ScaleAlgorithmShort & 14.6 & 16.3 & 13.3 & 14.7 & 7.75 & 7.95 & 8.44 & 8.62 & 7.25 & 7.29 & \textbf{7.98} & 8.02 \\
Qronos $\odot$ \ScaleAlgorithmShort & 14.8 & 16.3 & 13.0 & 14.4 & 7.76 & 7.96 & 8.44 & 8.61 & 7.25 & 7.29 & 7.99 & 8.03 \\
Qronos $\boxdot$ \ScaleAlgorithmShort & \textbf{12.9} & \textbf{15.2} & \textbf{11.9} & \textbf{13.2} & \textbf{7.72} & \textbf{7.91} & \textbf{8.37} & \textbf{8.53} & \textbf{7.23} & 7.28 & \textbf{7.98} & \textbf{8.01} \\
\bottomrule
\end{tabular}
}
\end{table}

We conduct experiments aimed to evaluate:
\begin{inparaenum}[(1)]
    \item the gains of PiSO's optimal data-aware scales over absmax and data-free scales under RTN;
    \item the effectiveness of the independent and sequential objectives for group-wise quantization; and
    \item the strategies for integrating PiSO with GPTQ and Qronos.
\end{inparaenum}
We further compare against data-aware PTQ baselines and study calibration-data efficiency. Throughout the tables, \textbf{bold} highlights the best result(s) across each model, bit-width, and metric.

\paragraph{Experimental Setup.}
We quantize Llama-3~\citep{Grattafiori2024Llama} (3.2-1B, 3.2-3B, 3.1-8B) and Qwen-2.5 Instruct~\citep{qwen2025qwen25technicalreport} (1.5B, 3B, 7B), using the HuggingFace checkpoints~\citep{wolf-etal-2020-transformers}. Appendix~\ref{sec:app-models} extends coverage to Qwen2.5 14B and Gemma models \citep{team2024gemma}. For calibration, we sample 512 sequences of 2048 tokens from the WikiText2 training split~\citep{merity2017pointer}. We focus on weight-only quantization, which is not only a relevant setting by itself, but also allows us to decouple evaluation from the choice of activation-quantization scheme. We quantize weights to 2, 3, and 4 bits, where lower bit-widths make the choice of PTQ algorithm more critical. All scales are stored in BF16.

\paragraph{Metrics.}
We report perplexity on the WikiText2~\citep{merity2017pointer} validation split,
and average zero-shot accuracy on ARC-Easy, ARC-Challenge~\citep{clark2018think}, and HellaSwag~\citep{zellers2019hellaswag} evaluated with LightEval~\citep{lighteval}. Evaluating LLMs remains an open challenge, as no single benchmark fully captures model quality. We therefore prioritize breadth across model families and sizes, bit-widths, quantization granularities, and PTQ pipelines rather than deep benchmarking for a specific configuration. This design allows us to assess whether PiSO provides consistent improvements across diverse settings. For a subset of models, Appendix~\ref{sec:app-tasks} extends task coverage to MMLU~\citep{hendrycks2021measuring} (general knowledge), IFEval~\citep{zhou2023instruction} (instruction following), and LongBench~\citep{bai2024longbench} (long-context question answering).

\paragraph{Baselines.}
We compare PiSO against two scale-selection strategies, each combined with RTN, GPTQ, and Qronos:
\begin{inparaenum}[(1)]
    \item \emph{absmax}, which maps the largest-magnitude weight to the extreme grid point; and
    \item \emph{OptQuant}~\citep{idelbayev2021optimal}, the exact minimizer of the data-free objective $\|\mathbf{w} - s\,q(\mathbf{w};s)\|^2$. It is a stronger version of the common grid-search heuristic, being exact and not restricted to $(0, s_{\mathrm{absmax}}]$.
\end{inparaenum}
We further compare against two standalone PTQ methods. AWQ~\citep{lin2024awq} is a widely used weight-only PTQ method that protects salient channels through activation-aware rescaling before quantization. Beacon~\citep{zhang2026beacon}  jointly optimizes channel-wise scales and grid assignments via coordinate descent. We select it as none of the data-aware methods (CDQuant, COMQ, Beacon) of Section~\ref{background} that optimize scales has publicly available code, and Beacon is the most recent one and methodologically close to COMQ. Appendix~\ref{sec:app-channel} additionally compares against the gradient-based AutoRound~\citep{cheng2024optimize}. Unless stated otherwise, RTN\,$+$\,PiSO minimizes the cross-activation objective, whereas with GPTQ or Qronos, PiSO minimizes the objective of the error correction method.

\paragraph{Channel-wise Quantization Results.}
At 2 bits (Table~\ref{tab:channel_2-bit}), RTN collapses regardless of the scale, although \ScaleAlgorithm{} still lowers its perplexity by two orders of magnitude. With error correction,  Qronos\,$\odot$\,\ScaleAlgorithm{} achieves the best results in most configurations, reducing the perplexity of Qronos\,$+$\,absmax by $3$--$7\times$ on Llama. Interleaving outperforms the decoupled strategy in 10 of 12 perplexity comparisons. At 3 and 4 bits (Table~\ref{tab:channel_llama_3-bit4-bit}, Qwen in Table~\ref{tab:channel_qwen_3-bit4-bit_expanded}), \ScaleAlgorithm{} yields the best result in most configurations. Notably, at 3 bits, even RTN\,$+$\,\ScaleAlgorithm{} outperforms GPTQ\,$+$\,absmax on all Llama models. At 4 bits the finer grid narrows the gaps, yet \ScaleAlgorithm{} still improves RTN by 2.6--4.8 accuracy points on Llama. RTN\,$+$\,\ScaleAlgorithm{} already achieves lower perplexity than AWQ, with further improvements achieved by the combination of PiSO with error correction. In contrast, the data-free OptQuant scales often perform even worse than absmax.

\paragraph{Group-wise Quantization Results.}
Table~\ref{tab:per_group_3B_ppl} reports group-wise results on Llama 3B and Qwen 3B with group sizes 16 and 32. Under RTN, the sequential objective outperforms the independent one in every configuration, by orders of magnitude at 2 bits, confirming the importance of cross-group interactions. RTN\,$+$\,\ScaleAlgorithm$_{\mathrm{S}}$ alone outperforms GPTQ\,$+$\,absmax in most configurations, and AWQ in all of them. With error correction, group-wise interleaving is the best integration strategy in every configuration. Qronos $\boxdot$ \ScaleAlgorithm{} achieves the best 2-bit perplexities, improving on Qronos\,$+$\,absmax by up to $6.5\times$. Appendix~\ref{sec:app-group} reports results for Llama 8B and Qwen 7B and the zero-shot accuracies.

\paragraph{Calibration size sensitivity.}
The amount of calibration data required by PTQ algorithms is often overlooked. We sweep the number of calibration sequences from 1 to 512 on Llama-3.2-1B with 3-bit channel-wise quantization, calibrating on WikiText2, C4~\citep{raffel2020exploring}, or FineWeb~\citep{penedo2024fineweb} and always evaluating on WikiText2 (Figure~\ref{fig:cal-sweep} for C4, remaining datasets in Appendix~\ref{ap:cal-size}). RTN$+$\ScaleAlgorithm{}, which only fits one parameter per channel (precisely the scale), is highly data-efficient: four sequences were enough to reach near-asymptotic perplexity on all three datasets. In contrast, GPTQ\,$+$\,absmax, which has to optimize each grid-to-weight assignment, needs about 64 sequences to stabilize. Somewhat surprisingly, \ScaleAlgorithm{} also makes GPTQ more data-efficient: with 16 sequences, GPTQ$\odot$\ScaleAlgorithm{} outperforms GPTQ$+$absmax calibrated on 512. This suggests that a well-chosen grid can reduce the dependence of error correction on calibration size.

\paragraph{Additional experiments.}
Appendix~\ref{sec:app-channel} includes ablations over the objectives supported by~\ScaleAlgorithmShort{} and experiments with the non-integer \texttt{E2M1} grid. For a subset of configurations, we extend the evaluation to Qwen 14B and Gemma models, broaden the task coverage, and compare against AutoRound, an established gradient-based PTQ method (Tables~\ref{tab:autoround_llama_ppl} and~\ref{tab:autoround_llama_acc}). Finally, wall-clock runtimes and overhead relative to GPTQ in Table~\ref{tab:runtime} show that PiSO introduces minimal overhead.

\section{Conclusion}
\label{sec:conclusion}
We revisited an underexplored component of PTQ: scale selection. We introduced \ScaleAlgorithm{}, which selects channel-wise scales to achieve provably optimal output reconstruction error to arbitrary accuracy in $\mathcal{O}(D^2|\G|)$ time per channel. Building on this solver, we proposed two tractable objectives for group-wise quantization, one accounting for cross-group interactions, and three principled strategies for integrating \ScaleAlgorithm{} with GPTQ and Qronos. \ScaleAlgorithm{} outperformed absmax, OptQuant, AWQ, and Beacon, with gains widening at lower bit-widths. Moreover, \emph{data-free} optimal scales often failed in the channel-wise setting, confirming the need for a \emph{data-aware} objective. The best results are mostly obtained by combining \ScaleAlgorithm{} with error correction, supporting the effectiveness of the proposed integrations. Finally, PiSO markedly improved GPTQ data efficiency in the evaluated settings.

\paragraph{Limitations and future work.}
Tightening the theoretical analysis of the integration strategies, and further studying the observed data-efficiency gains are natural next steps. We focused on unconstrained BF16 scales. Hardware formats such as MX~\citep{ocp2024mx} and NVFP4~\citep{alvarez2025nvfp4} impose additional constraints (e.g., low-bit or hierarchical scales). We plan to extend \ScaleAlgorithm{} to such restricted-scale settings, as well as to asymmetric and weight-activation quantization.
\clearpage

\bibliographystyle{unsrtnat}
\bibliography{references}

@article{frantar2022gptq,
  publtype={informal},
  author={Elias Frantar and Saleh Ashkboos and Torsten Hoefler and Dan Alistarh},
  title={{GPTQ}: Accurate Post-Training Quantization for Generative Pre-trained Transformers},
  year={2022},
  cdate={1640995200000},
  journal={CoRR},
  volume={abs/2210.17323},
}

@inproceedings{
zhang2026qronos,
title={Qronos: Correcting the Past by Shaping the Future... in Post-Training Quantization},
author={Shihao Zhang and Haoyu Zhang and Ian Colbert and Rayan Saab},
booktitle={The Fourteenth International Conference on Learning Representations},
year={2026},
}

@inproceedings{sanjeet2026mixquant,
title={Pushing the Limits of Block Rotations in Post-Training Quantization},
author={Sai Sanjeet and Ian Colbert and Pablo Monteagudo-Lago and Giuseppe Franco and Yaman Umuroglu and Nicholas J. Fraser},
booktitle={Forty-third International Conference on Machine Learning},
year={2026},
url={https://openreview.net/forum?id=nvehxSdMqg}
}

@inproceedings{
ashkboos2024quarot,
title={{QuaRot}: Outlier-Free 4-Bit Inference in Rotated {LLM}s},
author={Saleh Ashkboos and Amirkeivan Mohtashami and Maximilian L. Croci and Bo Li and Pashmina Cameron and Martin Jaggi and Dan Alistarh and Torsten Hoefler and James Hensman},
booktitle={The Thirty-eighth Annual Conference on Neural Information Processing Systems},
year={2024},
}

@article{zhang2023post,
  title={Post-training quantization for neural networks with provable guarantees},
  author={Zhang, Jinjie and Zhou, Yixuan and Saab, Rayan},
  journal={SIAM journal on mathematics of data science},
  volume={5},
  number={2},
  pages={373--399},
  year={2023},
  publisher={SIAM}
}

@inproceedings{
liu2025spinquant,
title={{SpinQuant}: {LLM} Quantization with Learned Rotations},
author={Zechun Liu and Changsheng Zhao and Igor Fedorov and Bilge Soran and Dhruv Choudhary and Raghuraman Krishnamoorthi and Vikas Chandra and Yuandong Tian and Tijmen Blankevoort},
booktitle={The Thirteenth International Conference on Learning Representations},
year={2025},
}

@misc{alvarez2025nvfp4,
  title   = {Introducing {NVFP4} for Efficient and Accurate Low-Precision Inference},
  author  = {Alvarez, Eduardo and Almog, Omri and Chung, Eric and Layton, Simon and Stosic, Dusan and Krashinsky, Ronny and Aubrey, Kyle},
  year    = {2025},
  month   = jun,
  note    = {NVIDIA Developer Blog}
}

@article{zhang2026beacon,
  title={Beacon: Post-Training Quantization with Integrated Grid Selection},
  author={Zhang, Shihao and Saab, Rayan},
  journal={IEEE Signal Processing Letters},
  year={2026},
  publisher={IEEE}
}

@article{zhang2025comq,
  title={{COMQ}: A backpropagation-free algorithm for post-training quantization},
  author={Zhang, Aozhong and Yang, Zi and Wang, Naigang and Qi, Yingyong and Xin, Jack and Li, Xin and Yin, Penghang},
  journal={IEEE Access},
  year={2025},
  publisher={IEEE}
}

@inproceedings{nagel2020up,
  title={Up or down? {Adaptive} rounding for post-training quantization},
  author={Nagel, Markus and Amjad, Rana Ali and Van Baalen, Mart and Louizos, Christos and Blankevoort, Tijmen},
  booktitle={International conference on machine learning},
  pages={7197--7206},
  year={2020},
  organization={PMLR}
}

@inproceedings{hubara2021accurate,
  title={Accurate post training quantization with small calibration sets},
  author={Hubara, Itay and Nahshan, Yury and Hanani, Yair and Banner, Ron and Soudry, Daniel},
  booktitle={International conference on machine learning},
  pages={4466--4475},
  year={2021},
  organization={PMLR}
}

@inproceedings{
li2021brecq,
title={{BRECQ}: Pushing the Limit of Post-Training Quantization by Block Reconstruction},
author={Yuhang Li and Ruihao Gong and Xu Tan and Yang Yang and Peng Hu and Qi Zhang and Fengwei Yu and Wei Wang and Shi Gu},
booktitle={International Conference on Learning Representations},
year={2021},
}

@inproceedings{cheng2024optimize,
  title={Optimize weight rounding via signed gradient descent for the quantization of llms},
  author={Cheng, Wenhua and Zhang, Weiwei and Shen, Haihao and Cai, Yiyang and He, Xin and Kaokao, Lv and Liu, Yi},
  booktitle={Findings of the Association for Computational Linguistics: EMNLP 2024},
  pages={11332--11350},
  year={2024}
}

@inproceedings{
shao2024omniquant,
title={{OmniQuant}: Omnidirectionally Calibrated Quantization for Large Language Models},
author={Wenqi Shao and Mengzhao Chen and Zhaoyang Zhang and Peng Xu and Lirui Zhao and Zhiqian Li and Kaipeng Zhang and Peng Gao and Yu Qiao and Ping Luo},
booktitle={The Twelfth International Conference on Learning Representations},
year={2024},
}

@article{frantar2022optimal,
  title={Optimal brain compression: A framework for accurate post-training quantization and pruning},
  author={Frantar, Elias and Alistarh, Dan},
  journal={Advances in Neural Information Processing Systems},
  volume={35},
  pages={4475--4488},
  year={2022}
}

@article{lybrand2021greedy,
  title={A greedy algorithm for quantizing neural networks},
  author={Lybrand, Eric and Saab, Rayan},
  journal={Journal of Machine Learning Research},
  volume={22},
  number={156},
  pages={1--38},
  year={2021}
}

@inproceedings{
egiazarian2026bridging,
title={Bridging the Gap Between Promise and Performance for Microscaling {FP}4 Quantization},
author={Vage Egiazarian and Roberto L. Castro and Denis Kuznedelev and Andrei Panferov and Eldar Kurtic and Shubhra Pandit and Alexandre Noll Marques and Mark Kurtz and Saleh Ashkboos and Torsten Hoefler and Dan Alistarh},
booktitle={The Fourteenth International Conference on Learning Representations},
year={2026},
}

@article{lin2024awq,
  title={{AWQ}: Activation-aware weight quantization for on-device llm compression and acceleration},
  author={Lin, Ji and Tang, Jiaming and Tang, Haotian and Yang, Shang and Chen, Wei-Ming and Wang, Wei-Chen and Xiao, Guangxuan and Dang, Xingyu and Gan, Chuang and Han, Song},
  journal={Proceedings of machine learning and systems},
  volume={6},
  pages={87--100},
  year={2024}
}

@article{clark2018think,
  title={Think you have solved question answering? try {ARC}, the {AI2} reasoning challenge},
  author={Clark, Peter and Cowhey, Isaac and Etzioni, Oren and Khot, Tushar and Sabharwal, Ashish and Schoenick, Carissa and Tafjord, Oyvind},
  journal={arXiv preprint arXiv:1803.05457},
  year={2018}
}

@inproceedings{zellers2019hellaswag,
  title={Hellaswag: Can a machine really finish your sentence?},
  author={Zellers, Rowan and Holtzman, Ari and Bisk, Yonatan and Farhadi, Ali and Choi, Yejin},
  booktitle={Proceedings of the 57th annual meeting of the association for computational linguistics},
  pages={4791--4800},
  year={2019}
}

@article{grattafiori2024llama,
  title={The Llama 3 herd of models},
  author={Grattafiori, Aaron and Dubey, Abhimanyu and Jauhri, Abhinav and Pandey, Abhinav and Kadian, Abhishek and Al-Dahle, Ahmad and Letman, Aiesha and Mathur, Akhil and Schelten, Alan and Vaughan, Alex and others},
  journal={arXiv preprint arXiv:2407.21783},
  year={2024}
}

@misc{qwen2025qwen25technicalreport,
      title={{Qwen2.5 Technical Report}}, 
      author={{Qwen Team} and An Yang and Baosong Yang and Beichen Zhang and Binyuan Hui and Bo Zheng and Bowen Yu and Chengyuan Li and Dayiheng Liu and Fei Huang and Haoran Wei and Huan Lin and Jian Yang and Jianhong Tu and Jianwei Zhang and Jianxin Yang and Jiaxi Yang and Jingren Zhou and Junyang Lin and Kai Dang and Keming Lu and Keqin Bao and Kexin Yang and Le Yu and Mei Li and Mingfeng Xue and Pei Zhang and Qin Zhu and Rui Men and Runji Lin and Tianhao Li and Tianyi Tang and Tingyu Xia and Xingzhang Ren and Xuancheng Ren and Yang Fan and Yang Su and Yichang Zhang and Yu Wan and Yuqiong Liu and Zeyu Cui and Zhenru Zhang and Zihan Qiu},
      year={2025},
      eprint={2412.15115},
      archivePrefix={arXiv},
      primaryClass={cs.CL},
}

@misc{lighteval,
  author = {Habib, Nathan and Fourrier, Clémentine and Kydlíček, Hynek and Wolf, Thomas and Tunstall, Lewis},
  title = {{LightEval}: A lightweight framework for LLM evaluation},
  year = {2023},
  version = {0.11.0},
  url = {https://github.com/huggingface/lighteval}
}

@inproceedings{
merity2017pointer,
title={Pointer Sentinel Mixture Models},
author={Stephen Merity and Caiming Xiong and James Bradbury and Richard Socher},
booktitle={International Conference on Learning Representations},
year={2017},
}

@inproceedings{wolf-etal-2020-transformers,
  title={Transformers: State-of-the-art natural language processing},
  author={Wolf, Thomas and Debut, Lysandre and Sanh, Victor and Chaumond, Julien and Delangue, Clement and Moi, Anthony and Cistac, Pierric and Rault, Tim and Louf, R{\'e}mi and Funtowicz, Morgan and others},
  booktitle={Proceedings of the 2020 conference on empirical methods in natural language processing: system demonstrations},
  pages={38--45},
  year={2020}
}

@inproceedings{hassibi1993optimal,
  title={Optimal brain surgeon and general network pruning},
  author={Hassibi, Babak and Stork, David G and Wolff, Gregory J},
  booktitle={IEEE international conference on neural networks},
  pages={293--299},
  year={1993},
  organization={IEEE}
}

@article{zhang2025provable,
  title={Provable post-training quantization: Theoretical analysis of {OPTQ} and {Qronos}},
  author={Zhang, Haoyu and Zhang, Shihao and Colbert, Ian and Saab, Rayan},
  journal={arXiv preprint arXiv:2508.04853},
  year={2025}
}

@article{nair2024cdquant,
  publtype={informal},
  author={Pranav Ajit Nair and Arun Sai Suggala},
  title={{CDQuant}: Accurate Post-training Weight Quantization of Large Pre-trained Models using Greedy Coordinate Descent},
  year={2024},
  cdate={1704067200000},
  journal={CoRR},
  volume={abs/2406.17542},
}

@misc{gptq,
  author  = {IST-DASLab},
  title   = {{GPTQ}},
  url     = {https://github.com/ist-daslab/gptq},
  year    = {2022}
}

@misc{pappalardo2025xilinx,
  title={Xilinx/brevitas: Release v0. 12.0},
  author={Pappalardo, Alessandro and Franco, Giuseppe and Colbert, Ian and Grob, Fabian and Costigan, Timothy and Savolainen, Oscar and Stoian, Andrei and Gerdelan, Anton and Umuroglu, Yaman and Paine, Tim and others},
  journal={Zenodo},
  year={2025}
}

@inproceedings{
liu2026paretoq,
title={{ParetoQ}: Improving Scaling Laws in Extremely Low-bit {LLM} Quantization},
author={Zechun Liu and Changsheng Zhao and Hanxian Huang and Sijia Chen and Jing Zhang and Jiawei Zhao and Scott Roy and Lisa Jin and Yunyang Xiong and Yangyang Shi and Lin Xiao and Yuandong Tian and Bilge Soran and Raghuraman Krishnamoorthi and Tijmen Blankevoort and Vikas Chandra},
booktitle={The Thirty-ninth Annual Conference on Neural Information Processing Systems},
year={2026},
}

@inproceedings{
li2025gptaq,
title={{GPTAQ}: Efficient Finetuning-Free Quantization for Asymmetric Calibration},
author={Yuhang Li and Ruokai Yin and Donghyun Lee and Shiting Xiao and Priyadarshini Panda},
booktitle={Forty-second International Conference on Machine Learning},
year={2025},
}

@inproceedings{Ansel_PyTorch_2_Faster_2024,
author = {Ansel, Jason and Yang, Edward and He, Horace and Gimelshein, Natalia and Jain, Animesh and Voznesensky, Michael and Bao, Bin and Bell, Peter and Berard, David and Burovski, Evgeni and Chauhan, Geeta and Chourdia, Anjali and Constable, Will and Desmaison, Alban and DeVito, Zachary and Ellison, Elias and Feng, Will and Gong, Jiong and Gschwind, Michael and Hirsh, Brian and Huang, Sherlock and Kalambarkar, Kshiteej and Kirsch, Laurent and Lazos, Michael and Lezcano, Mario and Liang, Yanbo and Liang, Jason and Lu, Yinghai and Luk, CK and Maher, Bert and Pan, Yunjie and Puhrsch, Christian and Reso, Matthias and Saroufim, Mark and Siraichi, Marcos Yukio and Suk, Helen and Suo, Michael and Tillet, Phil and Wang, Eikan and Wang, Xiaodong and Wen, William and Zhang, Shunting and Zhao, Xu and Zhou, Keren and Zou, Richard and Mathews, Ajit and Chanan, Gregory and Wu, Peng and Chintala, Soumith},
booktitle = {29th ACM International Conference on Architectural Support for Programming Languages and Operating Systems, Volume 2 (ASPLOS '24)},
doi = {10.1145/3620665.3640366},
month = apr,
publisher = {ACM},
title = {{PyTorch 2: Faster Machine Learning Through Dynamic Python Bytecode Transformation and Graph Compilation}},
year = {2024}
}

@techreport{ocp2024mx,
  title        = {{OCP} Microscaling Formats ({MX}) Specification v1.0},
  author       = {{Open Compute Project Foundation}},
  year         = {2024},
  institution  = {Open Compute Project},
  note         = {Accessed: 2025-07-15}
}

@inproceedings{hassibi2002expected,
  title={On the expected complexity of integer least-squares problems},
  author={Hassibi, Babak and Vikalo, Haris},
  booktitle={2002 IEEE International Conference on Acoustics, Speech, and Signal Processing},
  volume={2},
  pages={II--1497},
  year={2002},
  organization={IEEE}
}

@inproceedings{
Esser2020LEARNED,
title={LEARNED STEP SIZE QUANTIZATION},
author={Steven K. Esser and Jeffrey L. McKinstry and Deepika Bablani and Rathinakumar Appuswamy and Dharmendra S. Modha},
booktitle={International Conference on Learning Representations},
year={2020},
}

@inproceedings{idelbayev2021optimal,
  title={Optimal quantization using scaled codebook},
  author={Idelbayev, Yerlan and Molchanov, Pavlo and Shen, Maying and Yin, Hongxu and Carreira-Perpin{\'a}n, Miguel A and Alvarez, Jose M},
  booktitle={2021 IEEE/CVF Conference on Computer Vision and Pattern Recognition (CVPR)},
  pages={12090--12099},
  year={2021},
  organization={IEEE}
}

@article{wu2020easyquant,
  title={{EasyQuant}: Post-training quantization via scale optimization},
  author={Wu, Di and Tang, Qi and Zhao, Yongle and Zhang, Ming and Fu, Ying and Zhang, Debing},
  journal={arXiv preprint arXiv:2006.16669},
  year={2020}
}

@inproceedings{
hendrycks2021measuring,
title={Measuring Massive Multitask Language Understanding},
author={Dan Hendrycks and Collin Burns and Steven Basart and Andy Zou and Mantas Mazeika and Dawn Song and Jacob Steinhardt},
booktitle={International Conference on Learning Representations},
year={2021},
url={https://openreview.net/forum?id=d7KBjmI3GmQ}
}

@article{zhou2023instruction,
  title={Instruction-following evaluation for large language models},
  author={Zhou, Jeffrey and Lu, Tianjian and Mishra, Swaroop and Brahma, Siddhartha and Basu, Sujoy and Luan, Yi and Zhou, Denny and Hou, Le},
  journal={arXiv preprint arXiv:2311.07911},
  year={2023}
}

@inproceedings{bai2024longbench,
    title = "{L}ong{B}ench: A Bilingual, Multitask Benchmark for Long Context Understanding",
    author = "Bai, Yushi  and
      Lv, Xin  and
      Zhang, Jiajie  and
      Lyu, Hongchang  and
      Tang, Jiankai  and
      Huang, Zhidian  and
      Du, Zhengxiao  and
      Liu, Xiao  and
      Zeng, Aohan  and
      Hou, Lei  and
      Dong, Yuxiao  and
      Tang, Jie  and
      Li, Juanzi",
    editor = "Ku, Lun-Wei  and
      Martins, Andre  and
      Srikumar, Vivek",
    booktitle = "Proceedings of the 62nd Annual Meeting of the Association for Computational Linguistics (Volume 1: Long Papers)",
    month = aug,
    year = "2024",
    address = "Bangkok, Thailand",
    publisher = "Association for Computational Linguistics",
    url = "https://aclanthology.org/2024.acl-long.172/",
    doi = "10.18653/v1/2024.acl-long.172",
    pages = "3119--3137",
}

@article{raffel2020exploring,
  author  = {Colin Raffel and Noam Shazeer and Adam Roberts and Katherine Lee and Sharan Narang and Michael Matena and Yanqi Zhou and Wei Li and Peter J. Liu},
  title   = {Exploring the Limits of Transfer Learning with a Unified Text-to-Text Transformer},
  journal = {Journal of Machine Learning Research},
  year    = {2020},
  volume  = {21},
  number  = {140},
  pages   = {1--67},
  url     = {http://jmlr.org/papers/v21/20-074.html}
}

@inproceedings{
penedo2024fineweb,
title={The FineWeb Datasets: Decanting the Web for the Finest Text Data at Scale},
author={Guilherme Penedo and Hynek Kydl{\'\i}{\v{c}}ek and Loubna Ben allal and Anton Lozhkov and Margaret Mitchell and Colin Raffel and Leandro Von Werra and Thomas Wolf},
booktitle={The Thirty-eight Conference on Neural Information Processing Systems Datasets and Benchmarks Track},
year={2024},
url={https://openreview.net/forum?id=n6SCkn2QaG}
}

@article{team2024gemma,
  title={Gemma 2: Improving open language models at a practical size},
  author={Team, Gemma and Riviere, Morgane and Pathak, Shreya and Sessa, Pier Giuseppe and Hardin, Cassidy and Bhupatiraju, Surya and Hussenot, L{\'e}onard and Mesnard, Thomas and Shahriari, Bobak and Ram{\'e}, Alexandre and others},
  journal={arXiv preprint arXiv:2408.00118},
  year={2024}
}

\clearpage
\appendix
\clearpage
\noindent\rule{\textwidth}{0.6pt}\par\vspace{4pt}
\noindent{\Large\bfseries Appendix (Supplementary Materials)}\par\vspace{2pt}
\noindent\rule{\textwidth}{0.6pt}\par\vspace{6pt}

\startcontents[appendix]
\titlecontents{section}[1.8em]
  {\addvspace{12pt}\large\bfseries}
  {\contentslabel{1.8em}}
  {\hspace*{-1.8em}}
  {\hfill\textcolor{black}{\contentspage}}
\printcontents[appendix]{}{1}{\setcounter{tocdepth}{1}}

\vspace{10pt}\noindent\rule{\textwidth}{0.6pt}
\clearpage

\section{Implementation Details}\label{Implementation-det}
Algorithm~\ref{alg:scale-sweep} summarizes the optimization procedure described in Section~\ref{sec:channel-wise} for a single channel. In this implementation the sweep proceeds from large to small $s$ (right to left), processing positive scales first and then negative scales, if permitted. We implemented~\ScaleAlgorithmShort~using PyTorch 2.8~\citep{Ansel_PyTorch_2_Faster_2024}, and benefited from batched operations to parallelize it in GPU across channels in a given layer, and simultaneously across groups for the independent approximation (see Section~\ref{sec:group-wise}).

\begin{algorithm}[t]
\caption{PiSO: Piecewise Scale Optimization}\label{alg:scale-sweep}
\begin{algorithmic}[1]
\REQUIRE Weight vector $\mathbf{w} \in \mathbb{R}^D$, matrices $\mathbf{H}\in\mathbb{R}^{D\times D}$, $\mathbf{G}\in\mathbb{R}^{D\times D}$, sorted grid $\G = \{g_1, \ldots, g_L\}$ s.t. $0\in \G$
\ENSURE Optimal scale $\hat{s}$
\STATE Compute transition scales $t_{i,\ell} \gets \frac{2w_i}{g_\ell + g_{\ell+1}}$ for all $i \in [D]$, $\ell \in [L{-}1]$
\STATE Include $0$ in the boundary list, keeping only one occurrence of $0$ \COMMENT{Sign-change boundary}
\STATE Sort all boundaries in decreasing order: $\tau_1 > \cdots > \tau_T$, recording the corresponding $(i,\ell)$ pairs for each nonzero boundary
\STATE Set $\tau_{T+1}\gets-\infty$
\STATE $\mathbf{p} \gets \mathbf{G}\mathbf{w}$ \COMMENT{Precompute; does not change during sweep}
\STATE $\alpha \gets 0$, \; $\beta \gets 0$, \; $\mathbf{h} \gets \mathbf{0}$ \COMMENT{$\alpha = \mathbf{w}^\top\mathbf{G}^\top q$, \; $\beta = q^\top\mathbf{H}q$, \; $\mathbf{h} = \mathbf{H}q$}
\STATE $\hat{s} \gets \texttt{undef}$, \; $\hat{\phi} \gets +\infty$
\STATE Precompute sign-change transition values $(\alpha^{\mathrm{tr}}, \beta^{\mathrm{tr}}, \mathbf{h}^{\mathrm{tr}})$
\FOR{$j = 1, \ldots, T$}
        \IF{$\tau_j=0$}
    \STATE $\alpha \gets \alpha^{\mathrm{tr}}$, \; $\beta \gets \beta^{\mathrm{tr}}$, \; $\mathbf{h} \gets \mathbf{h}^{\mathrm{tr}}$ \COMMENT{When entering the negative scales sweep}
    \ELSE
        \STATE Let $i$ be the index whose $q_i$ changes at $\tau_j$ by $\delta = q_i^{\mathrm{new}} - q_i^{\mathrm{old}}$
        \STATE $\alpha \gets \alpha + \delta \cdot p_i$ \COMMENT{$\mathcal{O}(1)$}
        \STATE $\beta \gets \beta + 2\delta \cdot h_i + \delta^2 \cdot H_{ii}$ \COMMENT{$\mathcal{O}(1)$}
        \STATE $\mathbf{h} \gets \mathbf{h} + \delta \cdot \mathbf{H}_{:,i}$ \COMMENT{$\mathcal{O}(D)$ in the general case; $\mathcal{O}(1)$ if $\mathbf{H}$ is diagonal}
    \ENDIF
    \IF{$\beta=0$}
    \STATE Choose any finite $s^\ast_j \in (\tau_{j+1},\tau_j)$ \COMMENT{The objective is constant on this interval}
\ELSE
    \STATE $s^\ast_j \gets \mathrm{clamp}\bigl(\alpha/\beta,\tau_{j+1},\tau_j-\varepsilon\bigr)$ \textbf{if} $\tau_j>0$, \textbf{else} $\mathrm{clamp}\bigl(\alpha/\beta,\tau_{j+1}+\varepsilon,\tau_j\bigr)$
\ENDIF
    \STATE $\phi_j \gets (s^\ast_j)^2\, \beta - 2\, s^\ast_j\, \alpha$
    \IF{$\phi_j < \hat{\phi}$ \textbf{and} $|\tau_j - \tau_{j+1}| > \varepsilon$}
        \STATE $\hat{s} \gets s^\ast_j$, \; $\hat{\phi} \gets \phi_j$ \COMMENT{Skip degenerate intervals}
    \ENDIF
\ENDFOR
\RETURN $\hat{s}$
\end{algorithmic}
\end{algorithm}

In the following, we provide additional details on how corner cases arising in the scale optimization are handled.

\paragraph{Degenerate Intervals.}
\label{appendix:degenerate_intervals}

Two transition scales coincide ($t_{i,\ell} = t_{j,\ell'}$) whenever $w_i / w_j = m_\ell / m_{\ell'}$ for grid midpoints $m_\ell, m_{\ell'}$. When this happens, the interval between them has zero length. At such a point several entries of $\mathbf{q}$ attempt to change simultaneously, but since our sweep processes them one at a time, the intermediate state does not correspond to any realizable rounding. We detect these degenerate intervals by checking $|\tau_{j-1} - \tau_j| < \varepsilon$ and skip the error update on them.

\paragraph{Handling the scale sign change.}
\label{appendix:sign_change}

For quantization grids that are symmetric around zero, i.e. they verify the property that if $q \in \G \Rightarrow -q \in \G$, only positive scales $s > 0$ need to be considered to guarantee optimality. This holds since $s \cdot q(w; s) = (-s) \cdot q(w; -s)$. However, for asymmetric quantization grids, such as those of integer datatypes, the optimal scale may be negative. If one wants to allow for such negative scales,  the algorithm needs to handle the sign change at $s=0$, to account for the simultaneous changes in the unscaled grid assignments. Precisely, when the sweep crosses from positive to negative scales, the ratio $w_i / s$ flips sign for \emph{every} non-zero weight simultaneously. Indeed, apart from degenerate intervals, crossing $s=0$ is the only point where multiple grid assignments simultaneously change, and incremental updates do not apply. We handle it by precomputing a \emph{transition vector} $\mathbf{q}^{\mathrm{tr}}$ that assigns each weight to its limiting grid value as $s\to0^-$
\begin{equation*}
q^{\mathrm{tr}}_i = \begin{cases} g_L & \text{if } w_i < 0, \\ g_1 & \text{if } w_i > 0, \\ 0 & \text{if } w_i = 0. \end{cases}
\end{equation*}
together with the corresponding $\alpha^{\mathrm{tr}}$, $\beta^{\mathrm{tr}}$, and $\mathbf{h}^{\mathrm{tr}}$. When the sweep first enters a negative-scale interval, the running state is replaced by these precomputed values, after which incremental updates resume normally for $s < 0$. Under Remark~\ref{remark:tie-breaking}, the assignment is constant on $(a,b]$ between consecutive negative transition scales $a<b<0$. An excluded left endpoint is approached from inside the interval as $s\to a^+$.

\paragraph{Group processing order.}\label{group-proc-order}
Greedy error correction algorithms typically process the weights of a given layer in
descending order of the diagonal entries of $\mathbf{H} := \tilde{\mathbf{X}}^\top\tilde{\mathbf{X}}$~\citep{gptq, zhang2026qronos, pappalardo2025xilinx}:
for a weight vector $\mathbf{w}\in\R^D$, the entry $w_i$ is quantized earlier the
larger its associated diagonal element $H_{ii}$. As mentioned in
Section~\ref{sec:group-wise}, in the group-sequential variant of \ScaleAlgorithm{}
the order in which the groups are processed matters, since later groups cannot influence
earlier ones. We adopt a natural group-wise analogue of the standard
weight-level rule: the importance of group~$k$ is
$\sum_{i \in \mathcal{I}_k^G} H_{ii}$, the sum of the diagonal Hessian entries
of its weights, and groups are processed in descending order of this score.

When applying PiSO using the group-wise interleaved integration
with error correction described in the \emph{Interleaved group-wise
optimization} strategy (Section~\ref{sec:inte}), the scale of each
group must be optimized \emph{just before} error correction quantizes its
weights, which forces all weights of a given group to be processed, before any weight of a different group. The standard $\mathbf{H}$-based ordering used
in GPTQ implementations breaks this contiguity. To overcome this we use a \emph{group-aware} reordering: groups are sorted in descending
order of their aggregated importance $\sum_{i \in \mathcal{I}_k^G} H_{ii}$
defined above, and within each group columns are sorted in the usual way by
$H_{ii}$. We impose this constraint in the reordering only whenever \ScaleAlgorithm{} is
combined with GPTQ at the tightest integration level
(GPTQ~$\boxdot$~\ScaleAlgorithmShort~in the relevant Tables).

\section{Integration with Greedy Error Correction Algorithms}
\label{app:interleaved-group}

\begin{figure}[t]
  \centering
  \includegraphics[width=\textwidth]{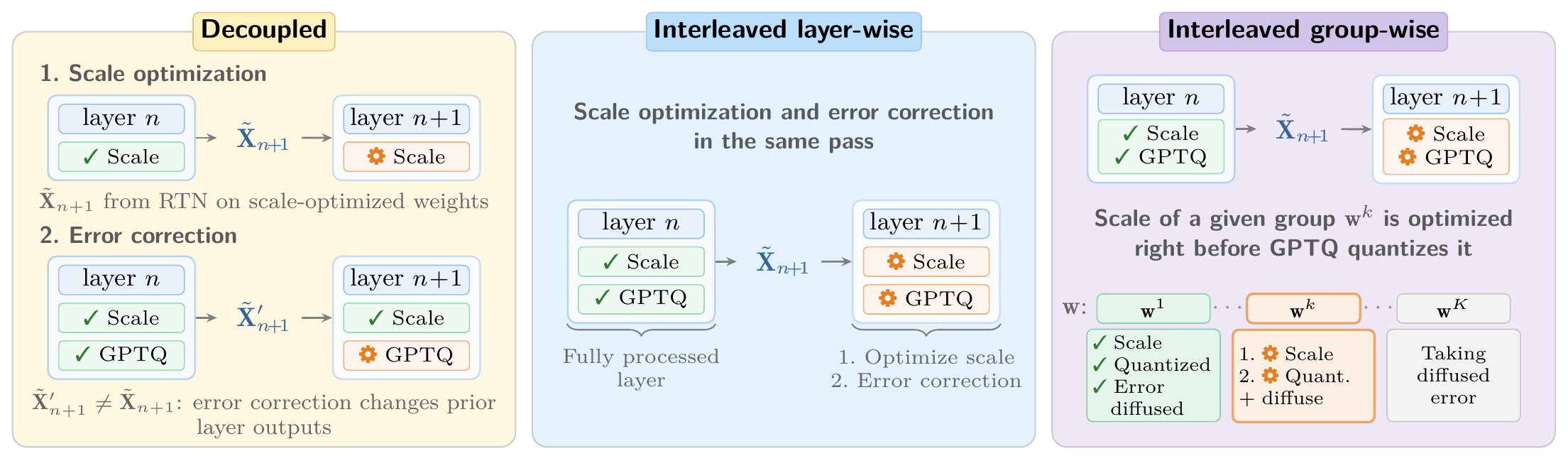}

  \caption{Overview of the three integration strategies of~\ScaleAlgorithm~with error correction algorithms (e.g., GPTQ). The strategies differ in how tightly scale optimization and error correction are interleaved: \emph{decoupled} optimizes all scales in a separate first pass, \emph{interleaved layer-wise} optimizes each layer's scales immediately before its error correction step, and \emph{interleaved group-wise} optimizes each group's scale immediately before error correction quantizes that group.}
  \label{fig:algorithm_gpxq_integration}
\end{figure}

Greedy error correction algorithms such as GPTQ~\citep{frantar2022gptq} and Qronos~\citep{zhang2026qronos} process the $D$ weights sequentially in two alternating steps:
\begin{inparaenum}[(1)]
\item error correction, in which the current weight is assigned to the scaled grid value that minimizes the quantization error;
\item and error diffusion, which adjusts the remaining unquantized to compensate for the induced error. These algorithms assume that the quantization grid has been fixed in advance and, therefore, their weight-to-grid assignments are conditional on the scaling factor that was chosen beforehand. In this regard, the intractable joint optimization of the scaling factor and the weight-to-grid assignments can be approximated in two steps: optimization of the scaling factor for a RTN-surrogate of the weight-to-grid assignment, and error correction to refine the assignments conditionally on the selected grid, as described in Section~\ref{sec:inte}.

\end{inparaenum}
Following, we provide additional detail into the integration strategies. Let $\tilde{\mathbf{w}}^{t}$ denote the partially quantized state after step $t$, with
\begin{equation*}
\tilde{\mathbf{w}}^{t} = \big(\tilde{\mathbf{w}}^{t}_{\{\le t\}},\, \mathbf{w}^{t}_{\{>t\}}\big),
\qquad
\tilde{\mathbf{w}}^{0} = \mathbf{w},\qquad t \in \{0,\ldots,D\}
\end{equation*}
where $\tilde{\mathbf{w}}^{t}_{\{\le t\}} = [s_{\lceil 1/G\rceil}q_1,\ldots,s_{\lceil t/G\rceil}q_t]$ and
$\mathbf{w}^{t}_{\{>t\}} = [w^{t}_{t+1},\ldots,w^{t}_D]$.
At step $t$, error correction selects $q_t$ while holding the other coordinates fixed,
and error diffusion updates the remaining weights~\citep{zhang2026qronos}:
\begin{align}
q_t
&= \argmin_{p \in \G}
\frac{1}{2}\left\|
\mathbf{X}\mathbf{w}
- \sum_{j=1}^{t-1}s_{\lceil j/G\rceil} q_j \tilde{\mathbf{X}}_j
- s_{\lceil t/G\rceil}p \tilde{\mathbf{X}}_t
- \sum_{j=t+1}^{D} w^{t-1}_j \tilde{\mathbf{X}}_j
\right\|^2,
\label{eq:ec}\\
\mathbf{w}^{t}_{\{>t\}}
&= \argmin_{\mathbf{v} \in \mathbb{R}^{D-t}}
\frac{1}{2}\left\|
\mathbf{X}\mathbf{w}
- \sum_{j=1}^t s_{\lceil j/G\rceil}q_j \tilde{\mathbf{X}}_j
- \sum_{j=t+1}^{D} v_j \tilde{\mathbf{X}}_j
\right\|^2.
\label{eq:ed}
\end{align}
These subproblems admit closed forms: error correction (Equation~\ref{eq:ec}) is a scalar
projection onto $\G$, and error diffusion (Equation~\ref{eq:ed}) is a least-squares update over the
unquantized weights.

In group-wise quantization,~\ScaleAlgorithmShort~can be run each time the error correction process reaches a group boundary. Let $t = (k{-}1)G+1$ be the first index of group $k$ in the processing order. Before the weights of group $k$ are quantized,~\ScaleAlgorithmShort~optimizes $s_k$ using the current partially quantized state $\tilde{\mathbf{w}}^{t-1}$. The form of this optimization depends on the group-wise heuristic:

\paragraph{Independent groups interleaved with error correction.}
Cross-group interactions are discarded and the scale of group $k$ is optimized in isolation, analogously to what is proposed in Equation~\ref{eq:independent-group}. The key difference is that the weight vector has been modified by error diffusion from preceding groups
\begin{equation}\label{eq:ec-independent-group}
    s_k^\ast=
      \argmin_{s_{k}}\frac{1}{2}\left\|
        \mathbf{X}_{\{k\}}\mathbf{w}^{t-1}_{\{k\}}
        - s_{k}\, \tilde{\mathbf{X}}_{\{k\}}\, q(\mathbf{w}^{t-1}_{\{k\}};\, s_{k})
      \right\|^2.
\end{equation}

\paragraph{Sequential groups interleaved with error correction.}
The scale of group $k$ is optimized under the layer-wise output reconstruction objective, accounting for the quantization error of the preceding groups, analogously to Equation~\ref{eq:sequential-group}. The key difference is that, under interleaving, the preceding groups $i < k$ have been quantized by error correction rather than by round-to-nearest, and the remaining groups $i > k$ have received the error diffusion updates
\begin{equation}\label{eq:ec-sequential-group}
    s_k^\ast = \argmin_{s_{k}} \left\| \mathbf{X}\,\mathbf{w} - \tilde{\mathbf{X}}\left[\tilde{\mathbf{w}}_{\{1\}}^{t-1},\ldots,s_{k}\,q(\mathbf{w}_{\{k\}}^{t-1};s_{k}),\ldots,\mathbf{w}_{\{K\}}^{t-1}\right]^\top \right\|^2.
\end{equation}

Following, we provide the pseudocode (Algorithm~\ref{alg:integration-schedules}) for the three integration strategies of \ScaleAlgorithm{} with error correction algorithms proposed in the main text (Section~\ref{sec:inte}). Figure~\ref{fig:algorithm_gpxq_integration} summarizes these graphically.

\paragraph{Algorithm notation.}
Let $\iota=1,\ldots,P$ index the layers in processing order. The reference weights $\mathbf{W}_{\iota}$ and calibration inputs $\mathbf{X}_{\iota}$ are cached once from the unquantized model, and $\mathcal{M}_{\iota}$ denotes the model prefix with quantized weights $\tilde{\mathbf{W}}_{1},\ldots,\tilde{\mathbf{W}}_{\iota}$. We write $\tilde{\mathbf{X}}$ for the inputs of the layer currently being processed. Three operators act on this state, all using $\tilde{\mathbf{X}}$ implicitly:
\begin{itemize}[leftmargin=*,nosep]
 \item $\mathsf{A}(\mathcal{M}_{\iota-1})$ returns the calibration inputs reaching layer $\iota$ through the prefix $\mathcal{M}_{\iota-1}$.
 \item $\mathsf{S}(\mathbf{V})$ returns the scales selected by \ScaleAlgorithm{} on the weights $\mathbf{V}$. The group-local call $\mathsf{S}_{k}(\mathbf{V})$ returns the scale of group $k$ under either the independent objective (Equation~\ref{eq:ec-independent-group}) or the sequential one (Equation~\ref{eq:ec-sequential-group}).
 \item $\mathsf{C}(\mathbf{V};\mathbf{s}_{\iota})$ returns the result of GPTQ or Qronos applied to $\mathbf{V}$ at fixed scales. The group-local call $\mathsf{C}_{k}(\mathbf{V};\mathbf{s}_{\iota,k})$ quantizes only group $k$ and diffuses its error into the remaining unquantized weights, including later groups; previously quantized weights stay fixed.
\end{itemize}

\begin{algorithm}[htbp]
\caption{Three granularities of interleaving scale selection and greedy error correction. $\mathsf{S}$ selects scales and $\mathsf{C}$ applies error correction for fixed scales.}
\label{alg:integration-schedules}
\begingroup
\footnotesize
\noindent\textbf{Shared setup:} cache $\mathbf{W}_{\iota},\mathbf{X}_{\iota}$ from the unquantized model. Each panel returns the quantized weights $\tilde{\mathbf{W}}_{1},\ldots,\tilde{\mathbf{W}}_{P}$.\par\medskip
\noindent
\begin{minipage}[t]{0.32\linewidth}
\centering\textbf{(a) Decoupled: $+$}\par
\raggedright
\begin{algorithmic}
 \FOR{$\iota=1,\ldots,P$}
  \STATE $\tilde{\mathbf{X}}\gets\mathsf{A}(\mathcal{M}_{\iota-1})$
  \STATE $\mathbf{s}_{\iota}\gets\mathsf{S}(\mathbf{W}_{\iota})$
  \STATE $\tilde{\mathbf{W}}_{\iota}\gets\mathcal{Q}(\mathbf{W}_{\iota};\mathbf{s}_{\iota})$
 \ENDFOR
 \FOR{$\iota=1,\ldots,P$}
  \STATE $\tilde{\mathbf{X}}\gets\mathsf{A}(\mathcal{M}_{\iota-1})$
  \STATE $\tilde{\mathbf{W}}_{\iota}\gets\mathsf{C}(\mathbf{W}_{\iota};\mathbf{s}_{\iota})$
 \ENDFOR
\end{algorithmic}
\end{minipage}\hfill
\begin{minipage}[t]{0.32\linewidth}
\centering\textbf{(b) Interleaved layer-wise: $\odot$}\par
\raggedright
\begin{algorithmic}
 \FOR{$\iota=1,\ldots,P$}
  \STATE $\tilde{\mathbf{X}}\gets\mathsf{A}(\mathcal{M}_{\iota-1})$
  \STATE $\mathbf{s}_{\iota}\gets\mathsf{S}(\mathbf{W}_{\iota})$
  \STATE $\tilde{\mathbf{W}}_{\iota}\gets\mathsf{C}(\mathbf{W}_{\iota};\mathbf{s}_{\iota})$
 \ENDFOR
\end{algorithmic}
\end{minipage}\hfill
\begin{minipage}[t]{0.32\linewidth}
\centering\textbf{(c) Interleaved group-wise: $\boxdot$}\par
\raggedright
\begin{algorithmic}
 \FOR{$\iota=1,\ldots,P$}
  \STATE $\tilde{\mathbf{X}}\gets\mathsf{A}(\mathcal{M}_{\iota-1})$
  \STATE $\tilde{\mathbf{W}}_{\iota}\gets\mathbf{W}_{\iota}$
  \FOR{$k=1,\ldots,K_{\iota}$}
   \STATE $\mathbf{s}_{\iota,k}\gets\mathsf{S}_{k}(\tilde{\mathbf{W}}_{\iota})$
   \STATE $\tilde{\mathbf{W}}_{\iota}\gets\mathsf{C}_{k}(\tilde{\mathbf{W}}_{\iota};\mathbf{s}_{\iota,k})$
  \ENDFOR
 \ENDFOR
\end{algorithmic}
\end{minipage}
\par\medskip
\noindent In (a), the RTN-quantized weights from the scale selection loop are discarded prior to error correction. In (c), $K_{\iota}$ is the number of groups per channel in layer $\iota$; groups are processed contiguously in the group-aware order of Appendix~\ref{group-proc-order}.
\endgroup
\end{algorithm}

\section{Proofs}
\label{appendix:proofs}

This appendix collects the formal statements and proofs of results referenced in the main text. We begin by introducing the tie-breaking rule for the rounding operator that will be used through out this appendix. Then we derive the optimal functional form of the unscaled assignment $q$, as a function of the scaling factor $s \in \mathbb{R} \setminus \{0\}$, in the sense of weight mean-squared error minimization:

\begin{remark}[Tie-breaking rule]
\label{remark:tie-breaking}
We assume throughout that $0\in\G$, as is the case for the standard grids (INT2, $\ldots$, INT$N$, NVFP4, MXFP4, etc.) and for all grids considered in this work. Ties are broken towards zero: writing $\mathcal{N}(y):=\argmin_{g\in\G}|g-y|$ for the set of grid values nearest to $y\in\mathbb{R}$, we define
\begin{equation*}
\round{y}_\G := \argmin_{g\in\mathcal{N}(y)}|g| .
\end{equation*}
Note that $|\mathcal{N}(y)|=2$ if and only if $y=m_\ell:=(g_\ell+g_{\ell+1})/2$ for some $\ell \in |\G|$.
\end{remark}

\begin{restatable}[Weight mean-squared error minimization is achieved by RTN and scale optimization]{proposition}{prop:optimal_mse}
\label{prop:optimal_mse}
Let $\mathbf{w} \in \mathbb{R}^{D}$ denote a weight vector, and assume that the scaling  factor $s \in \mathbb{R} \setminus \{0\}$ is shared for all the weights in the vector. Then, the \emph{data-free} optimal assignment to the grid $\mathbf{q}$ can be recovered by minimizing\footnote{We do not claim novelty for this result, although we could not find it explicitly stated as such in prior literature. We include the proof here for completeness.}
\begin{equation*}
\min_{\substack{s \in \mathbb{R} \setminus \{0\} \\\mathbf{q} \in \G^{D}}}\left\lVert \mathbf{w} -s\,\mathbf{q} \right\rVert^{2}=\min_{s \in \mathbb{R} \setminus \{0\}}\left\lVert \mathbf{w} -s\, q(\mathbf{w};s) \right\rVert^{2}, \quad q(\mathbf{w};s):=\left\lfloor\frac{\mathbf{w}}{s}\right\rceil_{\G}.
\end{equation*}
\end{restatable}

\begin{proof}
First, note that, for fixed $s \in \mathbb{R} \setminus \{0\}$, the objective $\min_{\mathbf{q} \in \G^{D}}\Vert \mathbf{w} -s\,\mathbf{q} \Vert^{2}$ is separable in each individual weight
\begin{equation*}
\Vert \mathbf{w} -s\,\mathbf{q} \Vert^{2}=\sum_{i=1}^{D}|w_{i}-s\, q_{i}|^{2}.
\end{equation*}
For each variable, since $s\neq0$, we have
\begin{equation*}
|w_i-sq_i|^2=s^2|w_i/s-q_i|^2.
\end{equation*}
Since $s^2>0$, minimizing this expression over $q_i\in\G$ amounts to choosing a grid value nearest to $w_i/s$. Thus, a minimizing assignment is
\begin{equation*}
q_i^\ast=\round{w_i/s}_{\G},
\end{equation*}
with ties resolved according to Remark~\ref{remark:tie-breaking}.
Consequently, the optimal assignment of the weights to the discrete grid can be recovered by optimizing the scaling factor for the objective:
\begin{equation*}
\min_{s \in \mathbb{R} \setminus \{0\}}\left\lVert \mathbf{w} -s\, q(\mathbf{w};s) \right\rVert^{2}, \quad q(\mathbf{w};s):=\left\lfloor\frac{\mathbf{w}}{s}\right\rceil_{\G}.
\end{equation*}
\end{proof}

The preceding result justifies the use of the round-to-nearest quantizer for fixed scale $s$. However, once calibration data enters the picture, the situation changes:

\begin{remark}
\label{rem:subopt_data_aware}
Proposition~\ref{prop:optimal_mse} does not extend to the data-aware objective
\begin{equation*}
\min_{\substack{s \in \mathbb{R} \setminus \{0\} \\\mathbf{q} \in \G^{D}}}\left\lVert \mathbf{X}\mathbf{w} -s\,\tilde{\mathbf{X}}\mathbf{q} \right\rVert^{2},
\end{equation*}
which corresponds to an integer least‑squares problem, a class of problems that is NP‑hard~\citep{hassibi2002expected}. Precisely, when the scaling factor $s \in \mathbb{R} \setminus \{0\}$ is fixed, the optimal solution \(\mathbf{q}^\ast\) is not necessarily $q(\mathbf{w};s):=\left\lfloor\mathbf{w}/s\right\rceil_{\G}$.
\end{remark}

We next analyze the expected layer-wise reconstruction error. Although Remark~\ref{rem:subopt_data_aware} shows that RTN is not optimal for the data-aware objective in general, the following decomposition reveals that the expected error on unseen data is bounded by a term involving the \emph{parameter mismatch} $\|\mathbf{w} - \mathbf{q}_{\mathbf{X}}\|^{2}$. Thus, for a given scale, this quantity is minimized by the RTN quantizer (Proposition~\ref{prop:optimal_mse}). Consequently, restricting the grid assignments to those produced by round-to-nearest controls one of the three factors in the generalization bound and provides a principled prior on the solution space of the data-aware reconstruction objective.

\begin{restatable}[Layer-wise error decomposition]{remark}{remark:layerwise_error}
\label{remark:layerwise_error}
Let $\mathbf{z} \sim \mathcal{X}$ be a random activation from the test distribution, and let  $\tilde{\mathbf{z}}$ denote the activation obtained after passing through previously quantized layers. Let $\mathbf{X}$ be a calibration dataset, and suppose $\mathbf{q}_{\mathbf{X}}$ denotes the reconstructed quantized weight vector, including its scaling factors, calibrated on $\mathbf{X}$. Then, for an unseen random activation $\mathbf{z}$, the expected layer-wise error can be decomposed as follows:
\begin{align*}
\label{eq:layer-error-error}
\mathbb{E}_{\mathbf{z}\sim\mathcal{X}}\left[
\bigl\|\mathbf{z}^\top\mathbf{w} - \tilde{\mathbf{z}}^\top\mathbf{q_{\mathbf{X}}}\bigr\|^2\right]
&=
\mathbb{E}_{\mathbf{z}\sim\mathcal{X}}\Big[
\underbrace{\bigl\|\mathbf{z}^\top(\mathbf{w}-\mathbf{q_{\mathbf{X}}})\bigr\|^2}_{(1)}
\\
&\quad
+ \underbrace{2\big\langle
\mathbf{z}^\top(\mathbf{w}-\mathbf{q_{\mathbf{X}}}),(\mathbf{z}-\tilde{\mathbf{z}})^\top\mathbf{q_{\mathbf{X}}}
\big\rangle}_{(2)}
\\
&\quad
+ \underbrace{\bigl\|(\mathbf{z}-\tilde{\mathbf{z}})^\top\mathbf{q_{\mathbf{X}}}\bigr\|^2}_{(3)}
\Big],
\end{align*}
Assuming $\tilde{\mathbf{z}} \approx \mathbf{z}$, terms (2) and (3) can be neglected, and, from~\cite[Remark~3.8]{zhang2025provable}, term (1) can be rewritten and bounded as
\begin{align*}
\mathbb{E}_{\mathbf{z}\sim\mathcal{X}}\Big[
\bigl\|\mathbf{z}^\top(\mathbf{w}-\mathbf{q_{\mathbf{X}}})\bigr\|^2\Big]&=\frac{1}{N}\|\mathbf{X}(\mathbf{w}-\mathbf{q_{\mathbf{X}}})\|^{2}+(\mathbf{w}-\mathbf{q_{\mathbf{X}}})^\top\left[\mathbb{E}_{\mathbf{z}\sim\mathcal{X}}\left[\mathbf{z}\mathbf{z}^\top\right]-\frac{1}{N}\mathbf{X}^\top\mathbf{X}\right](\mathbf{w}-\mathbf{q_{\mathbf{X}}})\\
&\le
\underbrace{\frac{1}{N}\|\mathbf{X}(\mathbf{w}-\mathbf{q_{\mathbf{X}}})\|^{2}}_{\text{(1)~calibration fit}}
\;+\;
\underbrace{\left\|
\mathbb{E}_{\mathbf{z}\sim\mathcal{X}}\!\left[\mathbf{z}\mathbf{z}^\top\right]
-
\frac{1}{N}\mathbf{X}^\top\mathbf{X}
\right\|_{\mathrm{op}}}_{\text{(2)~covariance deviation}}
\cdot
\underbrace{\big\|\mathbf{w}-\mathbf{q_{\mathbf{X}}}\big\|^{2}}_{\text{(3)~parameter mismatch}}.
\end{align*}

The bound highlights three factors governing generalization error~\citep{zhang2025provable}:
\begin{inparaenum}[(1)]
    \item calibration fit on $\mathbf X$,
    \item second-moment estimation error; and
    \item proximity between $\mathbf{w}$ and $\mathbf{q_{\mathbf{X}}}$.
\end{inparaenum}

\end{restatable}

The following sequence of results (Lemma~\ref{lemma:separable}, Lemma~\ref{lemma:closed_form_scale}, Lemma~\ref{lemma:decision_regions}, Corollary~\ref{cor:step_wise} and Proposition~\ref{prop:sweep_optimality}) establishes the theoretical foundation for the algorithm presented in Section~\ref{sec:channel-wise}. The resulting sweep identifies the global infimum of the channel-wise objective and returns an admissible scale arbitrarily close to it; in the data-free setting, the global minimum is attained exactly.

\begin{restatable}[Channel-wise separability]{lemma}{lemma:separable}
\label{lemma:separable}
Let $\mathbf{W}=[\mathbf{w}_1,\ldots,\mathbf{w}_M]$, and assign an
independent scale $s_j$ to each column $\mathbf{w}_j$. Then the objective
\begin{equation*}
E(\mathbf{s})
=
\left\|
\mathbf{XW}
-
\widetilde{\mathbf{X}}\,
\mathcal{Q}(\mathbf{W};\mathbf{s})
\right\|_F^2
\end{equation*}
decomposes as
\begin{equation*}
E(\mathbf{s})
=
\sum_{j=1}^{M}
\left\|
\mathbf{X}\mathbf{w}_j
-
\widetilde{\mathbf{X}}\,
\mathcal{Q}(\mathbf{w}_j;s_j)
\right\|^2.
\end{equation*}
Consequently, each of the channel-wise scales $s_j\in\mathbb{R}\setminus\{0\}$ can be optimized independently.
\end{restatable}
\begin{proof}
    Writing $\mathbf{W}=[\mathbf{w}_1,\ldots,\mathbf{w}_M]$ and assigning an independent scale $s_j$ to each column, the matrix products decompose column-wise as
\begin{equation*}
\mathbf{X}\mathbf{W}
= [\mathbf{X}\mathbf{w}_1,\ldots,\mathbf{X}\mathbf{w}_M],
\qquad
\tilde{\mathbf{X}}\,\mathcal{Q}(\mathbf{W};\mathbf{s})
= [\tilde{\mathbf{X}}\,\mathcal{Q}(\mathbf{w}_1;s_1),\ldots,\tilde{\mathbf{X}}\,\mathcal{Q}(\mathbf{w}_M;s_M)] .
\end{equation*}

Since $\|\mathbf{A}\|_F^2 = \sum_{j=1}^M \|\mathbf{a}_j\|^2$, the objective in Equation~\ref{eq:layer-obj} decomposes as
\begin{equation}
\label{eq:remark_separable}
E(\mathbf{s})
= \sum_{j=1}^M
\bigl\|
\mathbf{X}\mathbf{w}_j
-
\tilde{\mathbf{X}}\,\mathcal{Q}(\mathbf{w}_j; s_j)
\bigr\|^2.
\end{equation}
Hence, the objective in Equation~\ref{eq:remark_separable} separates into $M$ independent channel-wise problems, thus each $s_j$ can be optimized independently.
\end{proof}

\begin{restatable}[Closed-form for grid assignments]{lemma}{lemma:closed_form_scale}
\label{lemma:closed_form_scale}
The following result is standard and appears, e.g., as in~\cite[Proposition~2.1]{zhang2026beacon}; we restate it here for completeness. For fixed $\mathbf{q} \in \G^{D}$ with $\tilde{\mathbf{X}}\,\mathbf{q} \neq \mathbf{0}$, the objective
\begin{equation}\label{eq:channel_error}
E(s) = \|\mathbf{X}\mathbf{w} - s\,\tilde{\mathbf{X}}\, \mathbf{q}\|^2
\end{equation}
is a strictly convex quadratic in~$s$, uniquely minimized at
\begin{equation*}
s^\ast=\frac{\langle \tilde{\mathbf{X}}\mathbf{q}, \mathbf{X}\mathbf{w}\rangle}{\|\,\tilde{\mathbf{X}}\, \mathbf{q}\|^{2}}.
\end{equation*}
\end{restatable}

\begin{proof}
Expanding the squared norm:
\begin{equation*}
E(s) = \|\mathbf{X}\mathbf{w}\|^{2} - 2s\,\langle\mathbf{X}\mathbf{w},\tilde{\mathbf{X}}\, \mathbf{q}\rangle +s^{2}\|\,\tilde{\mathbf{X}}\, \mathbf{q}\|^{2}.
\end{equation*}
Since $\|\tilde{\mathbf{X}}\,\mathbf{q}\|^2 > 0$, this is a strictly convex quadratic in~$s$. Setting the derivative to zero:
\begin{align*}
\frac{dE}{ds}(s^\ast) = 0 &\iff - 2\,\langle\mathbf{X}\mathbf{w},\tilde{\mathbf{X}}\, \mathbf{q}\rangle +2s^\ast\|\,\tilde{\mathbf{X}}\, \mathbf{q}\|^{2} = 0 \\
&\iff s^\ast= \frac{\langle \tilde{\mathbf{X}}\mathbf{q}, \mathbf{X}\mathbf{w}\rangle}{\|\,\tilde{\mathbf{X}}\, \mathbf{q}\|^{2}},
\end{align*}
which is the unique global minimizer.
\end{proof}

\begin{remark}[Connection to $\alpha/\beta$ notation]
\label{remark:alpha_beta_equivalence}
The closed-form minimizer $s^\ast=\langle \tilde{\mathbf{X}}\mathbf{q}, \mathbf{X}\mathbf{w}\rangle / \|\tilde{\mathbf{X}}\mathbf{q}\|^{2}$ can be expressed in terms of the matrices $\mathbf{H} = \tilde{\mathbf{X}}^\top\tilde{\mathbf{X}}$ and $\mathbf{G} = \tilde{\mathbf{X}}^\top\mathbf{X}$ defined in Equation~\ref{eq:HG-def}. Indeed,
\begin{equation*}
\alpha(\mathbf{q}) := \langle \tilde{\mathbf{X}}\mathbf{q}, \mathbf{X}\mathbf{w}\rangle = (\tilde{\mathbf{X}}\mathbf{q})^\top (\mathbf{X}\mathbf{w}) = \mathbf{q}^\top \tilde{\mathbf{X}}^\top \mathbf{X}\, \mathbf{w} = \mathbf{q}^\top \mathbf{G}\, \mathbf{w},
\end{equation*}
and
\begin{equation*}
\beta(\mathbf{q}) := \|\tilde{\mathbf{X}}\mathbf{q}\|^{2} = (\tilde{\mathbf{X}}\mathbf{q})^\top (\tilde{\mathbf{X}}\mathbf{q}) = \mathbf{q}^\top \tilde{\mathbf{X}}^\top \tilde{\mathbf{X}}\, \mathbf{q} = \mathbf{q}^\top \mathbf{H}\, \mathbf{q},
\end{equation*}
recovering the definitions in Equation~\ref{eq:HG-def} and the closed form $s^\ast(\mathbf{q}) = \alpha(\mathbf{q})  / \beta(\mathbf{q}) $ in Equation~\ref{eq:closed-form}.
Similarly, Equation~\ref{eq:channel_error} can be expressed in terms of $\alpha/\beta$
\begin{equation*}
E(s \mid \mathbf{q}) = \|\mathbf{X}\mathbf{w}\|^{2} - 2s\,\langle\mathbf{X}\mathbf{w},\tilde{\mathbf{X}}\, \mathbf{q}\rangle +s^{2}\|\,\tilde{\mathbf{X}}\, \mathbf{q}\|^{2} = \mathrm{const} - 2s\,\alpha(\mathbf{q}) + s^{2}\,\beta(\mathbf{q}).
\end{equation*}

\end{remark}
\begin{lemma}[Decision regions of round-to-nearest]
\label{lemma:decision_regions}
Let $\G=\{g_1<g_2<\cdots<g_L\}\subset\mathbb{R}$ be a finite grid
such that $0\in\G$ and define
\begin{equation*}
m_\ell:=\frac{g_\ell+g_{\ell+1}}{2},
\qquad \ell=1,\ldots,L-1.
\end{equation*}
Then for a value $y\notin \{m_\ell\}$, $g_\ell$ is the unique nearest grid value whenever
\begin{equation*}
y\in(m_{\ell-1},m_\ell),
\end{equation*}
with $\ell\in\{1,\ldots,L\}$, where $m_0=-\infty$ and $m_L=+\infty$. At a transition point
$y=m_\ell$, the two nearest grid values are $g_\ell$ and
$g_{\ell+1}$, and the tie-breaking rule of
Remark~\ref{remark:tie-breaking} selects
\begin{equation*}
\round{m_\ell}_{\G}
=
\begin{cases}
g_\ell, & m_\ell>0,\\
g_{\ell+1}, & m_\ell<0.
\end{cases}
\end{equation*}
\end{lemma}
\begin{proof}
Fix $\ell\in\{2,\ldots,L-1\}$. We first compare the distances from
$y$ to the neighboring grid points. Since $g_{\ell-1}<g_\ell$, we have
\begin{equation*}
\begin{aligned}
|y-g_\ell|<|y-g_{\ell-1}|
&\iff
y-g_\ell>-(y-g_{\ell-1})
\qquad
(\text{since }g_{\ell-1}<g_\ell)\\
&\iff
2y>g_\ell+g_{\ell-1}\\
&\iff
y>\frac{g_{\ell-1}+g_\ell}{2}
=m_{\ell-1}.
\end{aligned}
\end{equation*}
Similarly,
\begin{equation*}
\begin{aligned}
|y-g_\ell|<|y-g_{\ell+1}|
&\iff
y-g_\ell<-(y-g_{\ell+1})
\qquad
(\text{since }g_\ell<g_{\ell+1})\\
&\iff
2y<g_\ell+g_{\ell+1}\\
&\iff
y<\frac{g_\ell+g_{\ell+1}}{2}
=m_\ell.
\end{aligned}
\end{equation*}
Therefore,
\begin{equation*}
y\in(m_{\ell-1},m_\ell)
\iff
\bigl(|y-g_\ell|<|y-g_{\ell-1}|\bigr)
\ \wedge\
\bigl(|y-g_\ell|<|y-g_{\ell+1}|\bigr).
\tag{$\star$}
\end{equation*}

Now let $y\in(m_{\ell-1},m_\ell)$. For any $j\leq\ell-1$,
since $g_j\leq g_{\ell-1}<y$, we have
\begin{equation*}
|y-g_j|
=
y-g_j
\geq
y-g_{\ell-1}
>
|y-g_\ell|.
\end{equation*}
For any $j\geq\ell+1$, since $y<g_{\ell+1}\leq g_j$, we have
\begin{equation*}
|y-g_j|
=
g_j-y
\geq
g_{\ell+1}-y
>
|y-g_\ell|.
\end{equation*}
Hence, $g_\ell$ uniquely minimizes $|y-g|$ over $g\in\G$, and thus
\begin{equation*}
\round{y}_{\G}=g_\ell.
\end{equation*}
The endpoint cases $\ell=1$ and $\ell=L$ follow analogously, using
only the existing neighboring grid point.

It remains to consider the transition points. Let $y=m_\ell$. By
the definition of the midpoint,
\begin{equation*}
|m_\ell-g_\ell|
=
|m_\ell-g_{\ell+1}|.
\end{equation*}
Moreover, for every $j\leq\ell-1$,
\begin{equation*}
|m_\ell-g_j|
>
|m_\ell-g_\ell|,
\end{equation*}
and, for every $j\geq\ell+2$,
\begin{equation*}
|m_\ell-g_j|
>
|m_\ell-g_{\ell+1}|.
\end{equation*}
Therefore, the two nearest grid values to $m_\ell$ are exactly
$g_\ell$ and $g_{\ell+1}$.

Since $0\in\G$ and $g_\ell$ and $g_{\ell+1}$ are adjacent grid
values, $m_\ell\neq 0$. If $m_\ell>0$, then
$g_\ell+g_{\ell+1}>0$, and
\begin{equation*}
|g_{\ell+1}|^2-|g_\ell|^2
=
(g_{\ell+1}-g_\ell)(g_{\ell+1}+g_\ell)
>
0.
\end{equation*}
Thus, $|g_\ell|<|g_{\ell+1}|$, and the tie-breaking rule of
Remark~\ref{remark:tie-breaking} selects $g_\ell$. Similarly, if
$m_\ell<0$, then
\begin{equation*}
|g_{\ell+1}|^2-|g_\ell|^2
=
(g_{\ell+1}-g_\ell)(g_{\ell+1}+g_\ell)
<
0.
\end{equation*}
Thus, $|g_{\ell+1}|<|g_\ell|$, and the tie-breaking rule selects
$g_{\ell+1}$. Consequently,
\begin{equation*}
\round{m_\ell}_{\G}
=
\begin{cases}
g_\ell, & \text{if }m_\ell>0,\\
g_{\ell+1}, & \text{if }m_\ell<0.
\end{cases}
\end{equation*}
Thus, the rounding assignment is constant between consecutive
midpoints, while its value at each midpoint is determined by the
tie-breaking rule of Remark~\ref{remark:tie-breaking}.
\end{proof}

\begin{corollary}[Step-wise structure]
\label{cor:step_wise}
Let $\G=\{g_1<\cdots<g_L\}$ be a finite grid, such that $0\in\G$, with midpoints
$m_\ell=(g_\ell+g_{\ell+1})/2$, with  $\ell\in \{1,\ldots,L-1\}$. Set $m_0=-\infty$ and $m_L=+\infty$.
For any $w\in\mathbb{R}$ and $s\in\mathbb{R^+}$, define $q(s):=\round{w/s}_\G$. Whenever $w/s\notin\{m_1,\ldots,m_{L-1}\}$, for every
$\ell\in\{1,\ldots,L\}$,
\begin{equation*}
q(s)=g_\ell
\quad\Longleftrightarrow\quad
\frac{w}{s}\in(m_{\ell-1},m_\ell).
\end{equation*}
At a midpoint, the assignment is determined by
Remark~\ref{remark:tie-breaking}.
Consequently, $q(s)$ is piecewise constant on $(0,\infty)$,
with jumps occurring only at the positive transition scales of the form
\begin{equation*}
t_\ell:=\frac{w}{m_\ell}.
\end{equation*}
Moreover, $q$ is right-continuous on $(0,\infty)$: at each such transition,
\begin{equation*}
q(t_\ell)
=
\lim_{\varepsilon\to0^+}q(t_\ell+\varepsilon).
\end{equation*}
\end{corollary}

\begin{proof}
If $w=0$, then $q(s)=0$ for every $s>0$, so the result
holds trivially. For $w\neq 0$, the map $s\mapsto w/s$ is continuous and strictly monotone on
$(0,\infty)$. By
Lemma~\ref{lemma:decision_regions}, away from the midpoints,
$\round{y}_\G=g_\ell$ if and only if
$y\in(m_{\ell-1},m_\ell)$. Substituting $y=w/s$ yields the
stated equivalence.

The assignment can change only when $w/s$ crosses a midpoint,
that is, at a positive solution of $w/s=m_\ell$, given by
$t_\ell=w/m_\ell$. Thus, $q(s)$ is constant between consecutive
positive transition scales. Fix a transition $t_\ell > 0$, so that $w/t_\ell=m_\ell$.
For $s>t_\ell$, the ratio $w/s$ has the same sign as $m_\ell$ (since $s>0$ and $t_\ell=w/m_\ell>0$) and satisfies
$|w/s|=|w|/s<|w|/t_\ell=|m_\ell|$, hence $w/s$ lies strictly between $0$ and $m_\ell$. By continuity of $s \mapsto w/s$, there is $\delta>0$ such that, for all $s\in(t_\ell,t_\ell+\delta)$,
$w/s\in(m_{\ell-1},m_\ell)$ if $m_\ell>0$, and $w/s\in(m_\ell,m_{\ell+1})$ if
$m_\ell<0$. Therefore, $q(s)=g_\ell$ if $m_\ell>0$, and $q(s)=g_{\ell+1}$ if $m_\ell<0$, which by
Lemma~\ref{lemma:decision_regions} equals $q(t_\ell)=\round{m_\ell}_\G$. Thus
$q$ is constant on $[t_\ell,t_\ell+\delta)$, proving the stated limit.
\end{proof}

Any alternative deterministic tie-breaking rule would preserve the
piecewise-constant structure of $q(\cdot)$, but would change which of the two
adjacent intervals each transition scale $t_\ell$ belongs to, and hence the
endpoint conventions used when sweeping them.

\begin{corollary}[Piecewise-constant structure for the full weight vector]
\label{cor:full_vector}
Let $\G=\{g_1<\cdots<g_L\}$ be a finite grid such that $0\in\G$ with midpoints $m_\ell=(g_\ell+g_{\ell+1})/2$, and let $\mathbf{w}\in\R^D$.  For each non-zero weight $w_i$, Corollary~\ref{cor:step_wise} provides at most
$L-1$ transition scales $t_{i,\ell}=w_i/m_\ell$,
$\ell=1,\dots,L-1$.  Collect all $T$ positive transition scales ($T\leq D(L-1)$) and,
assuming they are distinct (the degenerate case is treated in
Appendix~\ref{appendix:degenerate_intervals}), sort them in decreasing order
\begin{equation*}
\tau_1>\tau_2>\cdots>\tau_T.
\end{equation*}
If $T=0$ (for example, if all the weights are 0), there are no positive transition scales and the assignment is constant on $(0,\infty)$. Otherwise, define the half-open intervals
\begin{equation*}
J_0:=[\tau_1,+\infty),\qquad
J_j:=[\tau_{j+1},\tau_j)\;\;(j=1,\dots,T-1),\qquad
J_{T}:=(0,\tau_T).
\end{equation*}
Then the vector-valued assignment
$\mathbf{q}(s):=\lfloor\mathbf{w}/s\rceil_{\G}$ is constant on each $J_j$,
$j=0,\dots,T$.  Moreover, at each transition $\tau_j$ exactly one
component of $\mathbf{q}$ changes: if $\tau_j=t_{i,\ell}$, then only $q_i$
jumps (from $g_\ell$ to $g_{\ell+1}$ or vice versa) while all other entries
remain fixed.
\end{corollary}

\begin{proof}
By Corollary~\ref{cor:step_wise}, each component $q_i(s)=\lfloor w_i/s\rceil_{\G}$
is piecewise constant on $(0,\infty)$ with changes happening only at the positive transition scales
$\{t_{i,\ell}\}_{\ell=1}^{L-1}$.  If $w_i=0$, its assignment is constantly
zero. The full vector $\mathbf{q}(s)$ can therefore
change only when \emph{some} component $q_i$ changes, which occurs only at one of the identifies
transition scales.  Between
consecutive transition scales, no
component crosses a midpoint, so every $q_i(s)$ is constant, and hence
$\mathbf{q}(s)$ is constant on $J_j$. At each positive
transition, Corollary~\ref{cor:step_wise} gives
\begin{equation*}
\mathbf{q}(\tau_j)
=
\lim_{\varepsilon\to0^+}
\mathbf{q}(\tau_j+\varepsilon).
\end{equation*}
Thus, the transition belongs to the interval on its right,
establishing constancy on the stated intervals, including
their closed endpoints.

Since all transition scales are assumed distinct, each
$\tau_j$ corresponds to a unique pair $(i,\ell)$.
Only $q_i$ crosses a midpoint there, changing between
$g_\ell$ and $g_{\ell+1}$; all other components remain
unchanged. If there are no positive transitions, every
component is constant throughout $(0,\infty)$.
\end{proof}

\begin{proposition}[Global infimum of the piecewise-quadratic sweep]
\label{prop:sweep_optimality}
Under the setting and notation of Corollary~\ref{cor:full_vector}, assume $T\geq1$ and let
$\mathbf{H}=\tilde{\mathbf{X}}^\top\tilde{\mathbf{X}}\succeq 0$ and
$\mathbf{G}=\tilde{\mathbf{X}}^\top\mathbf{X}$, and consider the channel-wise
error
\begin{equation*}
E(s):=\|\mathbf{X}\mathbf{w}-s\,\tilde{\mathbf{X}}\,q(\mathbf{w};\,s)\|^2,
\qquad s>0.
\end{equation*}
Write $\mathbf{q}_j$ for the constant value of $\mathbf{q}(s)$ on $J_j$, and
define
\begin{equation*}
\alpha_j:=\mathbf{q}_j^\top\mathbf{G}\,\mathbf{w},\qquad
\beta_j:=\mathbf{q}_j^\top\mathbf{H}\,\mathbf{q}_j,\qquad
\varphi_j(s):=s^2\beta_j-2s\,\alpha_j+c,
\end{equation*}
where $c:=\|\mathbf{X}\mathbf{w}\|^2$ is a constant independent of $s$.  Then:
\begin{enumerate}
\item $\varphi_j$ is a convex quadratic that coincides with $E$ on $J_j$, i.e.,\
      $E(s)=\varphi_j(s)$ for all $s\in J_j$;
\item For each interval $J_j$,
      \begin{equation*}
      \inf_{s\in J_j}E(s)
      =\min_{s\in\overline{J_j}}\varphi_j(s)
      =\varphi_j(s_j^\ast).
      \end{equation*}
      If $\beta_j>0$, define
      \begin{equation*}
      s_j^\ast
      :=
      \mathrm{clamp}\!\left(
      \frac{\alpha_j}{\beta_j},
      \inf J_j,\sup J_j
      \right).
      \end{equation*}
      If $\beta_j=0$, then $\alpha_j=0$ and
      $\varphi_j(s)=c$ for every $s$; in this case,
      choose any $s_j^\ast\in J_j$.

\item $\displaystyle\inf_{s>0}E(s)
      =\min_{j\in\{0,\dots,T\}}\varphi_j(s_j^\ast)$.
      If there exists an index $j^\ast$ attaining the outer minimum such that
      $s_{j^\ast}^\ast\in J_{j^\ast}$, then the infimum is attained exactly.
      Otherwise, for every $\eta>0$, there exists an admissible scale
      $\hat{s}>0$ such that
      \begin{equation*}
      E(\hat{s})-\inf_{s>0}E(s)<\eta.
      \end{equation*}
      In the data-free setting $\mathbf{X}=\tilde{\mathbf{X}}=\mathbf{I}$,
      the infimum is always attained.
\end{enumerate}
\end{proposition}

\begin{proof}
\emph{Part~(i).}
On $J_j$ the assignment $\mathbf{q}(s)\equiv\mathbf{q}_j$ is constant by
Corollary~\ref{cor:full_vector}.  Expanding the squared norm gives
\begin{equation*}
E(s)
=\|\mathbf{X}\mathbf{w}\|^2
  -2s\,\mathbf{q}_j^\top\mathbf{G}\,\mathbf{w}
  +s^2\,\mathbf{q}_j^\top\mathbf{H}\,\mathbf{q}_j
=\varphi_j(s).
\end{equation*}
Since $\mathbf{H}\succeq 0$, we have
$\beta_j=\mathbf{q}_j^\top\mathbf{H}\,\mathbf{q}_j
        =\|\tilde{\mathbf{X}}\,\mathbf{q}_j\|^2\ge 0$,
so $\varphi_j$ is convex.

\emph{Part~(ii).}
If $\beta_j=0$, then
$\|\tilde{\mathbf{X}}\mathbf{q}_j\|^2=0$, and hence
\begin{equation*}
\alpha_j
=
\langle\tilde{\mathbf{X}}\mathbf{q}_j,\mathbf{X}\mathbf{w}\rangle
=
0.
\end{equation*}
Therefore, $\varphi_j(s)=c$ for every $s$, and any
$s_j^\ast\in J_j$ attains the interval minimum.
In particular, this covers the upper interval
$J_0=[\tau_1,+\infty)$, where $\mathbf{q}_0=\mathbf{0}$.

Now suppose $\beta_j>0$. The quadratic $\varphi_j$ has
unconstrained minimizer $\alpha_j/\beta_j$.
If we temporarily include both endpoints of $J_j$,
its minimum is attained by clamping this value:
\begin{equation*}
s_j^\ast
=
\mathrm{clamp}\!\left(
\frac{\alpha_j}{\beta_j},\inf J_j,\sup J_j
\right).
\end{equation*}
Thus, for every $s\in J_j$,
\begin{equation*}
E(s)=\varphi_j(s)\geq\varphi_j(s_j^\ast).
\end{equation*}

If $s_j^\ast\in J_j$, equality is attained.
Otherwise, $s_j^\ast$ is an excluded endpoint.
We can approach that endpoint using scales inside $J_j$.
Since $\varphi_j$ is a polynomial, their errors satisfy
\begin{equation*}
\lim_{\substack{s\to s_j^\ast\\s\in J_j}}E(s)
=
\lim_{\substack{s\to s_j^\ast\\s\in J_j}}\varphi_j(s)
=
\varphi_j(s_j^\ast).
\end{equation*}
Consequently,
\begin{equation*}
\inf_{s\in J_j}E(s)
=
\min_{s\in\overline{J_j}}\varphi_j(s)
=
\varphi_j(s_j^\ast).
\end{equation*}

For the interval $J_T=(0,\tau_T)$, a candidate scale at zero
is interpreted through the limit $s\to0^+$, as the quantizer is not defined for a scaling factor of $0$.
Similarly, a candidate scale that lies at an excluded boundary is
evaluated through the limit from inside its interval which
can be different from the actual objective value at
that transition point.

\emph{Part~(iii).}
The half-open intervals $\{J_j\}_{j=0}^{T}$ are pairwise disjoint and their
union is $(0,+\infty)$.  Therefore
\begin{equation*}
\inf_{s>0}E(s)
=\min_{j\in\{0,\dots,T\}}\inf_{s\in J_j}E(s)
=\min_j\,\varphi_j(s_j^\ast),
\end{equation*}
where the second equality uses Part~(ii).

Let $j^\ast$ be an index attaining the outer minimum. If
$s_{j^\ast}^\ast\in J_{j^\ast}$, then
\begin{equation*}
E(s_{j^\ast}^\ast)
=\varphi_{j^\ast}(s_{j^\ast}^\ast)
=\inf_{s>0}E(s),
\end{equation*}
and the infimum is attained exactly.

In the degenerate case where $s_{j^\ast}^\ast=0$, the global infimum equals
$\varphi_{j^\ast}(0)=c$. This value is also attained
at every $s\geq\tau_1$, where all weights round to zero.
Thus, a positive scale attains the global infimum
in this case as well. It remains to consider the case in which the minimizing candidate is the
excluded right endpoint of its interval,
\begin{equation*}
s_{j^\ast}^\ast=\tau_{j^\ast}\notin J_{j^\ast}.
\end{equation*}
For any
\begin{equation*}
0<\varepsilon<
\min\{\tau_{j^\ast}-\inf J_{j^\ast},\,1\},
\end{equation*}
define
\begin{equation*}
s_\varepsilon:=\tau_{j^\ast}-\varepsilon.\end{equation*}
Then $s_\varepsilon\in J_{j^\ast}$ and therefore
$E(s_\varepsilon)=\varphi_{j^\ast}(s_\varepsilon)$. Since
\begin{equation*}
\inf_{s>0}E(s)
=\varphi_{j^\ast}(\tau_{j^\ast}),
\end{equation*}
the sub-optimality gap is
\begin{equation*}
\begin{aligned}
E(s_\varepsilon)-\inf_{s>0}E(s)
&=
\varphi_{j^\ast}(\tau_{j^\ast}-\varepsilon)
-\varphi_{j^\ast}(\tau_{j^\ast})\\
&=
2\varepsilon
\bigl(\alpha_{j^\ast}
-\beta_{j^\ast}\tau_{j^\ast}\bigr)
+\beta_{j^\ast}\varepsilon^2.
\end{aligned}
\end{equation*}
Consequently,
\begin{equation*}
E(s_\varepsilon)-\inf_{s>0}E(s)
=\mathcal{O}(\varepsilon),
\end{equation*}
and the gap converges to zero as $\varepsilon\to0^+$. Hence, for every
$\eta>0$, one can choose $\varepsilon$ sufficiently small so that
\begin{equation*}
E(s_\varepsilon)-\inf_{s>0}E(s)<\eta.
\end{equation*}

Finally, consider the data-free setting
$\mathbf{X}=\tilde{\mathbf{X}}=\mathbf{I}$. At a transition scale
$\tau=t_{i,\ell}$, we have
\begin{equation*}
\frac{w_i}{\tau}
=m_\ell
=\frac{g_\ell+g_{\ell+1}}{2}.
\end{equation*}
It follows that
\begin{equation*}
(w_i-\tau g_\ell)^2
=(w_i-\tau g_{\ell+1})^2.
\end{equation*}
All other grid assignments remain unchanged at this transition. Therefore,
the data-free error
\begin{equation*}
E(s)=\|\mathbf{w}-s\,q(\mathbf{w};s)\|^2
\end{equation*}
is continuous at every transition scale. In particular, when an interval
minimum lies at an otherwise excluded transition point,
\begin{equation*}
E(\tau_{j^\ast})
=\varphi_{j^\ast}(\tau_{j^\ast}).
\end{equation*}
Thus, in the data-free setting, the global infimum is attained exactly.
\end{proof}

Proposition~\ref{prop:sweep_optimality} guarantees that the global infimum is identified by enumerating all $T+1 \le D(|\G|{-}1)+1$ intervals, but a naive implementation that recomputes $\alpha$, $\beta$ from scratch at each transition would cost $\mathcal{O}(D^2)$ per interval, yielding $\mathcal{O}(D^3|\G|)$ per channel. The following result shows that, because exactly one component of $\mathbf{q}$ changes at each transition (Corollary~\ref{cor:full_vector}), the quantities $\alpha$, $\beta$, and the auxiliary vector $\mathbf{h} = \mathbf{H}\mathbf{q}$ can be maintained incrementally: $\alpha$ and $\beta$ update in $\mathcal{O}(1)$, while $\mathbf{h}$ updates in $\mathcal{O}(D)$, reducing the channel-wise cost of the sweep to $\mathcal{O}(D^2|\G|)$ with a dense Hessian and to $\mathcal{O}(D|\G|)$ when $\mathbf{H}$ is diagonal.

\begin{restatable}[Correctness of incremental updates]{proposition}{prop:incremental_updates}
\label{prop:incremental_updates}
Define
\begin{align*}
    \alpha &:=\alpha(\mathbf{q}) = \mathbf{q}^\top \mathbf{G}\,\mathbf{w}, \\
    \mathbf{h} &:= \mathbf{h}(\mathbf{q}) = \mathbf{H}\,\mathbf{q}, \\
    \beta  &:=\beta(\mathbf{q}) = \mathbf{q}^\top \mathbf{h}(\mathbf{q}),
\end{align*}
where $\mathbf{q} \in \G^D$ is the current grid assignment.  Suppose entry~$i$ of $\mathbf{q}$ changes from $q_i^{\mathrm{old}}$ to $q_i^{\mathrm{new}} = q_i^{\mathrm{old}} + \delta$,
while all other entries remain unchanged. Writing $\mathbf{q}' = \mathbf{q} + \delta\,\mathbf{e}_i$ for the updated assignment (where $\mathbf{e}_i$ is the $i$-th standard basis vector), the updated quantities
\begin{align*}
    \alpha' &= \alpha + \delta \cdot (\mathbf{G}\mathbf{w})_i, \\
    \mathbf{h}' &= \mathbf{h} + \delta \cdot \mathbf{H}_{:,i} \\
    \beta'  &= \beta + 2\delta \cdot h_i + \delta^2 \cdot H_{ii},
\end{align*}
satisfy $\alpha' = \alpha(\mathbf{q}')$, $\beta' =\beta(\mathbf{q}')$, and $\mathbf{h}' = \mathbf{h}(\mathbf{q}')$.
\end{restatable}

\begin{proof}
Since $\mathbf{q}' = \mathbf{q} + \delta\,\mathbf{e}_i$, we have:

\emph{Update for~$\alpha$:}
$\alpha(\mathbf{q'}) = (\mathbf{q}')^\top \mathbf{G}\,\mathbf{w} = (\mathbf{q} + \delta\,\mathbf{e}_i)^\top \mathbf{G}\,\mathbf{w} = \mathbf{q}^\top \mathbf{G}\,\mathbf{w} + \delta\,\mathbf{e}_i^\top \mathbf{G}\,\mathbf{w} = \alpha + \delta \cdot (\mathbf{G}\mathbf{w})_i = \alpha'$.

\emph{Update for~$\mathbf{h}$:}
$\mathbf{h}' = \mathbf{H}\,\mathbf{q}' = \mathbf{H}(\mathbf{q} + \delta\,\mathbf{e}_i) = \mathbf{H}\mathbf{q} + \delta\,\mathbf{H}\mathbf{e}_i = \mathbf{h} + \delta \cdot \mathbf{H}_{:,i}$.

\emph{Update for~$\beta$:}
\begin{align*}
\beta' &= (\mathbf{q}')^\top \mathbf{H}\,\mathbf{q}' = (\mathbf{q} + \delta\,\mathbf{e}_i)^\top \mathbf{H}\,(\mathbf{q} + \delta\,\mathbf{e}_i) \\
       &= \mathbf{q}^\top\mathbf{H}\mathbf{q} + 2\delta\,\mathbf{e}_i^\top\mathbf{H}\mathbf{q} + \delta^2\,\mathbf{e}_i^\top\mathbf{H}\mathbf{e}_i \\
       &= \beta + 2\delta\,(\mathbf{H}\mathbf{q})_i + \delta^2\,H_{ii} = \beta + 2\delta \cdot h_i + \delta^2 \cdot H_{ii}.
\end{align*}
\end{proof}

The remaining results concern the group-wise adaptations (Section~\ref{sec:group-wise}). Proposition~\ref{prop:block_diag_gap} bounds the error introduced by the block-diagonal (independent groups) objective, while Proposition~\ref{prop:sequential_equivalence} shows that, under the self-activation assumption ($\mathbf{X}=\tilde{\mathbf{X}}$), the sequential group-wise objective coincides with the layer-wise reconstruction objective.

\begin{restatable}[Block-diagonal approximation error]{proposition}{prop}
\label{prop:block_diag_gap}
Let $\mathbf{s}=(s_1,\ldots,s_K)$ and define the stacked vectors
\begin{equation*}
\mathbf{w}
:= \left[\mathbf{w}_{\{1\}}, \ldots, \mathbf{w}_{\{K\}}\right]^\top
\qquad
\mathbf{q}(\mathbf{s}):=q(\mathbf{w};\mathbf{s})
= \left[\mathbf{q}_{\{1\}}, \ldots, \mathbf{q}_{\{K\}}\right]^\top
\qquad
\mathbf{q}_{\{k\}}:=q(\mathbf{w}_{\{k\}};s_k)\;,
\end{equation*}
and the block-wise matrices $\mathbf{G}_{ij}:=(\tilde{\mathbf{X}}_{\{i\}})^\top\mathbf{X}_{\{j\}}$ and
$\mathbf{H}_{ij}:=(\tilde{\mathbf{X}}_{\{i\}})^\top\tilde{\mathbf{X}}_{\{j\}}$.
Let $\mathbf{S}(\mathbf{s}) := \mathrm{blkdiag}(s_1\mathbf{I},\ldots,s_K\mathbf{I})$ (with $\mathbf{I}$ matching the group dimension),
and define the block matrices
\begin{equation*}
\mathbf{G}:=\big[\mathbf{G}_{ij}\big]_{i,j=1}^K,\qquad
\mathbf{H}:=\big[\mathbf{H}_{ij}\big]_{i,j=1}^K,
\end{equation*}
together with their block-diagonal parts
\begin{equation*}
\mathbf{G}_{\mathrm{ind}}:=\mathrm{blkdiag}(\mathbf{G}_{11},\ldots,\mathbf{G}_{KK}),
\qquad
\mathbf{H}_{\mathrm{ind}}:=\mathrm{blkdiag}(\mathbf{H}_{11},\ldots,\mathbf{H}_{KK}),
\end{equation*}
and off-diagonal parts $\mathbf{G}_{\mathrm{off}}:=\mathbf{G}-\mathbf{G}_{\mathrm{ind}}$,
$\mathbf{H}_{\mathrm{off}}:=\mathbf{H}-\mathbf{H}_{\mathrm{ind}}$.
For this proposition, we use $E_{\mathrm{ind}}$ to denote
the objective in Equation~\ref{eq:independent-group}
shifted by the scale-independent constant
\begin{equation*}
\|\mathbf{X}\mathbf{w}\|^2
-
\sum_{k=1}^{K}
\|\mathbf{X}_{\{k\}}\mathbf{w}_{\{k\}}\|^2.
\end{equation*}
This leaves its minimizers unchanged and gives
$E$ and $E_{\mathrm{ind}}$ the same constant term
in their quadratic expansions.
With these definitions, the objective
\begin{equation*}
E(\mathbf{s}) = \|\mathbf{X}\mathbf{w} - \tilde{\mathbf{X}}\mathcal{Q}(\mathbf{w};\mathbf{s})\|^{2}
\end{equation*}
can be expressed as
\begin{equation*}
E(\mathbf{s}) = \mathrm{const.} - 2\,\mathbf{w}^\top\mathbf{G}^\top\mathbf{S}(\mathbf{s})\mathbf{q}(\mathbf{s}) + \big(\mathbf{S}(\mathbf{s})\mathbf{q}(\mathbf{s})\big)^\top\mathbf{H}\big(\mathbf{S}(\mathbf{s})\mathbf{q}(\mathbf{s})\big).
\end{equation*}
and its block-diagonal approximation
\begin{equation*}
E_{\mathrm{ind}}(\mathbf{s}) = \mathrm{const.} - 2\,\mathbf{w}^\top\mathbf{G}_{\mathrm{ind}}^\top\mathbf{S}(\mathbf{s})\mathbf{q}(\mathbf{s}) + \big(\mathbf{S}(\mathbf{s})\mathbf{q}(\mathbf{s})\big)^\top\mathbf{H}_{\mathrm{ind}}\big(\mathbf{S}(\mathbf{s})\mathbf{q}(\mathbf{s})\big).
\end{equation*}

Then, defining
\begin{equation*}
\Delta(\mathbf{s}) := 2\,\|\mathbf{w}\|_2\,\|\mathbf{G}_{\mathrm{off}}\|_2\,\|\mathbf{S}(\mathbf{s})\mathbf{q}(\mathbf{s})\|_2
\;+\;
\|\mathbf{H}_{\mathrm{off}}\|_2\,\|\mathbf{S}(\mathbf{s})\mathbf{q}(\mathbf{s})\|_2^2,
\end{equation*}
Assuming both minima are attained, let $\mathbf{s}_{\mathrm{ind}}^\ast \in\argmin_{\mathbf{s}}E_{\mathrm{ind}}(\mathbf{s})$ and $\mathbf{s}^\ast\in\argmin_{\mathbf{s}}E(\mathbf{s})$. Then,
\begin{equation*}
|E(\mathbf{s}_{\mathrm{ind}}^\ast) - E(\mathbf{s}^\ast)| \le \Delta(\mathbf{s}_{\mathrm{ind}}^\ast) + \Delta(\mathbf{s}^\ast).
\end{equation*}

\end{restatable}

\begin{proof}
Write $\mathbf{z}(\mathbf{s}):=\mathbf{S}(\mathbf{s})\mathbf{q}(\mathbf{s})$.
Then
\begin{equation*}
E(\mathbf{s})-E_{\mathrm{ind}}(\mathbf{s}) = -2\,\mathbf{w}^\top\mathbf{G}_{\mathrm{off}}^\top\mathbf{z}(\mathbf{s}) + \mathbf{z}(\mathbf{s})^\top\mathbf{H}_{\mathrm{off}}\mathbf{z}(\mathbf{s}).
\end{equation*}
Taking absolute values and applying Cauchy--Schwarz and submultiplicativity
of the operator norm gives
\begin{equation}\label{eq:group_optim_gap}
\big|E(\mathbf{s})-E_{\mathrm{ind}}(\mathbf{s})\big|
\le
2\,\|\mathbf{w}\|_2\,\|\mathbf{G}_{\mathrm{off}}\|_2\,\|\mathbf{z}(\mathbf{s})\|_2
+
\|\mathbf{H}_{\mathrm{off}}\|_2\,\|\mathbf{z}(\mathbf{s})\|_2^2 =: \Delta(\mathbf{s}).
\end{equation}
Therefore, the following chain of inequalities is verified:
\begin{align*}
E(\mathbf{s}^\ast_{\mathrm{ind}})
&\le E_{\mathrm{ind}}(\mathbf{s}^\ast_{\mathrm{ind}})
+ \Delta(\mathbf{s}^\ast_{\mathrm{ind}})&& \text{from Equation~\ref{eq:group_optim_gap}}\\
&\le E_{\mathrm{ind}}(\mathbf{s}^\ast) + \Delta(\mathbf{s}^\ast_{\mathrm{ind}}) && \text{optimality of $\mathbf{s}^\ast_{\mathrm{ind}}$ for $E_{\mathrm{ind}}$}\\
&\le E(\mathbf{s}^\ast) + \Delta(\mathbf{s}^\ast) + \Delta(\mathbf{s}_{\mathrm{ind}}^\ast) && \text{from Equation~\ref{eq:group_optim_gap}.}
\end{align*}
Therefore
\begin{equation*}
|E(\mathbf{s^\ast_{\mathrm{ind}}}) - E(\mathbf{s^\ast})| \leq \Delta(\mathbf{s^\ast_{\mathrm{ind}}})+\Delta(\mathbf{s^\ast}),
\end{equation*}
concluding the proof.
\end{proof}

\begin{proposition}[Equivalence of the sequential objective]
\label{prop:sequential_equivalence}
Let $\mathbf{w} = [\mathbf{w}_{\{1\}}, \ldots, \mathbf{w}_{\{K\}}]^\top$ be partitioned into
$K$ groups. Define the partially quantized vector
\begin{equation*}
\tilde{\mathbf{w}}(s_{k}) := [\tilde{\mathbf{w}}_{\{1\}}, \ldots,
\tilde{\mathbf{w}}_{\{k-1\}}, s_{k}\, q(\mathbf{w}_{\{k\}}, s_{k}), \mathbf{w}_{\{k+1\}}, \ldots, \mathbf{w}_{\{K\}}]^\top,
\end{equation*}
where $\tilde{\mathbf{w}}_{\{i\}} = s_i^\ast q(\mathbf{w}_{\{i\}}, s_i^\ast)$, $i \in [k-1]$, with the scales $s_1^\ast, \ldots, s_{k-1}^\ast$ fixed for the first $k-1$ groups. Consider the objectives
\begin{align}
\overline{E}_{\mathrm{seq}}(s_k)
&:= \bigl\| \mathbf{X}\,\mathbf{w}
     - \tilde{\mathbf{X}}\,\tilde{\mathbf{w}}(s_{k}) \bigr\|^2,
\label{eq:mix} \\
E_{\mathrm{seq}}(s_k)
&:= \bigl\| \mathbf{X}_{\{\leq k\}}\,\mathbf{w}_{\{\leq k\}}
     - \tilde{\mathbf{X}}_{\{<k\}}\,\tilde{\mathbf{w}}_{\{<k\}}
     - s_k\,\tilde{\mathbf{X}}_{\{k\}}\,q(\mathbf{w}_{\{k\}},s_k) \bigr\|^2.
\label{eq:seq}
\end{align}
In the self-activation case ($\mathbf{X} = \tilde{\mathbf{X}}$), the two objectives coincide for every admissible $s_k$. Consequently, they have the same infimum and, whenever it is attained, the same set of minimizers: $\argmin_{s_k} \overline{E}_{\mathrm{seq}}(s_k) = \argmin_{s_k} E_{\mathrm{seq}}(s_k)$.
\end{proposition}
\begin{proof}
Write $\mathbf{q}_{\{k\}} := q(\mathbf{w}_{\{k\}}, s_k)$ for brevity.
Decomposing column-wise:
\begin{equation*}
\mathbf{X}\,\mathbf{w}
= \mathbf{X}_{\{\leq k\}}\,\mathbf{w}_{\{\leq k\}}
  + \mathbf{X}_{\{>k\}}\,\mathbf{w}_{\{>k\}},
\end{equation*}
\begin{equation*}
\tilde{\mathbf{X}}\,\tilde{\mathbf{w}}(s_{k})
= \tilde{\mathbf{X}}_{\{<k\}}\,\tilde{\mathbf{w}}_{\{<k\}}
  + s_k\,\tilde{\mathbf{X}}_{\{k\}}\,\mathbf{q}_{\{k\}}
  + \tilde{\mathbf{X}}_{\{>k\}}\,\mathbf{w}_{\{>k\}}.
\end{equation*}
Subtracting and defining
\begin{align*}
A(s_k) &:= \mathbf{X}_{\{\leq k\}}\,\mathbf{w}_{\{\leq k\}}
           - \tilde{\mathbf{X}}_{\{<k\}}\,\tilde{\mathbf{w}}_{\{<k\}}
           - s_k\,\tilde{\mathbf{X}}_{\{k\}}\,\mathbf{q}_{\{k\}}, \\
B      &:= \left(\mathbf{X}_{\{>k\}} - \tilde{\mathbf{X}}_{\{>k\}}\right)\,\mathbf{w}_{\{>k\}},
\end{align*}
we obtain
\begin{equation*}
\mathbf{X}\,\mathbf{w} - \tilde{\mathbf{X}}\,\tilde{\mathbf{w}}(s_{k})
= A(s_k) + B.
\end{equation*}
Note that $E_{\mathrm{seq}}(s_k) = \|A(s_k)\|^2$.  Expanding $\overline{E}_{\mathrm{seq}}$:
\begin{equation}\label{eq:mixed_expansion}
\overline{E}_{\mathrm{seq}}(s_k)
= \|A(s_k) + B\|^2
= \underbrace{\|A(s_k)\|^2}_{= E_{\mathrm{seq}}(s_k)}
  + 2\langle A(s_k),\, B \rangle
  + \underbrace{\|B\|^2}_{\text{const.\ in } s_k}.
\end{equation}

Then, if $\mathbf{X} = \tilde{\mathbf{X}}$, then $\mathbf{X}_{\{>k\}} - \tilde{\mathbf{X}}_{\{>k\}} = \mathbf{0}$, so $B = \mathbf{0}$. Hence,
Equation~\ref{eq:mixed_expansion} reduces to
$\overline{E}_{\mathrm{seq}}(s_k) = E_{\mathrm{seq}}(s_k) + 2\langle A(s_k),\, \mathbf{0} \rangle+ \|\mathbf{0}\|^2
 = E_{\mathrm{seq}}(s_k)$, so the two objectives coincide, as claimed.
\end{proof}

\clearpage
\section{Low-Dimensional Example}
\label{app:low-dimensional-example}

We provide a two-dimensional example illustrating the piecewise-quadratic structure exploited by \ScaleAlgorithm{} and why a scale that is optimal for data-free weight reconstruction can be substantially suboptimal for layer-output reconstruction. Consider a single output channel with
\begin{equation*}
\mathbf{w}
=
\begin{bmatrix}
1\\
2
\end{bmatrix},
\qquad
\G=\{-1,0,1\},
\qquad
\mathbf{X}=\tilde{\mathbf{X}}
=
\begin{bmatrix}
3&1\\
0&1
\end{bmatrix}.
\end{equation*}
The corresponding matrices in Equation~\ref{eq:HG-def} are
\begin{equation*}
\mathbf{H}
=
\tilde{\mathbf{X}}^\top\tilde{\mathbf{X}}
=
\mathbf{G}
=
\tilde{\mathbf{X}}^\top\mathbf{X}
=
\begin{bmatrix}
9&3\\
3&2
\end{bmatrix}.
\end{equation*}
The unequal diagonal entries show that errors in the two weights are amplified differently, while the nonzero off-diagonal entries capture their interaction at the layer output. We consider positive scales and RTN quantization,

\paragraph{Step 1: Computing the transition scales.}
Since the unscaled grid $\G$ has three values, we get two midpoints which are
\begin{equation*}
m_1=-\frac{1}{2},
\qquad
m_2=\frac{1}{2}.
\end{equation*}
Since we are focusing on positive scales ($s>0$) we only have to consider the positive midpoint above ($m_2=1/2$). Using Equation~\ref{eq:transition}, we obtain the following transition scales
\begin{equation*}
t_{1,2}
=
\frac{w_1}{m_2}
=
\frac{1}{1/2}
=
2,
\qquad
t_{2,2}
=
\frac{w_2}{m_2}
=
\frac{2}{1/2}
=
4.
\end{equation*}
Sorting the transition scales in decreasing order gives
\begin{equation*}
\tau_1=4,
\qquad
\tau_2=2.
\end{equation*}

No component of \(q(\mathbf{w};s)\) changes between two consecutive transition scales. The transition scales partition the positive scale domain into intervals on
which the grid assignment is constant:
\begin{equation*}
q(\mathbf{w};s)
=
\begin{cases}
(0,0)^\top, & s\in[4,+\infty),\\
(0,1)^\top, & s\in[2,4),\\
(1,1)^\top, & s\in(0,2).
\end{cases}
\end{equation*}
Since the assignment is constant within each interval, it is sufficient to
evaluate \(q(\mathbf{w};s)\) at any one scale in that interval to determine
the assignment for the whole interval. For example, choosing
\(s=8\in(4,+\infty)\) gives
\begin{equation*}
q(\mathbf{w};8)
=
\round{\mathbf{w}/8}_{\G}
=
\round{(1/8,\,1/4)^\top}_{\G}
=
(0,0)^\top.
\end{equation*}
The assignments in the other intervals can be obtained analogously.

Following the quadratic decomposition in
Equation~\ref{eq:closed-form}, the term independent of \(s\) is
\begin{equation*}
c
:=
\|\mathbf{X}\mathbf{w}\|^2
=
\|(5,2)^\top\|^2
=
29.
\end{equation*}

Substituting the three assignments into
\begin{equation*}
E(s)
=
c-2s\alpha+s^2\beta
\end{equation*}
gives the piecewise-quadratic function
\begin{equation*}
E(s)
=
\begin{cases}
29,
&
s\in[4, +\infty),
\\
29-14s+2s^2,
&
s \in [2,4),
\\
29-44s+17s^2,
&
s \in (0,2).
\end{cases}
\end{equation*}
We next recover these pieces by explicitly following the incremental \ScaleAlgorithm{} sweep.

\paragraph{Step 2: Initializing the sweep.}
The sweep proceeds from large to small scales. We first precompute
\begin{equation*}
\mathbf{p}
:=
\mathbf{G}\mathbf{w}
=
\begin{bmatrix}
9&3\\
3&2
\end{bmatrix}
\begin{bmatrix}
1\\
2
\end{bmatrix}
=
\begin{bmatrix}
15\\
7
\end{bmatrix}.
\end{equation*}
For \(s>4\), the assignment is
\begin{equation*}
\mathbf{q}_0
=
\begin{bmatrix}
0\\
0
\end{bmatrix}.
\end{equation*}
Therefore, the running quantities are initialized as
\begin{equation*}
\alpha_0
=
\mathbf{q}_0^\top\mathbf{G}\mathbf{w}
=
0,
\qquad
\beta_0
=
\mathbf{q}_0^\top\mathbf{H}\mathbf{q}_0
=
0,
\qquad
\mathbf{h}_0
=
\mathbf{H}\mathbf{q}_0
=
\begin{bmatrix}
0\\
0
\end{bmatrix}.
\end{equation*}
The objective is consequently constant on this interval:
\begin{equation*}
\varphi_0(s)=c=29.
\end{equation*}

For each subsequent interval, \ScaleAlgorithm{} maintains
\begin{equation*}
\alpha_j
=
\mathbf{q}_j^\top\mathbf{G}\mathbf{w},
\qquad
\beta_j
=
\mathbf{q}_j^\top\mathbf{H}\mathbf{q}_j,
\qquad
\mathbf{h}_j
=
\mathbf{H}\mathbf{q}_j.
\end{equation*}
If one assignment changes by \(\delta\) at coordinate \(i\), these quantities are updated as
\begin{equation*}
\alpha
\leftarrow
\alpha+\delta p_i,
\qquad
\beta
\leftarrow
\beta+2\delta h_i+\delta^2H_{ii},
\qquad
\mathbf{h}
\leftarrow
\mathbf{h}+\delta\mathbf{H}_{:,i}.
\end{equation*}

\paragraph{Step 3: Crossing the transition \(\tau_1=4\).}
At \(\tau_1=4\), the second grid assignment changes from \(0\) to \(1\). Thus,
\begin{equation*}
i=2,
\qquad
\delta=1,
\qquad
\mathbf{q}_1
=
\begin{bmatrix}
0\\
1
\end{bmatrix}.
\end{equation*}
Starting from \(\alpha_0,\beta_0,\mathbf{h}_0\), the incremental updates give
\begin{equation*}
\begin{aligned}
\alpha_1
&=
\alpha_0+\delta p_2
=
0+1\cdot7
=
7,
\\
\beta_1
&=
\beta_0+2\delta(\mathbf{h}_0)_2+\delta^2H_{22}
=
0+0+2
=
2,
\\
\mathbf{h}_1
&=
\mathbf{h}_0+\delta\mathbf{H}_{:,2}
=
\begin{bmatrix}
0\\
0
\end{bmatrix}
+
\begin{bmatrix}
3\\
2
\end{bmatrix}
=
\begin{bmatrix}
3\\
2
\end{bmatrix}.
\end{aligned}
\end{equation*}
These values can also be verified directly:
\begin{equation*}
\alpha_1
=
\begin{bmatrix}
0&1
\end{bmatrix}
\begin{bmatrix}
15\\
7
\end{bmatrix}
=
7,
\qquad
\beta_1
=
\begin{bmatrix}
0&1
\end{bmatrix}
\begin{bmatrix}
9&3\\
3&2
\end{bmatrix}
\begin{bmatrix}
0\\
1
\end{bmatrix}
=
2.
\end{equation*}

The quadratic associated with this interval is
\begin{equation*}
\varphi_1(s)
=
c-2\alpha_1s+\beta_1s^2
=
29-14s+2s^2.
\end{equation*}
Its unconstrained minimizer is
\begin{equation*}
s_1^{\mathrm{unc}}
=
\frac{\alpha_1}{\beta_1}
=
\frac{7}{2}.
\end{equation*}
Since
\begin{equation*}
2<\frac{7}{2}<4,
\end{equation*}
the minimizer lies strictly inside the interval and no clamping is needed:
\begin{equation*}
s_1^\ast
=
\mathrm{clamp}
\left(
\frac{\alpha_1}{\beta_1},\,2,\,4
\right)
=
\frac{7}{2}.
\end{equation*}

Using the reduced score maintained by the algorithm,
\begin{equation*}
f_j
:=
(s_j^\ast)^2\beta_j-2s_j^\ast\alpha_j,
\end{equation*}
we obtain
\begin{equation*}
f_1
=
\left(\frac{7}{2}\right)^2 2
-
2\left(\frac{7}{2}\right)7
=
-\frac{49}{2}.
\end{equation*}
The corresponding full layer-output error is
\begin{equation*}
\varphi_1(s_1^\ast)
=
c+f_1
=
29-\frac{49}{2}
=
\frac{9}{2}.
\end{equation*}

\paragraph{Step 4: Crossing the transition \(\tau_2=2\).}
At \(\tau_2=2\), the first grid assignment changes from \(0\) to \(1\). Thus,
\begin{equation*}
i=1,
\qquad
\delta=1,
\qquad
\mathbf{q}_2
=
\begin{bmatrix}
1\\
1
\end{bmatrix}.
\end{equation*}
Starting from the quantities for the preceding interval, the incremental updates give
\begin{equation*}
\begin{aligned}
\alpha_2
&=
\alpha_1+\delta p_1
=
7+1\cdot15
=
22,
\\
\beta_2
&=
\beta_1+2\delta(\mathbf{h}_1)_1+\delta^2H_{11}
=
2+2\cdot3+9
=
17,
\\
\mathbf{h}_2
&=
\mathbf{h}_1+\delta\mathbf{H}_{:,1}
=
\begin{bmatrix}
3\\
2
\end{bmatrix}
+
\begin{bmatrix}
9\\
3
\end{bmatrix}
=
\begin{bmatrix}
12\\
5
\end{bmatrix}.
\end{aligned}
\end{equation*}
Again, the resulting quantities can be checked directly:
\begin{equation*}
\alpha_2
=
\begin{bmatrix}
1&1
\end{bmatrix}
\begin{bmatrix}
15\\
7
\end{bmatrix}
=
22,
\end{equation*}
and
\begin{equation*}
\beta_2
=
\begin{bmatrix}
1&1
\end{bmatrix}
\begin{bmatrix}
9&3\\
3&2
\end{bmatrix}
\begin{bmatrix}
1\\
1
\end{bmatrix}
=
17.
\end{equation*}

The quadratic associated with this interval is
\begin{equation*}
\varphi_2(s)
=
c-2\alpha_2s+\beta_2s^2
=
29-44s+17s^2.
\end{equation*}
Its unconstrained minimizer is
\begin{equation*}
s_2^{\mathrm{unc}}
=
\frac{\alpha_2}{\beta_2}
=
\frac{22}{17}.
\end{equation*}
Since
\begin{equation*}
0<\frac{22}{17}<2,
\end{equation*}
this minimizer also lies strictly inside its interval:
\begin{equation*}
s_2^\ast
=
\mathrm{clamp}
\left(
\frac{\alpha_2}{\beta_2},\,0,\,2
\right)
=
\frac{22}{17}.
\end{equation*}
Its reduced score is
\begin{equation*}
f_2
=
-\frac{\alpha_2^2}{\beta_2}
=
-\frac{484}{17},
\end{equation*}
and its full layer-output error is
\begin{equation*}
\varphi_2(s_2^\ast)
=
c+f_2
=
29-\frac{484}{17}
=
\frac{9}{17}.
\end{equation*}

\paragraph{Step 5: Selecting the best interval.}
The smallest layer-output error found in each interval is
\begin{equation*}
\varphi_0=29,
\qquad
\varphi_1(s_1^\ast)=\frac{9}{2},
\qquad
\varphi_2(s_2^\ast)=\frac{9}{17}.
\end{equation*}
Since
\begin{equation*}
\frac{9}{17}
<
\frac{9}{2}
<
29,
\end{equation*}
\ScaleAlgorithm{} selects the data-aware scale
\begin{equation*}
s_{\mathrm{out}}^\ast
=
\frac{22}{17}.
\end{equation*}

\paragraph{Step 6: Comparing with the data-free optimum.}
Data-free scale optimization instead minimizes
\begin{equation*}
E_{\mathrm{df}}(s)
=
\|\mathbf{w}-s\,q(\mathbf{w};s)\|^2.
\end{equation*}
The transition scales and grid assignments are unchanged because they depend only on \(\mathbf{w}\) and \(\G\). For $s\in(0,2)$, where
\begin{equation*}
q(\mathbf{w};s)
=
\begin{bmatrix}
1\\
1
\end{bmatrix},
\end{equation*}
the data-free objective is
\begin{equation*}
E_{\mathrm{df}}(s)
=
(1-s)^2+(2-s)^2
=
5-6s+2s^2.
\end{equation*}
Its minimizer is
\begin{equation*}
s_{\mathrm{df}}^\ast
=
\frac{3}{2},
\qquad
E_{\mathrm{df}}(s_{\mathrm{df}}^\ast)
=
\frac{1}{2}.
\end{equation*}
For $s\in[2,4)$, the smallest possible data-free error is \(1\), while for $s\in[4,+\infty)$ it is \(5\). Hence \(s_{\mathrm{df}}^\ast=3/2\) is the globally optimal data-free scale.

Although the data-free and data-aware scales both induce the same RTN assignment,
\begin{equation*}
q(\mathbf{w};s_{\mathrm{df}}^\ast)
=
q(\mathbf{w};s_{\mathrm{out}}^\ast)
=
\begin{bmatrix}
1\\
1
\end{bmatrix},
\end{equation*}
their layer-output errors differ substantially. Evaluating the data-free optimal scale under the actual layer-output objective gives
\begin{equation*}
\begin{aligned}
E(s_{\mathrm{df}}^\ast)
&=
\left\|
\mathbf{X}
\left(
\mathbf{w}
-
\frac{3}{2}
\begin{bmatrix}
1\\
1
\end{bmatrix}
\right)
\right\|^2
\\
&=
\left\|
\begin{bmatrix}
3&1\\
0&1
\end{bmatrix}
\begin{bmatrix}
-1/2\\
1/2
\end{bmatrix}
\right\|^2
\\
&=
\left\|
\begin{bmatrix}
-1\\
1/2
\end{bmatrix}
\right\|^2
=
\frac{5}{4}.
\end{aligned}
\end{equation*}
By comparison, the data-aware scale selected by \ScaleAlgorithm{} gives
\begin{equation*}
E(s_{\mathrm{out}}^\ast)
=
\frac{9}{17}.
\end{equation*}
Consequently,
\begin{equation*}
\frac{E(s_{\mathrm{df}}^\ast)}
     {E(s_{\mathrm{out}}^\ast)}
=
\frac{5/4}{9/17}
=
\frac{85}{36}
\approx2.36.
\end{equation*}

Thus, even in this two-weight example, the scale that is globally optimal for weight reconstruction produces more than twice the layer-output error of the data-aware scale. Note that both scales induce exactly the same grid assignment, but the layer output error differs. The improvement comes  from selecting the scale according to how the calibration activations amplify and correlate the weight errors.
More broadly, this example shows how \ScaleAlgorithm{} enumerates the finitely many intervals on which the RTN assignments are constant, obtains the corresponding quadratic through inexpensive incremental updates, and selects the scale achieving the smallest layer-output reconstruction error.

\newpage
\section{Complementary Results}
\label{appendix:full_results}

This appendix presents the complete quantization results summarized in the main text. The results are organized by quantization granularity: channel-wise first, then group-wise.

\paragraph{Baseline implementations.} \label{check_list_setup}Baseline PTQ algorithms were run using the Brevitas quantization library~\citep{pappalardo2025xilinx}, with the exception of Beacon, whose implementation was not publicly available and was therefore implemented by us.

\paragraph{Runtime analysis.} All runtimes were measured on the same GPU model, with each model quantized and evaluated on a single GPU. Table~\ref{tab:runtime} reports the wall-clock runtimes of GPTQ with~\ScaleAlgorithmShort~on two Llama 3 models (1B and 3B) at 2-bit and 3-bit weight quantization. GPTQ $\odot$~\ScaleAlgorithmShort~introduces negligible overhead (between $0.99\times$ and $1.13\times$ the baseline GPTQ runtime), while GPTQ $+$~\ScaleAlgorithmShort~incurs a $1.25$–$1.37\times$ overhead due to the additional forward passes over the calibration data. Qronos itself runs at $1.09\times$–$1.13\times$ the GPTQ runtime, and Qronos $\odot$~\ScaleAlgorithmShort{} remains in the same regime at $1.12\times$–$1.22\times$, i.e.~\ScaleAlgorithmShort{} adds only $1.02\times$–$1.12\times$ on top of Qronos. In contrast, the gradient-based AutoRound~\citep{cheng2024optimize} requires $2.93\times$–$4.67\times$, so the gains reported in Tables~\ref{tab:autoround_llama_ppl} and~\ref{tab:autoround_llama_acc} are obtained at a fraction of its quantization cost.

\begin{table}[htbp]
\centering
\caption{Wall-clock absolute (in minutes) and relative (compared to GPTQ) runtimes of GPTQ and Qronos with~\ScaleAlgorithmShort~scale optimization, and of AutoRound~\citep{cheng2024optimize}. $+/\odot$ denote decoupled/interleaved optimization (see Section~\ref{sec:inte}). All relative runtimes are expressed with respect to GPTQ.}
\label{tab:runtime}
\resizebox{\textwidth}{!}{
\begin{tabular}{ll c cc cc cc cc cc}
\toprule
 & & GPTQ & \multicolumn{2}{c}{GPTQ $+$ PiSO} & \multicolumn{2}{c}{GPTQ $\odot$ PiSO} & \multicolumn{2}{c}{Qronos} & \multicolumn{2}{c}{Qronos $\odot$ PiSO} & \multicolumn{2}{c}{AutoRound} \\
\cmidrule(lr){3-3} \cmidrule(lr){4-5} \cmidrule(lr){6-7} \cmidrule(lr){8-9} \cmidrule(lr){10-11} \cmidrule(lr){12-13}
Model & Bits & min & min & $\times$ & min & $\times$ & min & $\times$ & min & $\times$ & min & $\times$ \\
\midrule
Llama-3.2-1B & 2 & 35.3 & 48.3 & $1.37$ & 36.4 & $1.03$ & 38.9 & $1.10$ & 39.5 & $1.12$ & 164.7 & $4.67$ \\
Llama-3.2-1B & 3 & 35.9 & 44.7 & $1.25$ & 35.5 & $0.99$ & 40.4 & $1.13$ & 41.4 & $1.15$ & 165.4 & $4.61$ \\
Llama-3.2-3B & 2 & 77.0 & 104.8 & $1.36$ & 85.9 & $1.11$ & 83.8 & $1.09$ & 88.8 & $1.15$ & 283.6 & $3.68$ \\
Llama-3.2-3B & 3 & 78.1 & 105.1 & $1.34$ & 88.3 & $1.13$ & 84.9 & $1.09$ & 95.5 & $1.22$ & 228.7 & $2.93$ \\
\bottomrule
\end{tabular}
}
\end{table}

\FloatBarrier

\paragraph{Scale Distribution.} \label{ap:scale-dist}
Figure~\ref{fig:scale-dist} shows the optimal channel-wise scale distributions, normalized by the absmax scale, for the data‑free ($\mathbf{X}=\tilde{\mathbf{X}}=\mathbf{I}$) and cross‑activation ($\mathbf{X}\neq\tilde{\mathbf{X}}$) objectives in third down‑projection layer of Llama‑3.2‑1B. One can clearly observe how the distribution of the scales varies depending on the objective with respect to which they are optimized. Notably, the distribution of optimal scales under the cross‑activation objective exhibits a larger spread, with a non‑negligible fraction exceeding the absmax value (ratios greater than~1). This behavior highlights one of the limitations of grid‑search–based methods such as CDQuant~\citep{nair2024cdquant}, which restrict the search space to the interval $(0, s_{\mathrm{absmax}}]$.
\label{sec:scale_dist}

\begin{figure}[t]
    \centering
    \begin{subfigure}[b]{0.48\textwidth}
        \centering
        \includegraphics[width=\textwidth]{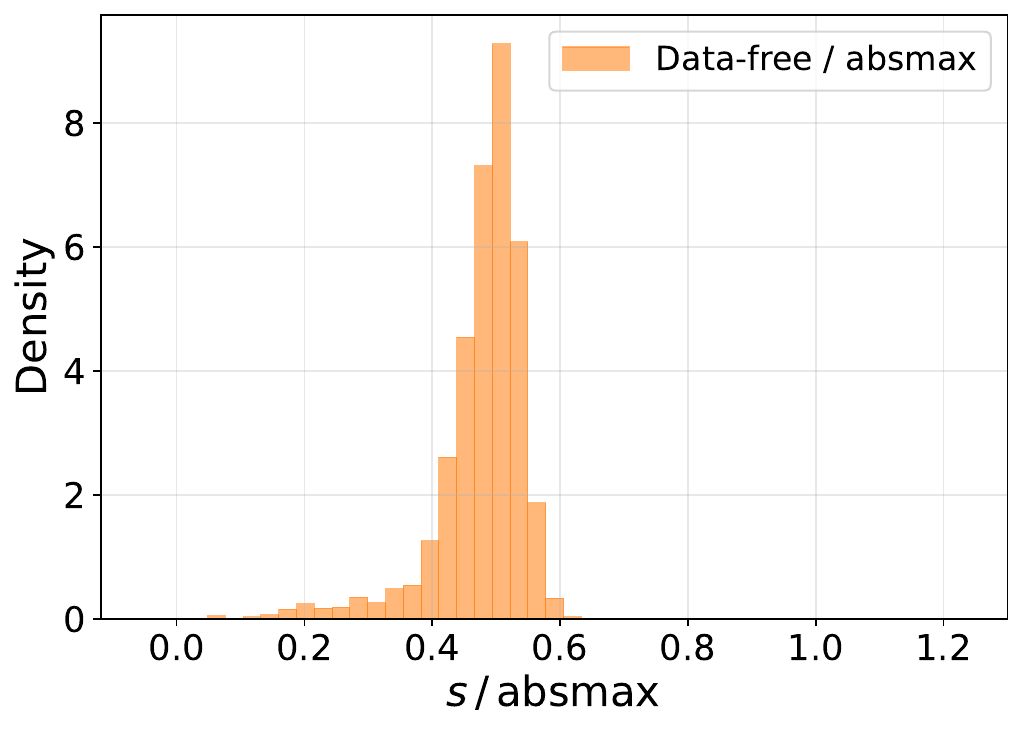}
        \caption*{$s^{*}=\argmin_{s}\|\mathbf{w}-s\,q(\mathbf{w};s)\|^{2}$}
        \label{fig:scale-dist-datafree}
    \end{subfigure}
    \hfill
    \begin{subfigure}[b]{0.48\textwidth}
        \centering
        \includegraphics[width=\textwidth]{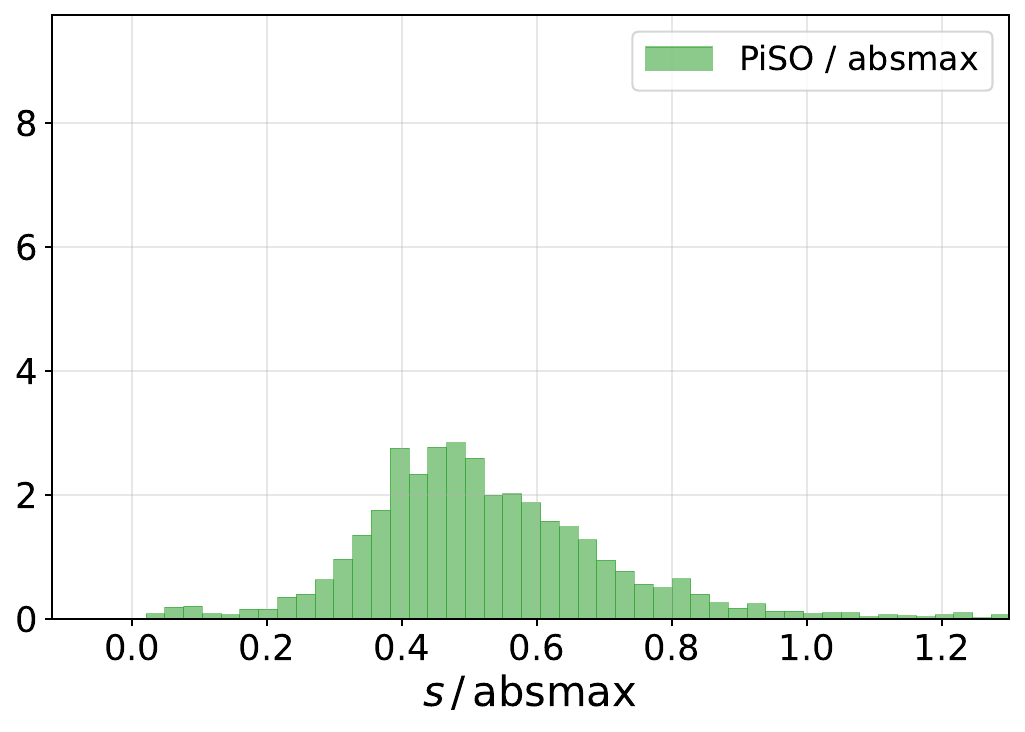}
        \caption*{$s^{*}=\argmin_{s}\|\mathbf{X}\mathbf{w}-s\,\tilde{\mathbf{X}}\,q(\mathbf{w};s)\|^{2}$}
        \label{fig:scale-dist-piso}
    \end{subfigure}
    \caption{Optimal scale distributions for the data‑free (left) and cross-activation (right) objectives, normalized by the absmax scales, for the 2048 output channels of the third down‑projection layer of Llama‑3.2‑1B.}
    \label{fig:scale-dist}
\end{figure}

\paragraph{Calibration size sensitivity}\label{ap:cal-size}
Figure~\ref{fig:cal-sweep-appendix} extends the calibration size sensitivity experiments from the Results section in the main text (Figure~\ref{fig:cal-sweep}). The results show that with either matched or mismatched calibration/evaluation datasets, PiSO continues to deliver large gains over the absmax baseline. When PiSO is not used, the gap between the matched/mismatched settings increases. PiSO is the most data-efficient strategy across domains. Under matched WikiText2 calibration, one sequence already gives strong RTN + PiSO perplexity (22.5). Under domain shift, a single sequence is noisier but performance stabilizes quickly (with as few as 4–16 sequences, out-of-domain calibration essentially reaches its asymptotic quality).

\begin{figure}[t]
    \centering
    \begin{subfigure}[b]{0.49\textwidth}
        \centering
        \includegraphics[width=\textwidth]{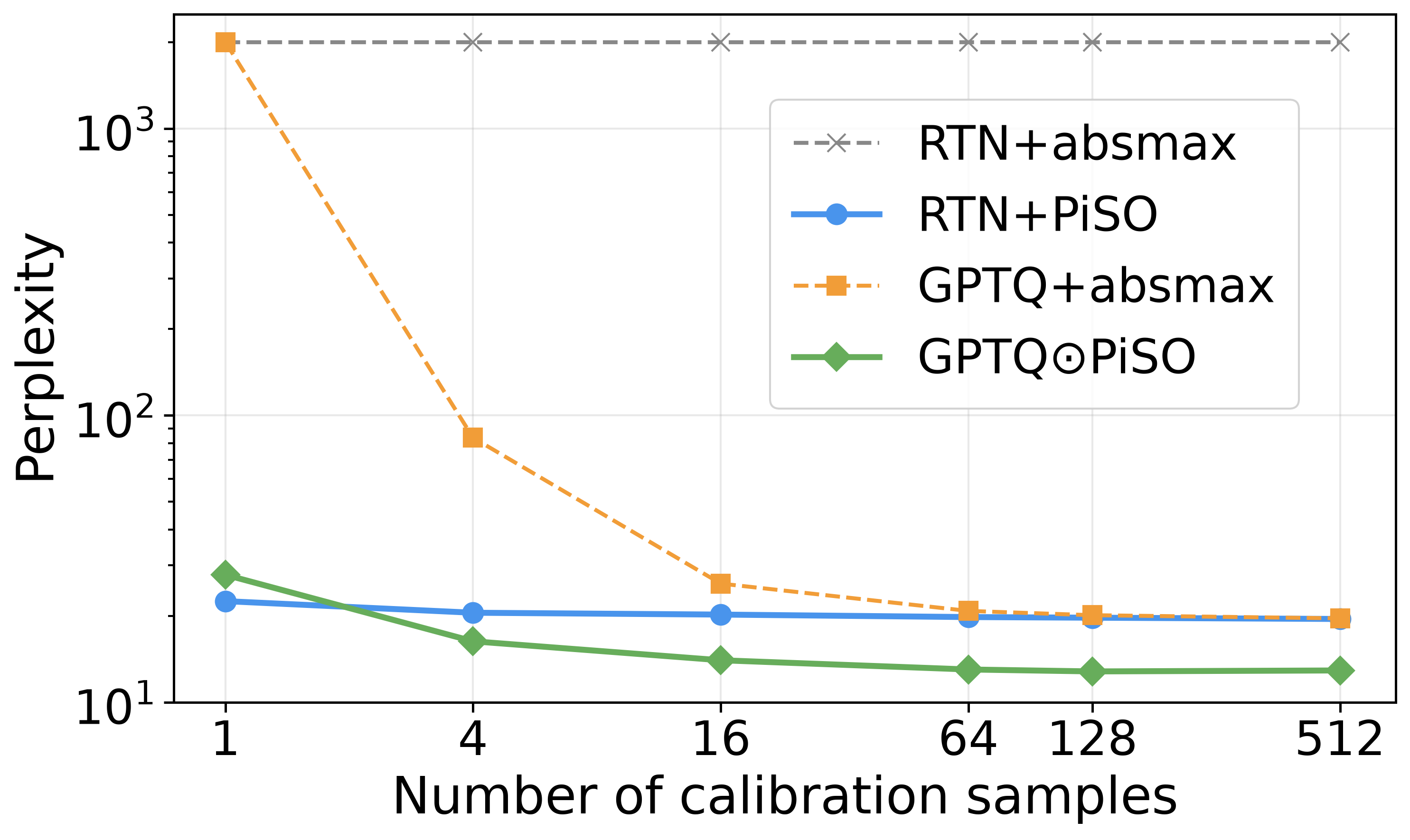}
        \caption{WikiText2 calibration, WikiText2 perplexity ($\downarrow$)}
    \end{subfigure}
    \hfill
    \begin{subfigure}[b]{0.49\textwidth}
        \centering
        \includegraphics[width=\textwidth]{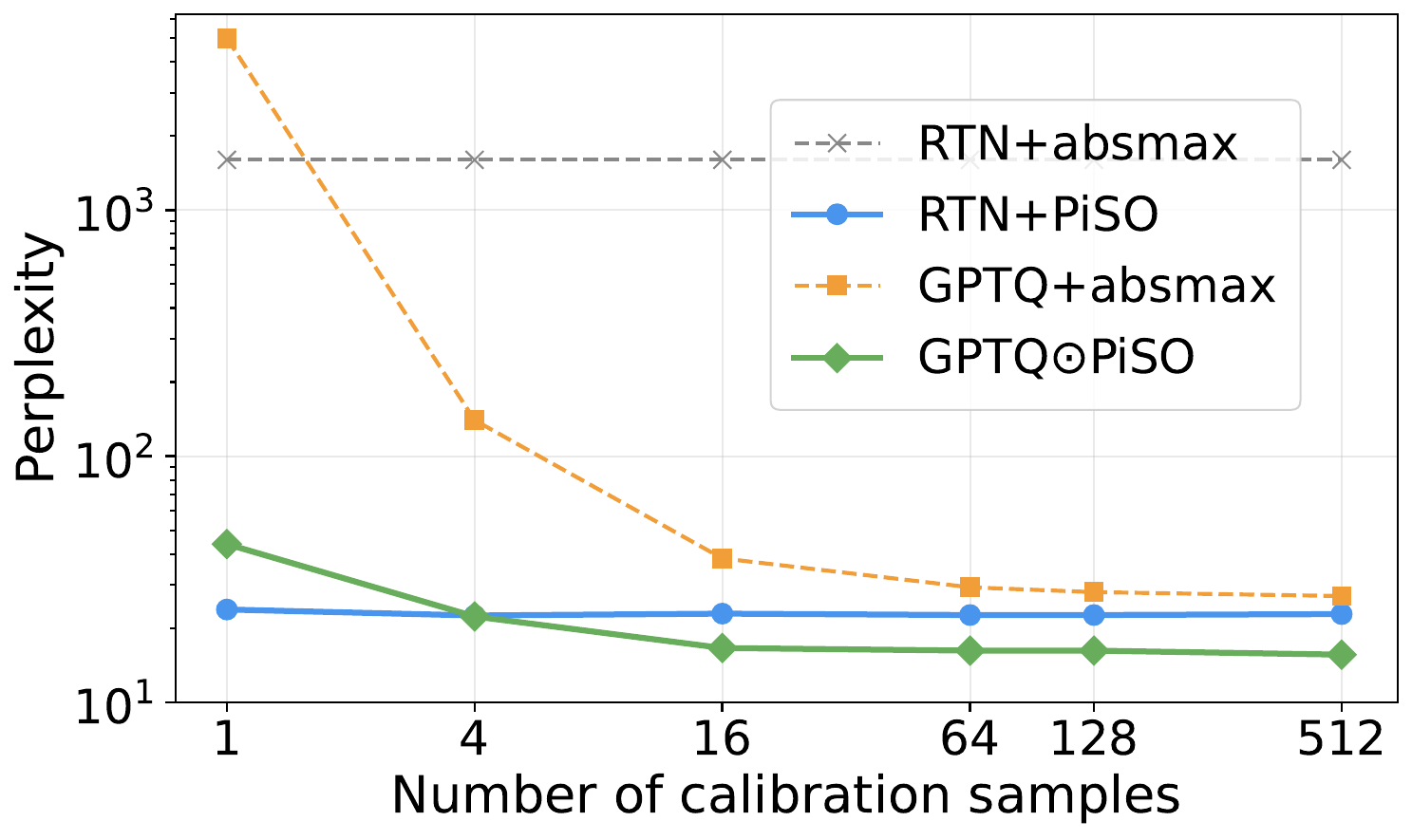}
        \caption{FineWeb calibration, WikiText2 perplexity ($\downarrow$)}
    \end{subfigure}
\caption{Effect of calibration set size on Llama-3.2-1B under 3-bit channel-wise weight quantization. The trend displayed in the main text for C4 calibration (Figure~\ref{fig:cal-sweep}) holds: \ScaleAlgorithm{} is stable with few samples, whereas GPTQ with absmax requires $\geq 64$ samples to converge. GPTQ$\odot$\ScaleAlgorithmShort{} achieves the best perplexity and improves the data efficiency of GPTQ+absmax .}
\label{fig:cal-sweep-appendix}
\end{figure}

\paragraph{Channel-wise Quantization.}
\label{sec:app-channel}
Tables~\ref{tab:channel_2-bit_expanded}--\ref{tab:channel_qwen_3-bit4-bit_expanded} report channel-wise weight-only quantization results at 2, 3, and 4~bits on Llama-3 and Qwen-2.5 Instruct models using the standard integer grid, unpacking the scale optimization objectives and integration strategies with error correction methods that are aggregated in Tables~\ref{tab:channel_2-bit} and~\ref{tab:channel_llama_3-bit4-bit}. Table~\ref{tab:channel_datafree_llama} demonstrates the catastrophic quality degradation of data-free scale selection strategies in channel-wise low-bit quantization. Table~\ref{tab:channel_e2m1_expanded} evaluates~\ScaleAlgorithmShort{} on the non-uniform \texttt{E2M1} (4-bit floating-point) grid, demonstrating that the algorithm applies to arbitrary quantization grids beyond integer formats. Finally, Tables~\ref{tab:autoround_llama_ppl} and~\ref{tab:autoround_llama_acc} compare~\ScaleAlgorithmShort{} against AutoRound~\citep{cheng2024optimize}, a gradient-based PTQ method that learns the rounding decisions and the scales jointly.

\FloatBarrier

\begin{table}[!htbp]
\centering
\caption{Channel-wise 2-bit quantization results on Llama-3 and Qwen-2.5 models, reporting WikiText2 perplexity ($\downarrow$) and average zero-shot accuracy ($\uparrow$). The ``Obj.'' column denotes the scale optimization objective (see Equation~\ref{eq:layer-obj}): $\mathbf{X}\text{-}\mathbf{X}$ (float activations), $\tilde{\mathbf{X}}\text{-}\tilde{\mathbf{X}}$ (self-activation), or $\mathbf{X}\text{-}\tilde{\mathbf{X}}$ (cross-activation); blank where not applicable (absmax, Beacon, AWQ). For PiSO, $+/\odot$ denote decoupled/interleaved optimization (see Section~\ref{sec:inte}).}
\label{tab:channel_2-bit_expanded}
\resizebox{\textwidth}{!}{
\begin{tabular}{ll ccc ccc ccc ccc}
\toprule
 & & \multicolumn{6}{c}{\textbf{2-bit Llama}} & \multicolumn{6}{c}{\textbf{2-bit Qwen}} \\
\cmidrule(lr){3-8} \cmidrule(lr){9-14}
 & & \multicolumn{3}{c}{WikiText2 ($\downarrow$)} & \multicolumn{3}{c}{0-shot ($\uparrow$)} & \multicolumn{3}{c}{WikiText2 ($\downarrow$)} & \multicolumn{3}{c}{0-shot ($\uparrow$)} \\
\cmidrule(lr){3-5} \cmidrule(lr){6-8} \cmidrule(lr){9-11} \cmidrule(lr){12-14}
 & Obj. & 1B & 3B & 8B & 1B & 3B & 8B & 1.5B & 3B & 7B & 1.5B & 3B & 7B \\
\midrule
BF16 & & 8.97 & 7.14 & 5.89 & 52.6 & 62.0 & 68.6 & 8.89 & 7.88 & 6.92 & 54.9 & 56.5 & 57.7 \\
\midrule
RTN + absmax & & 2e5 & 4e5 & 4e5 & 33.0 & 32.5 & 33.4 & 7e6 & 6e6 & 1e5 & 31.8 & 32.9 & 32.3 \\
RTN + \DataFreeAlgorithm & $\mathbf{I}\text{-}\mathbf{I}$ & 2e6 & 8e5 & 2e6 & 32.8 & 32.3 & 32.1 & 2e8 & 3e8 & 7e6 & 32.2 & 31.4 & 32.2 \\
RTN + \ScaleAlgorithmShort & $\mathbf{X}\text{-}\mathbf{X}$ & 1e5 & 4e4 & 2e5 & 32.0 & 31.6 & 32.5 & 1e4 & 2e4 & 9e3 & 31.8 & 32.2 & 32.3 \\
RTN + \ScaleAlgorithmShort & $\tilde{\mathbf{X}}\text{-}\tilde{\mathbf{X}}$ & 2e4 & 2e4 & 8e4 & 31.9 & 32.2 & 32.7 & 3e3 & 3e3 & 7e3 & 32.4 & 31.6 & 31.8 \\
RTN + \ScaleAlgorithmShort & $\mathbf{X}\text{-}\tilde{\mathbf{X}}$ & 3e3 & 3e3 & 3e3 & 32.3 & 31.6 & 32.0 & 2e3 & 2e3 & 4e3 & 31.8 & 32.7 & 32.2 \\
\midrule
AWQ &  & 2e6 & 8e5 & 1e6 & 32.7 & 32.0 & 32.1 & 9e7 & 7e8 & 9e6 & 32.4 & 31.7 & 32.5 \\
Beacon & & 52.8 & 72.4 & 25.5 & 33.8 & 33.3 & 36.8 & 1e3 & 33.2 & 21.9 & 32.5 & 32.4 & 33.9 \\
\midrule
GPTQ + absmax & & 3e3 & 415 & 331 & 31.6 & 31.9 & 31.5 & 675 & 596 & 72.7 & 33.0 & 32.8 & 32.1 \\
GPTQ + \DataFreeAlgorithm & $\mathbf{I}\text{-}\mathbf{I}$ & 3e3 & 9e3 & 1e6 & 31.7 & 32.3 & 31.8 & 353 & 1e5 & 185 & 32.1 & 31.1 & 34.0 \\
GPTQ + \ScaleAlgorithmShort & $\mathbf{X}\text{-}\mathbf{X}$ & 116 & 69.7 & 59.1 & 32.6 & 33.1 & 33.3 & 45.3 & 32.7 & \textbf{15.6} & 33.3 & 32.9 & 37.1 \\
GPTQ + \ScaleAlgorithmShort & $\tilde{\mathbf{X}}\text{-}\tilde{\mathbf{X}}$ & 92.5 & 47.5 & 32.5 & 33.1 & 33.2 & 34.7 & \textbf{39.2} & 38.5 & 15.9 & \textbf{34.4} & 33.6 & 34.9 \\
GPTQ $\odot$ \ScaleAlgorithmShort & $\tilde{\mathbf{X}}\text{-}\tilde{\mathbf{X}}$ & 54.7 & 29.5 & 22.3 & 35.5 & 36.3 & 40.2 & 44.7 & 30.6 & 17.0 & 33.7 & \textbf{34.4} & 34.9 \\
\midrule
Qronos + absmax & & 197 & 68.9 & 49.5 & 31.6 & 32.7 & 33.5 & 111 & 35.8 & 2e4 & 32.1 & 32.5 & 32.1 \\
Qronos + \DataFreeAlgorithm & $\mathbf{I}\text{-}\mathbf{I}$ & 5e3 & 1e4 & 3e5 & 31.2 & 31.5 & 31.6 & 3e3 & 657 & 1e4 & 31.2 & 32.6 & 32.6 \\
Qronos + \ScaleAlgorithmShort & $\mathbf{X}\text{-}\tilde{\mathbf{X}}$ & 216 & 84.8 & 76.0 & 32.1 & 33.3 & 33.0 & 196 & 27.4 & 24.5 & 32.3 & 32.9 & 36.1 \\
Qronos $\odot$ \ScaleAlgorithmShort & $\mathbf{X}\text{-}\tilde{\mathbf{X}}$ & \textbf{27.9} & \textbf{18.9} & \textbf{16.6} & \textbf{37.5} & \textbf{41.4} & \textbf{44.0} & 131 & \textbf{17.2} & 16.3 & 32.7 & \textbf{34.8} & \textbf{38.1} \\
\bottomrule
\end{tabular}
}
\end{table}

\begin{table}[!htbp]
\centering
\caption{Channel-wise 3-bit and 4-bit quantization results on Llama-3 models, reporting WikiText2 perplexity ($\downarrow$) and average zero-shot accuracy ($\uparrow$). The ``Obj.'' column denotes the scale optimization objective (see Equation~\ref{eq:layer-obj}): $\mathbf{X}\text{-}\mathbf{X}$ (float activations), $\tilde{\mathbf{X}}\text{-}\tilde{\mathbf{X}}$ (self-activation), or $\mathbf{X}\text{-}\tilde{\mathbf{X}}$ (cross-activation); blank where not applicable (absmax, Beacon, AWQ). For PiSO, $+/\odot$ denote decoupled/interleaved optimization (see Section~\ref{sec:inte}).}
\label{tab:channel_llama_3-bit4-bit_expanded}
\resizebox{\textwidth}{!}{
\begin{tabular}{ll ccc ccc ccc ccc}
\toprule
 & & \multicolumn{6}{c}{\textbf{3-bit}} & \multicolumn{6}{c}{\textbf{4-bit}} \\
\cmidrule(lr){3-8} \cmidrule(lr){9-14}
 & & \multicolumn{3}{c}{WikiText2 ($\downarrow$)} & \multicolumn{3}{c}{0-shot ($\uparrow$)} & \multicolumn{3}{c}{WikiText2 ($\downarrow$)} & \multicolumn{3}{c}{0-shot ($\uparrow$)} \\
\cmidrule(lr){3-5} \cmidrule(lr){6-8} \cmidrule(lr){9-11} \cmidrule(lr){12-14}
 & Obj. & 1B & 3B & 8B & 1B & 3B & 8B & 1B & 3B & 8B & 1B & 3B & 8B \\
\midrule
BF16 & & 8.97 & 7.14 & 5.89 & 52.6 & 62.0 & 68.6 & 8.97 & 7.14 & 5.89 & 52.6 & 62.0 & 68.6 \\
\midrule
RTN + absmax & & 2e3 & 436 & 1e3 & 32.8 & 35.7 & 38.0 & 22.3 & 9.61 & 7.83 & 46.3 & 55.6 & 64.7 \\
RTN + \DataFreeAlgorithm & $\mathbf{I}\text{-}\mathbf{I}$ & 9e3 & 6e4 & 5e4 & 31.9 & 32.4 & 32.2 & 136 & 8e3 & 9e4 & 33.8 & 31.8 & 32.2 \\
RTN + \ScaleAlgorithmShort & $\mathbf{X}\text{-}\mathbf{X}$ & 27.7 & 11.4 & 9.69 & 42.5 & 50.4 & 59.3 & 11.3 & 8.03 & 6.72 & 50.7 & 59.9 & 67.2 \\
RTN + \ScaleAlgorithmShort & $\tilde{\mathbf{X}}\text{-}\tilde{\mathbf{X}}$ & 29.7 & 11.5 & 9.96 & 41.3 & 49.2 & 59.1 & 11.5 & 8.04 & 6.74 & 49.2 & 60.2 & 66.7 \\
RTN + \ScaleAlgorithmShort & $\mathbf{X}\text{-}\tilde{\mathbf{X}}$ & 19.5 & 10.6 & 8.99 & 44.2 & 52.1 & 59.9 & 11.0 & 7.92 & 6.64 & 50.0 & 60.4 & 67.3 \\
\midrule
AWQ &  & 102 & 50.1 & 11.8 & 38.2 & 52.5 & 58.7 & 11.9 & 8.02 & 6.76 & 50.4 & 59.8 & 66.6 \\
Beacon & & 19.1 & 33.7 & 9.35 & 42.4 & 40.1 & 53.5 & 13.6 & 25.1 & 7.53 & 46.1 & 46.3 & 59.7 \\
\midrule
GPTQ + absmax & & 19.6 & 11.3 & 9.23 & 39.6 & 49.8 & 55.1 & 10.4 & 7.79 & 6.49 & 49.3 & 59.4 & 65.6 \\
GPTQ + \DataFreeAlgorithm & $\mathbf{I}\text{-}\mathbf{I}$ & 546 & 9e3 & 8e5 & 33.1 & 32.0 & 31.7 & 255 & 9e3 & 4e5 & 35.1 & 31.5 & 31.7 \\
GPTQ + \ScaleAlgorithmShort & $\mathbf{X}\text{-}\mathbf{X}$ & 12.9 & 9.19 & 7.42 & 47.5 & 54.6 & \textbf{63.0} & 9.94 & 7.63 & 6.32 & \textbf{52.0} & 60.2 & \textbf{68.3} \\
GPTQ + \ScaleAlgorithmShort & $\tilde{\mathbf{X}}\text{-}\tilde{\mathbf{X}}$ & 13.0 & 9.20 & 7.51 & 46.1 & 53.5 & 62.4 & 9.94 & 7.64 & 6.32 & 50.9 & 59.7 & 67.5 \\
GPTQ $\odot$ \ScaleAlgorithmShort & $\tilde{\mathbf{X}}\text{-}\tilde{\mathbf{X}}$ & 12.8 & 9.08 & 7.41 & 47.0 & \textbf{55.9} & 62.6 & 9.91 & 7.65 & 6.32 & 51.3 & 60.0 & 68.1 \\
\midrule
Qronos + absmax & & 15.0 & 9.83 & 8.27 & 40.6 & 51.4 & 57.9 & 10.0 & 7.64 & 6.36 & 50.2 & 60.1 & 65.2 \\
Qronos + \DataFreeAlgorithm & $\mathbf{I}\text{-}\mathbf{I}$ & 2e3 & 9e3 & 8e5 & 32.8 & 31.9 & 32.3 & 335 & 9e3 & 4e5 & 34.6 & 31.8 & 31.1 \\
Qronos + \ScaleAlgorithmShort & $\mathbf{X}\text{-}\tilde{\mathbf{X}}$ & 11.6 & 8.42 & 7.01 & \textbf{48.6} & 55.5 & 61.4 & 9.61 & \textbf{7.44} & 6.19 & \textbf{51.5} & \textbf{61.2} & 67.4 \\
Qronos $\odot$ \ScaleAlgorithmShort & $\mathbf{X}\text{-}\tilde{\mathbf{X}}$ & \textbf{11.6} & \textbf{8.41} & \textbf{6.98} & 48.0 & \textbf{55.6} & \textbf{62.7} & \textbf{9.59} & \textbf{7.44} & \textbf{6.18} & 51.1 & 60.5 & 67.0 \\
\bottomrule
\end{tabular}
}
\end{table}

\begin{table}[!htbp]
\centering
\caption{Channel-wise 3-bit and 4-bit quantization results on Qwen-2.5 models reporting WikiText2 perplexity ($\downarrow$) and average zero-shot accuracy ($\uparrow$). The ``Obj.'' column denotes the scale optimization objective (see Equation~\ref{eq:layer-obj}): $\mathbf{X}\text{-}\mathbf{X}$ (float activations), $\tilde{\mathbf{X}}\text{-}\tilde{\mathbf{X}}$ (self-activation), or $\mathbf{X}\text{-}\tilde{\mathbf{X}}$ (cross-activation); blank where not applicable (absmax, Beacon, AWQ). For PiSO, $+/\odot$ denote decoupled/interleaved optimization (see Section~\ref{sec:inte}).}
\label{tab:channel_qwen_3-bit4-bit_expanded}
\resizebox{\textwidth}{!}{
\begin{tabular}{ll ccc ccc ccc ccc}
\toprule
 & & \multicolumn{6}{c}{\textbf{3-bit}} & \multicolumn{6}{c}{\textbf{4-bit}} \\
\cmidrule(lr){3-8} \cmidrule(lr){9-14}
 & & \multicolumn{3}{c}{WikiText2 ($\downarrow$)} & \multicolumn{3}{c}{0-shot ($\uparrow$)} & \multicolumn{3}{c}{WikiText2 ($\downarrow$)} & \multicolumn{3}{c}{0-shot ($\uparrow$)} \\
\cmidrule(lr){3-5} \cmidrule(lr){6-8} \cmidrule(lr){9-11} \cmidrule(lr){12-14}
 & Obj. & 1.5B & 3B & 7B & 1.5B & 3B & 7B & 1.5B & 3B & 7B & 1.5B & 3B & 7B \\
\midrule
BF16 & & 8.89 & 7.88 & 6.92 & 54.9 & 56.5 & 57.7 & 8.89 & 7.88 & 6.92 & 54.9 & 56.5 & 57.7 \\
\midrule
RTN + absmax & & 861 & 5e5 & 9e7 & 33.7 & 32.1 & 32.7 & 13.9 & 1e4 & 9.13 & 49.5 & 31.6 & 53.5 \\
RTN + \DataFreeAlgorithm & $\mathbf{I}\text{-}\mathbf{I}$ & 142 & 1e5 & 535 & 34.6 & 32.4 & 33.5 & 14.3 & 13.9 & 10.9 & 51.9 & 50.6 & 55.0 \\
RTN + \ScaleAlgorithmShort & $\mathbf{X}\text{-}\mathbf{X}$ & 16.9 & 11.9 & 9.88 & 46.8 & 48.1 & 54.6 & 10.0 & 8.60 & 7.55 & 54.7 & 54.6 & 58.7 \\
RTN + \ScaleAlgorithmShort & $\tilde{\mathbf{X}}\text{-}\tilde{\mathbf{X}}$ & 16.8 & 12.0 & 10.0 & 44.6 & 46.8 & 53.2 & 10.0 & 8.60 & 7.56 & 54.5 & 55.2 & 58.0 \\
RTN + \ScaleAlgorithmShort & $\mathbf{X}\text{-}\tilde{\mathbf{X}}$ & 14.9 & 11.5 & 9.50 & 46.9 & 50.8 & 51.8 & 10.00 & 8.67 & 7.51 & 54.0 & 55.1 & 57.6 \\
\midrule
AWQ &  & 27.0 & 16.6 & 20.3 & 45.3 & 47.2 & 49.2 & 10.7 & 9.02 & 7.87 & 51.7 & 54.4 & 56.7 \\
Beacon & & 380 & 12.3 & 12.6 & 31.6 & 39.9 & 40.4 & 142 & 10.2 & 10.8 & 33.6 & 44.8 & 43.2 \\
\midrule
GPTQ + absmax & & 14.1 & 11.5 & 9.23 & 41.1 & 42.0 & 44.6 & 9.76 & 8.54 & 7.34 & 53.2 & 54.7 & 55.0 \\
GPTQ + \DataFreeAlgorithm & $\mathbf{I}\text{-}\mathbf{I}$ & 29.2 & 914 & 8.71 & 43.7 & 32.3 & 52.2 & 14.2 & 8.72 & 9.02 & 43.2 & 54.7 & \textbf{60.1} \\
GPTQ + \ScaleAlgorithmShort & $\mathbf{X}\text{-}\mathbf{X}$ & 11.5 & 9.43 & 7.90 & 48.2 & \textbf{55.0} & 56.5 & 9.50 & 8.26 & 7.17 & \textbf{55.2} & 55.7 & 59.0 \\
GPTQ + \ScaleAlgorithmShort & $\tilde{\mathbf{X}}\text{-}\tilde{\mathbf{X}}$ & 11.6 & 9.45 & 7.91 & 47.0 & 53.3 & 55.7 & 9.49 & 8.25 & 7.18 & 53.5 & \textbf{57.8} & 59.5 \\
GPTQ $\odot$ \ScaleAlgorithmShort & $\tilde{\mathbf{X}}\text{-}\tilde{\mathbf{X}}$ & \textbf{11.4} & 9.41 & 7.82 & \textbf{50.6} & 52.0 & 57.0 & 9.47 & 8.26 & 7.15 & 54.7 & 56.9 & 58.5 \\
\midrule
Qronos + absmax & & 16.7 & 10.4 & 8.86 & 35.3 & 41.9 & 52.2 & 9.66 & 8.42 & 7.30 & 51.4 & 54.2 & 55.6 \\
Qronos + \DataFreeAlgorithm & $\mathbf{I}\text{-}\mathbf{I}$ & 105 & 168 & 11.3 & 33.1 & 32.0 & 43.6 & 16.8 & 10.1 & 8.83 & 41.9 & 45.0 & 59.6 \\
Qronos + \ScaleAlgorithmShort & $\mathbf{X}\text{-}\tilde{\mathbf{X}}$ & 11.8 & \textbf{9.10} & 7.73 & 43.8 & 53.5 & \textbf{58.3} & 9.36 & \textbf{8.19} & 7.12 & 52.0 & 56.8 & 58.1 \\
Qronos $\odot$ \ScaleAlgorithmShort & $\mathbf{X}\text{-}\tilde{\mathbf{X}}$ & 12.0 & \textbf{9.10} & \textbf{7.73} & 39.9 & 53.8 & 57.5 & \textbf{9.34} & 8.20 & \textbf{7.11} & 52.5 & 54.7 & 58.1 \\
\bottomrule
\end{tabular}
}
\end{table}

\begin{table}[!htbp]
\centering
\caption{Channel-wise quantization results on Llama-3 models, reporting WikiText2 perplexity ($\downarrow$) and average zero-shot accuracy ($\uparrow$). None of these scale selection methods requires calibration data: the grid search of the Brevitas library \citep{pappalardo2025xilinx} approximates the minimizer of the weight reconstruction objective by evaluating a uniform grid of candidates; while data-free optimizes this objective analytically.}
\label{tab:channel_datafree_llama}
\resizebox{\textwidth}{!}{
\begin{tabular}{l c c c c c c c c c c c c}
\toprule
 & \multicolumn{4}{c}{\textbf{2-bit}} & \multicolumn{4}{c}{\textbf{3-bit}} & \multicolumn{4}{c}{\textbf{4-bit}} \\
\cmidrule(lr){2-5} \cmidrule(lr){6-9} \cmidrule(lr){10-13}
 & \multicolumn{2}{c}{WikiText2 ($\downarrow$)} & \multicolumn{2}{c}{0-shot ($\uparrow$)} & \multicolumn{2}{c}{WikiText2 ($\downarrow$)} & \multicolumn{2}{c}{0-shot ($\uparrow$)} & \multicolumn{2}{c}{WikiText2 ($\downarrow$)} & \multicolumn{2}{c}{0-shot ($\uparrow$)} \\
\cmidrule(lr){2-3} \cmidrule(lr){4-5} \cmidrule(lr){6-7} \cmidrule(lr){8-9} \cmidrule(lr){10-11} \cmidrule(lr){12-13}
 & 1B & 3B & 1B & 3B & 1B & 3B & 1B & 3B & 1B & 3B & 1B & 3B \\
\midrule
BF16 & 8.97 & 7.14 & 52.6 & 62.0 & 8.97 & 7.14 & 52.6 & 62.0 & 8.97 & 7.14 & 52.6 & 62.0 \\
\midrule
RTN + grid search & 2e6 & 8e5 & \textbf{32.9} & 32.3 & 4e3 & 1e5 & 32.0 & 31.7 & \textbf{108} & \textbf{4e3} & 35.0 & \textbf{32.3} \\
RTN + \DataFreeAlgorithm & 2e6 & 8e5 & 32.8 & 32.3 & 9e3 & 6e4 & 31.9 & \textbf{32.4} & 136 & 8e3 & 33.8 & 31.8 \\
\midrule
GPTQ + grid search & \textbf{4e3} & \textbf{7e3} & 32.3 & \textbf{32.5} & 452 & \textbf{9e3} & 32.8 & 32.4 & 258 & 8e3 & \textbf{35.1} & 31.8 \\
GPTQ + \DataFreeAlgorithm & 5e3 & 8e3 & 31.8 & 32.4 & \textbf{390} & 1e4 & \textbf{33.3} & 32.1 & 300 & 9e3 & 34.6 & 31.8 \\
\bottomrule
\end{tabular}
}
\end{table}

\begin{table}[!htbp]
\centering
\caption{Channel-wise 4-bit  quantization results using the \texttt{E2M1} (\texttt{FP4}) floating-point format on Llama-3 models, reporting WikiText2 perplexity ($\downarrow$) and average zero-shot accuracy ($\uparrow$). The ``Obj.''~column specifies the scale-optimization objective (see Equation~\ref{eq:layer-obj}): $\mathbf{X}\text{-}\mathbf{X}$ (float activations), $\tilde{\mathbf{X}}\text{-}\tilde{\mathbf{X}}$ (self-activation), or $\mathbf{X}\text{-}\tilde{\mathbf{X}}$ (cross-activation); blank where not applicable (absmax).}
\label{tab:channel_e2m1_expanded}
\begin{tabular}{ll cc cc}
\toprule
 & & \multicolumn{2}{c}{WikiText2 ($\downarrow$)} & \multicolumn{2}{c}{0-shot ($\uparrow$)} \\
\cmidrule(lr){3-4} \cmidrule(lr){5-6}
 & Obj. & 1B & 3B & 1B & 3B \\
\midrule
BF16 & & 8.97 & 7.14 & 52.6 & 62.0 \\
\midrule
RTN + absmax &  & 12.6 & 8.39 & 49.2 & 59.5 \\
RTN + \ScaleAlgorithmShort & $\mathbf{X}\text{-}\mathbf{X}$ & 10.4 & 7.75 & 50.7 & 60.5 \\
RTN + \ScaleAlgorithmShort & $\tilde{\mathbf{X}}\text{-}\tilde{\mathbf{X}}$ & 10.4 & 7.75 & 50.5 & 61.0 \\
RTN + \ScaleAlgorithmShort & $\mathbf{X}\text{-}\tilde{\mathbf{X}}$ & 10.3 & 7.69 & 50.6 & 59.8 \\
\midrule
GPTQ + absmax &  & 9.78 & 7.52 & 50.8 & 60.1 \\
GPTQ + \ScaleAlgorithmShort & $\mathbf{X}\text{-}\mathbf{X}$ & 9.67 & 7.52 & \textbf{51.1} & \textbf{62.1} \\
GPTQ + \ScaleAlgorithmShort & $\tilde{\mathbf{X}}\text{-}\tilde{\mathbf{X}}$ & \textbf{9.67} & \textbf{7.52} & 50.5 & 60.8 \\
GPTQ $\odot$ \ScaleAlgorithmShort & $\tilde{\mathbf{X}}\text{-}\tilde{\mathbf{X}}$ & 9.69 & 7.52 & 50.5 & 60.7 \\
\bottomrule
\end{tabular}
\end{table}

\begin{table}[!htbp]
\centering
\caption{Comparison against AutoRound~\citep{cheng2024optimize} for channel-wise weight-only quantization of Llama-3 models, reporting WikiText2 perplexity ($\downarrow$). All methods share the same protocol: 512 WikiText2 calibration sequences of 2048 tokens, symmetric per-channel weight quantization, and the same zero-shot task suite. For AutoRound, the hyperparameter configuration of~\citet{cheng2024optimize} was adapted. For error correction with PiSO, $+/\odot$ denote decoupled/interleaved (Section~\ref{sec:inte}).}
\label{tab:autoround_llama_ppl}
\begin{tabular}{l ccc ccc ccc}
\toprule
 & \multicolumn{3}{c}{\textbf{2-bit}} & \multicolumn{3}{c}{\textbf{3-bit}} & \multicolumn{3}{c}{\textbf{4-bit}} \\
\cmidrule(lr){2-4} \cmidrule(lr){5-7} \cmidrule(lr){8-10}
 & 1B & 3B & 8B & 1B & 3B & 8B & 1B & 3B & 8B \\
\midrule
BF16 & 8.97 & 7.14 & 5.89 & 8.97 & 7.14 & 5.89 & 8.97 & 7.14 & 5.89 \\
\midrule
RTN + absmax & 2e5 & 4e5 & 4e5 & 2e3 & 436 & 1e3 & 22.3 & 9.61 & 7.83 \\
RTN + \ScaleAlgorithmShort & 3e3 & 3e3 & 3e3 & 19.5 & 10.6 & 8.99 & 11.0 & 7.92 & 6.64 \\
\midrule
AutoRound & 87.6 & 37.7 & 35.7 & 14.8 & 10.0 & 8.72 & 10.0 & 7.57 & 6.48 \\
\midrule
GPTQ + \ScaleAlgorithmShort & 92.5 & 47.5 & 32.5 & 12.9 & 9.19 & 7.42 & 9.94 & 7.63 & 6.32 \\
GPTQ $\odot$ \ScaleAlgorithmShort & 54.7 & 29.5 & 22.3 & 12.8 & 9.08 & 7.41 & 9.91 & 7.65 & 6.32 \\
\midrule
Qronos + \ScaleAlgorithmShort & 216 & 84.8 & 76.0 & \textbf{11.6} & 8.42 & 7.01 & 9.61 & \textbf{7.44} & 6.19 \\
Qronos $\odot$ \ScaleAlgorithmShort & \textbf{27.9} & \textbf{18.9} & \textbf{16.6} & 11.6 & \textbf{8.41} & \textbf{6.98} & \textbf{9.59} & 7.44 & \textbf{6.18} \\
\bottomrule
\end{tabular}
\end{table}

\begin{table}[!htbp]
\centering
\caption{Comparison against AutoRound~\citep{cheng2024optimize} for channel-wise weight-only quantization of Llama-3 models, reporting average zero-shot accuracy ($\uparrow$). All methods share the same protocol: 512 WikiText2 calibration sequences of 2048 tokens, symmetric per-channel weight quantization, and the same zero-shot task suite. For AutoRound, the hyperparameter configuration of~\citet{cheng2024optimize} was adopted. For error correction with PiSO, $+/\odot$ denote decoupled/interleaved (Section~\ref{sec:inte}).}
\label{tab:autoround_llama_acc}
\begin{tabular}{l ccc ccc ccc}
\toprule
 & \multicolumn{3}{c}{\textbf{2-bit}} & \multicolumn{3}{c}{\textbf{3-bit}} & \multicolumn{3}{c}{\textbf{4-bit}} \\
\cmidrule(lr){2-4} \cmidrule(lr){5-7} \cmidrule(lr){8-10}
 & 1B & 3B & 8B & 1B & 3B & 8B & 1B & 3B & 8B \\
\midrule
BF16 & 52.6 & 62.0 & 68.6 & 52.6 & 62.0 & 68.6 & 52.6 & 62.0 & 68.6 \\
\midrule
RTN + absmax & 33.0 & 32.5 & 33.4 & 32.8 & 35.7 & 38.0 & 46.3 & 55.6 & 64.7 \\
RTN + \ScaleAlgorithmShort & 32.3 & 31.6 & 32.0 & 44.2 & 52.1 & 59.9 & 50.0 & 60.4 & 67.3 \\
\midrule
AutoRound & 33.2 & 34.1 & 32.5 & 45.6 & 53.6 & 60.0 & 51.1 & 59.9 & 66.4 \\
\midrule
GPTQ + \ScaleAlgorithmShort & 33.1 & 33.2 & 34.7 & 47.5 & 54.6 & \textbf{63.0} & \textbf{52.0} & 60.2 & \textbf{68.3} \\
GPTQ $\odot$ \ScaleAlgorithmShort & 35.5 & 36.3 & 40.2 & 47.0 & \textbf{55.9} & 62.6 & 51.3 & 60.0 & 68.1 \\
\midrule
Qronos + \ScaleAlgorithmShort & 32.1 & 33.3 & 33.0 & \textbf{48.6} & 55.5 & 61.4 & 51.5 & \textbf{61.2} & 67.4 \\
Qronos $\odot$ \ScaleAlgorithmShort & \textbf{37.5} & \textbf{41.4} & \textbf{44.0} & 48.0 & 55.6 & 62.7 & 51.1 & 60.5 & 67.0 \\
\bottomrule
\end{tabular}
\end{table}

\FloatBarrier

\paragraph{Extended Model Coverage.}
\label{sec:app-models}
Tables~\ref{tab:channel_qwen14b} and~\ref{tab:channel_gemma} extend the channel-wise 3-bit evaluation to Qwen2.5-14B-Instruct and to the Gemma-2 Instruct models. The trends match those observed on the smaller models:~\ScaleAlgorithmShort{} recovers near-BF16 quality where the absmax scales collapse, improving RTN from $77.5$ / $36.6$ to $7.20$ / $60.1$ and GPTQ from $7.37$ / $51.5$ to $6.49$ / $62.4$ on Qwen2.5-14B-Instruct, and RTN from $13.6$ / $58.3$ to $8.57$ / $63.8$ and GPTQ from $8.52$ / $61.8$ to $8.03$ / $65.9$ on Gemma-2-9B-it.

\begin{table}[!htbp]
\centering
\caption{Channel-wise 3-bit quantization results on Qwen2.5-14B-Instruct, reporting WikiText2 perplexity ($\downarrow$) and average zero-shot accuracy ($\uparrow$). The ``Obj.'' column denotes the scale optimization objective (see Equation~\ref{eq:layer-obj}): $\mathbf{X}\text{-}\mathbf{X}$ (float activations), $\tilde{\mathbf{X}}\text{-}\tilde{\mathbf{X}}$ (self-activation), or $\mathbf{X}\text{-}\tilde{\mathbf{X}}$ (cross-activation); blank where not applicable (absmax). For \ScaleAlgorithmShort{}, $+/\odot$ denote decoupled/interleaved optimization (see Section~\ref{sec:inte}).}
\label{tab:channel_qwen14b}
\begin{tabular}{ll cc}
\toprule
 & Obj. & WikiText2 ($\downarrow$) & 0-shot ($\uparrow$) \\
\midrule
BF16 &  & 5.73 & 61.6 \\
\midrule
RTN + absmax &  & 77.5 & 36.6 \\
RTN + \DataFreeAlgorithm & $\mathbf{I}\text{-}\mathbf{I}$ & 199 & 37.5 \\
RTN + \ScaleAlgorithmShort & $\mathbf{X}\text{-}\mathbf{X}$ & 7.48 & 61.7 \\
RTN + \ScaleAlgorithmShort & $\tilde{\mathbf{X}}\text{-}\tilde{\mathbf{X}}$ & 7.49 & 61.4 \\
RTN + \ScaleAlgorithmShort & $\mathbf{X}\text{-}\tilde{\mathbf{X}}$ & 7.20 & 60.1 \\
\midrule
GPTQ + absmax &  & 7.37 & 51.5 \\
GPTQ + \DataFreeAlgorithm & $\mathbf{I}\text{-}\mathbf{I}$ & 7.04 & 55.7 \\
GPTQ + \ScaleAlgorithmShort & $\mathbf{X}\text{-}\mathbf{X}$ & 6.49 & 62.4 \\
GPTQ + \ScaleAlgorithmShort & $\tilde{\mathbf{X}}\text{-}\tilde{\mathbf{X}}$ & 6.50 & \textbf{62.5} \\
GPTQ $\odot$ \ScaleAlgorithmShort & $\tilde{\mathbf{X}}\text{-}\tilde{\mathbf{X}}$ & \textbf{6.47} & 61.7 \\
\bottomrule
\end{tabular}
\end{table}

\begin{table}[!htbp]
\centering
\caption{Channel-wise 3-bit quantization results on Gemma-2 Instruct models, reporting WikiText2 perplexity ($\downarrow$) and average zero-shot accuracy ($\uparrow$). The ``Obj.'' column denotes the scale optimization objective (see Equation~\ref{eq:layer-obj}): $\mathbf{X}\text{-}\mathbf{X}$ (float activations), $\tilde{\mathbf{X}}\text{-}\tilde{\mathbf{X}}$ (self-activation), or $\mathbf{X}\text{-}\tilde{\mathbf{X}}$ (cross-activation); blank where not applicable (absmax). For \ScaleAlgorithmShort{}, $+/\odot$ denote decoupled/interleaved optimization (see Section~\ref{sec:inte}).}
\label{tab:channel_gemma}
\begin{tabular}{ll cc cc}
\toprule
 & & \multicolumn{2}{c}{WikiText2 ($\downarrow$)} & \multicolumn{2}{c}{0-shot ($\uparrow$)} \\
\cmidrule(lr){3-4} \cmidrule(lr){5-6}
 & Obj. & 2B & 9B & 2B & 9B \\
\midrule
BF16 &  & 10.7 & 7.59 & 57.6 & 65.8 \\
\midrule
RTN + absmax &  & 41.2 & 13.6 & 47.2 & 58.3 \\
RTN + \ScaleAlgorithmShort & $\mathbf{X}\text{-}\mathbf{X}$ & 14.4 & 8.64 & 55.2 & 64.4 \\
RTN + \ScaleAlgorithmShort & $\tilde{\mathbf{X}}\text{-}\tilde{\mathbf{X}}$ & 14.4 & 8.68 & 55.0 & 64.4 \\
RTN + \ScaleAlgorithmShort & $\mathbf{X}\text{-}\tilde{\mathbf{X}}$ & 14.3 & 8.57 & 54.5 & 63.8 \\
\midrule
GPTQ + absmax &  & 13.5 & 8.52 & 52.4 & 61.8 \\
GPTQ + \ScaleAlgorithmShort & $\mathbf{X}\text{-}\mathbf{X}$ & \textbf{12.1} & \textbf{8.03} & 54.9 & \textbf{65.9} \\
GPTQ + \ScaleAlgorithmShort & $\tilde{\mathbf{X}}\text{-}\tilde{\mathbf{X}}$ & \textbf{12.1} & \textbf{8.03} & 55.8 & 65.1 \\
GPTQ $\odot$ \ScaleAlgorithmShort & $\tilde{\mathbf{X}}\text{-}\tilde{\mathbf{X}}$ & 12.2 & 8.06 & \textbf{56.0} & 65.6 \\
\bottomrule
\end{tabular}
\end{table}

\paragraph{Extended Task Coverage.}
\label{sec:app-tasks}
Table~\ref{tab:extended_tasks_qwen} extends the evaluation of channel-wise 3-bit Qwen-2.5 Instruct models to MMLU~\citep{hendrycks2021measuring}, IFEval~\citep{zhou2023instruction}, and LongBench~\citep{bai2024longbench}. The trends are consistent with those observed for WikiText2 perplexity and average zero-shot accuracy: optimizing the scales with~\ScaleAlgorithmShort{} improves over the absmax scales on every task, model, and pipeline. The gap is widest without error correction, where absmax collapses to chance-level MMLU and near-zero LongBench performance (MMLU $25.5 \to 55.2$ and LongBench $1.2 \to 17.0$ on the 3B model).

\begin{table}[!htbp]
\centering
\caption{Channel-wise 3-bit quantization results on Qwen-2.5 Instruct models, reporting WikiText2 perplexity ($\downarrow$), average zero-shot accuracy ($\uparrow$), and MMLU, IFEval, and LongBench performance ($\uparrow$). The ``Obj.'' column denotes the scale optimization objective (see Equation~\ref{eq:layer-obj}): $\mathbf{X}\text{-}\mathbf{X}$ (float activations), $\tilde{\mathbf{X}}\text{-}\tilde{\mathbf{X}}$ (self-activation), or $\mathbf{X}\text{-}\tilde{\mathbf{X}}$ (cross-activation); blank where not applicable (absmax). For \ScaleAlgorithmShort{}, $+/\odot$ denote decoupled/interleaved optimization (see Section~\ref{sec:inte}).}
\label{tab:extended_tasks_qwen}
\resizebox{\textwidth}{!}{
\begin{tabular}{ll cc cc cc cc cc}
\toprule
 & & \multicolumn{2}{c}{WikiText2 ($\downarrow$)} & \multicolumn{2}{c}{0-shot ($\uparrow$)} & \multicolumn{2}{c}{MMLU ($\uparrow$)} & \multicolumn{2}{c}{IFEval ($\uparrow$)} & \multicolumn{2}{c}{LongBench ($\uparrow$)} \\
\cmidrule(lr){3-4} \cmidrule(lr){5-6} \cmidrule(lr){7-8} \cmidrule(lr){9-10} \cmidrule(lr){11-12}
 & Obj. & 1.5B & 3B & 1.5B & 3B & 1.5B & 3B & 1.5B & 3B & 1.5B & 3B \\
\midrule
BF16 &  & 8.89 & 7.88 & 54.9 & 56.5 & 60.1 & 65.5 & 27.0 & 47.7 & 28.5 & 28.4 \\
\midrule
RTN + absmax &  & 861 & 5e5 & 33.7 & 32.1 & 24.6 & 25.5 & 9.4 & 11.3 & 2.6 & 1.2 \\
RTN + \ScaleAlgorithmShort & $\mathbf{X}\text{-}\tilde{\mathbf{X}}$ & 14.9 & 11.5 & 46.9 & 50.8 & 42.6 & 55.2 & 18.1 & \textbf{29.2} & 15.0 & 17.0 \\
\midrule
GPTQ + absmax &  & 14.1 & 11.5 & 41.1 & 42.0 & 25.7 & 29.9 & 10.9 & 14.6 & 13.3 & 12.8 \\
GPTQ + \ScaleAlgorithmShort & $\mathbf{X}\text{-}\mathbf{X}$ & 11.5 & 9.43 & 48.2 & \textbf{55.0} & \textbf{45.3} & 54.0 & \textbf{21.3} & 25.9 & \textbf{19.8} & \textbf{19.3} \\
GPTQ $\odot$ \ScaleAlgorithmShort & $\tilde{\mathbf{X}}\text{-}\tilde{\mathbf{X}}$ & \textbf{11.4} & \textbf{9.41} & \textbf{50.6} & 52.0 & 40.2 & \textbf{55.6} & 16.6 & 27.7 & 19.1 & 18.7 \\
\bottomrule
\end{tabular}
}
\end{table}

\FloatBarrier

\paragraph{Group-wise Quantization.}
\label{sec:app-group}

Tables~\ref{tab:group_expanded_3B_ppl} and~\ref{tab:group_expanded_3B_acc} report group-wise quantization results, with group sizes 16 and 32, on Llama-3.2-3B and Qwen-2.5-3B, covering all combinations of independent/sequential objectives and decoupled/interleaved integration strategies (Sections~\ref{sec:group-wise} and~\ref{sec:inte}) that are aggregated in the main text (Table~\ref{tab:per_group_3B_ppl}). Table~\ref{tab:group_combined_big_g16} shows the group-wise quantization results for Llama-3.1-8B and Qwen-2.5-7B for group size 16.

\begin{table}[!htbp]
\centering
\caption{\textbf{WikiText2 perplexity ($\downarrow$)} results for group-wise quantization on Llama-3.2-3B and Qwen-2.5-3B with group sizes 16 and 32 (G16/G32). The ``Obj.'' column denotes the scale optimization objective (see Equation~\ref{eq:layer-obj}): $\mathbf{X}\text{-}\mathbf{X}$ (float activations), $\tilde{\mathbf{X}}\text{-}\tilde{\mathbf{X}}$ (self-activation), or $\mathbf{X}\text{-}\tilde{\mathbf{X}}$ (cross-activation); blank where not applicable (absmax, AWQ). In the ``Var.'' column, $\mathrm{I}/\mathrm{S}$ denotes independent/sequential group-wise heuristics (see Section~\ref{sec:group-wise}).}
\label{tab:group_expanded_3B_ppl}
\resizebox{\textwidth}{!}{
\begin{tabular}{lll cccc cccc cccc}
\toprule
 & & & \multicolumn{4}{c}{\textbf{2-bit}} & \multicolumn{4}{c}{\textbf{3-bit}} & \multicolumn{4}{c}{\textbf{4-bit}} \\
\cmidrule(lr){4-7} \cmidrule(lr){8-11} \cmidrule(lr){12-15}
 & Var. & Obj. & \multicolumn{2}{c}{Llama 3B} & \multicolumn{2}{c}{Qwen 3B} & \multicolumn{2}{c}{Llama 3B} & \multicolumn{2}{c}{Qwen 3B} & \multicolumn{2}{c}{Llama 3B} & \multicolumn{2}{c}{Qwen 3B} \\
\cmidrule(lr){4-5} \cmidrule(lr){6-7} \cmidrule(lr){8-9} \cmidrule(lr){10-11} \cmidrule(lr){12-13} \cmidrule(lr){14-15}
 & & & G16 & G32 & G16 & G32 & G16 & G32 & G16 & G32 & G16 & G32 & G16 & G32 \\
\midrule
BF16 & & & 7.14 & 7.14 & 7.88 & 7.88 & 7.14 & 7.14 & 7.88 & 7.88 & 7.14 & 7.14 & 7.88 & 7.88 \\
\midrule
RTN + absmax &  &  & 8e5 & 8e5 & 3e8 & 2e8 & 4e3 & 11.7 & 21.0 & 102 & 7.69 & 7.72 & 8.78 & 8.91 \\
RTN + \DataFreeAlgorithm & & $\mathbf{I}\text{-}\mathbf{I}$ & 8e5 & 8e5 & 3e8 & 3e8 & 10.1 & 13.0 & 38.0 & 76.1 & 7.49 & 7.65 & 9.30 & 8.69 \\
RTN + \ScaleAlgorithmShort & I & $\tilde{\mathbf{X}}\text{-}\tilde{\mathbf{X}}$ & 6e4 & 5e4 & 388 & 758 & 8.38 & 8.95 & 8.98 & 9.37 & 7.35 & 7.46 & 8.08 & 8.18 \\
RTN + \ScaleAlgorithmShort & S & $\tilde{\mathbf{X}}\text{-}\tilde{\mathbf{X}}$ & 21.5 & 32.6 & 18.0 & 29.4 & 8.01 & 8.44 & 8.63 & 9.01 & 7.28 & 7.37 & 8.00 & 8.05 \\
RTN + \ScaleAlgorithmShort & I & $\mathbf{X}\text{-}\tilde{\mathbf{X}}$ & 2e4 & 2e4 & 612 & 774 & 8.21 & 8.62 & 8.85 & 9.13 & 7.33 & 7.42 & 8.09 & 8.18 \\
RTN + \ScaleAlgorithmShort & S & $\mathbf{X}\text{-}\tilde{\mathbf{X}}$ & 19.4 & 28.7 & 17.1 & 25.9 & 7.86 & 8.17 & 8.58 & 8.92 & 7.26 & 7.33 & 8.00 & 8.04 \\
\midrule
AWQ &  &  & 8e5 & 8e5 & 2e8 & 2e8 & 9.61 & 9.57 & 11.4 & 11.4 & 7.50 & 7.55 & 8.37 & 8.46 \\
\midrule
GPTQ + absmax &  &  & 30.3 & 26.4 & 34.8 & 30.4 & 8.41 & 8.29 & 8.86 & 8.91 & 7.35 & 7.38 & 8.06 & 8.03 \\
GPTQ + \DataFreeAlgorithm & & $\mathbf{I}\text{-}\mathbf{I}$ & 22.1 & 25.8 & 17.3 & 26.4 & 8.05 & 8.31 & 8.54 & 8.84 & 7.29 & 7.33 & 8.00 & \textbf{8.01} \\
GPTQ + \ScaleAlgorithmShort & I & $\tilde{\mathbf{X}}\text{-}\tilde{\mathbf{X}}$ & 28.8 & 32.0 & 16.6 & 19.6 & 8.13 & 8.53 & 8.55 & 8.79 & 7.30 & 7.37 & 8.02 & 8.03 \\
GPTQ + \ScaleAlgorithmShort & S & $\tilde{\mathbf{X}}\text{-}\tilde{\mathbf{X}}$ & 21.5 & 25.2 & 15.5 & 18.1 & 8.12 & 8.50 & 8.56 & 8.82 & 7.30 & 7.37 & 8.01 & 8.04 \\
GPTQ $\odot$ \ScaleAlgorithmShort & I & $\tilde{\mathbf{X}}\text{-}\tilde{\mathbf{X}}$ & 23.9 & 29.9 & 16.0 & 19.2 & 8.11 & 8.50 & 8.52 & 8.72 & 7.30 & 7.37 & 8.01 & 8.03 \\
GPTQ $\odot$ \ScaleAlgorithmShort & S & $\tilde{\mathbf{X}}\text{-}\tilde{\mathbf{X}}$ & 20.7 & 24.0 & 15.3 & 18.1 & 8.12 & 8.44 & 8.53 & 8.74 & 7.30 & 7.36 & 8.01 & 8.04 \\
GPTQ $\boxdot$ \ScaleAlgorithmShort & I & $\tilde{\mathbf{X}}\text{-}\tilde{\mathbf{X}}$ & 17.4 & 21.3 & 14.3 & 16.9 & 7.92 & 8.30 & 8.48 & 8.68 & 7.27 & 7.34 & \textbf{7.98} & 8.03  \\
GPTQ $\boxdot$ \ScaleAlgorithmShort & S & $\tilde{\mathbf{X}}\text{-}\tilde{\mathbf{X}}$ & 17.0 & 22.3 & 14.3 & 16.6 & 7.92 & 8.29 & 8.50 & 8.70 & 7.27 & 7.34 & \textbf{7.98} & 8.02  \\
\midrule
Qronos + absmax &  &  & 22.6 & 24.5 & 77.3 & 21.1 & 8.07 & 8.08 & 8.67 & 8.66 & 7.32 & 7.31 & 8.03 & 8.03 \\
Qronos + \DataFreeAlgorithm &  & $\mathbf{I}\text{-}\mathbf{I}$ & 15.1 & 18.1 & 17.7 & 35.9 & 7.74 & \textbf{7.91} & 8.43 & 8.64 & 7.24 & \textbf{7.27} & 7.99 & 8.02   \\
Qronos + \ScaleAlgorithmShort & I & $\mathbf{X}\text{-}\tilde{\mathbf{X}}$ & 23.6 & 28.3 & 22.0 & 23.7 & 7.75 & 7.95 & 8.44 & 8.62 & 7.25 & 7.29 & \textbf{7.98} & 8.02  \\
Qronos + \ScaleAlgorithmShort & S & $\mathbf{X}\text{-}\tilde{\mathbf{X}}$ & 14.6 & 16.3 & 13.3 & 14.7 & 7.80 & 7.99 & 8.49 & 8.66 & 7.25 & 7.30 & 8.00 & 8.03 \\
Qronos $\odot$ \ScaleAlgorithmShort & I & $\mathbf{X}\text{-}\tilde{\mathbf{X}}$ & 15.5 & 17.9 & 13.6 & 15.1 & 7.76 & 7.96 & 8.44 & 8.61 & 7.25 & 7.29 & 7.99 & 8.03 \\
Qronos $\odot$ \ScaleAlgorithmShort & S & $\mathbf{X}\text{-}\tilde{\mathbf{X}}$ & 14.8 & 16.3 & 13.0 & 14.4 & 7.82 & 8.01 & 8.48 & 8.66 & 7.26 & 7.30 & 8.00 & 8.03 \\
Qronos $\boxdot$ \ScaleAlgorithmShort & I & $\mathbf{X}\text{-}\tilde{\mathbf{X}}$ & \textbf{12.9} & 15.4 & \textbf{11.9} & \textbf{13.2} & \textbf{7.72} & 7.94 & \textbf{8.37} & \textbf{8.53} & \textbf{7.23} & 7.28 & \textbf{7.98} & \textbf{8.01}   \\
Qronos $\boxdot$ \ScaleAlgorithmShort & S & $\mathbf{X}\text{-}\tilde{\mathbf{X}}$ & 13.7 & \textbf{15.2} & 12.5 & 13.8 & 7.73 & \textbf{7.91} & 8.45 & 8.63 & 7.24 & 7.28 & 7.99 & 8.03  \\
\bottomrule
\end{tabular}
}
\end{table}

\begin{table}[!htbp]
\centering
\caption{\textbf{Average zero-shot accuracy ($\uparrow$)} results for group-wise quantization on Llama-3.2-3B and Qwen-2.5-3B with group sizes 16 and 32 (G16/G32). The ``Obj.'' column denotes the scale optimization objective (see Equation~\ref{eq:layer-obj}): $\mathbf{X}\text{-}\mathbf{X}$ (float activations), $\tilde{\mathbf{X}}\text{-}\tilde{\mathbf{X}}$ (self-activation), or $\mathbf{X}\text{-}\tilde{\mathbf{X}}$ (cross-activation); blank where not applicable (absmax, AWQ). In the ``Var.'' column, $\mathrm{I}/\mathrm{S}$ denotes independent/sequential group-wise heuristics (see Section~\ref{sec:group-wise}).}
\label{tab:group_expanded_3B_acc}
\resizebox{\textwidth}{!}{
\begin{tabular}{lll cccc cccc cccc}
\toprule
 & & & \multicolumn{4}{c}{\textbf{2-bit}} & \multicolumn{4}{c}{\textbf{3-bit}} & \multicolumn{4}{c}{\textbf{4-bit}} \\
\cmidrule(lr){4-7} \cmidrule(lr){8-11} \cmidrule(lr){12-15}
 & Var. & Obj. & \multicolumn{2}{c}{Llama 3B} & \multicolumn{2}{c}{Qwen 3B} & \multicolumn{2}{c}{Llama 3B} & \multicolumn{2}{c}{Qwen 3B} & \multicolumn{2}{c}{Llama 3B} & \multicolumn{2}{c}{Qwen 3B} \\
\cmidrule(lr){4-5} \cmidrule(lr){6-7} \cmidrule(lr){8-9} \cmidrule(lr){10-11} \cmidrule(lr){12-13} \cmidrule(lr){14-15}
 & & & G16 & G32 & G16 & G32 & G16 & G32 & G16 & G32 & G16 & G32 & G16 & G32 \\
\midrule
BF16 & & & 62.0 & 62.0 & 56.5 & 56.5 & 62.0 & 62.0 & 56.5 & 56.5 & 62.0 & 62.0 & 56.5 & 56.5 \\
\midrule
RTN + absmax &  &  & 32.3 & 32.2 & 31.1 & 31.5 & 50.7 & 52.0 & 46.2 & 31.6 & 60.3 & 59.3 & 56.2 & 55.0 \\
RTN + \DataFreeAlgorithm & & $\mathbf{I}\text{-}\mathbf{I}$ & 32.2 & 32.3 & 31.2 & 31.3 & 54.7 & 51.5 & 37.7 & 38.5 & 61.2 & 61.0 & 54.4 & 54.3 \\
RTN + \ScaleAlgorithmShort & I & $\tilde{\mathbf{X}}\text{-}\tilde{\mathbf{X}}$ & 32.5 & 32.1 & 32.6 & 32.7 & 58.2 & 57.2 & 54.3 & 54.4 & 61.2 & 60.8 & 55.5 & 55.2 \\
RTN + \ScaleAlgorithmShort & S & $\tilde{\mathbf{X}}\text{-}\tilde{\mathbf{X}}$ & 44.3 & 40.5 & 38.5 & 37.1 & \textbf{59.8} & 57.6 & 54.8 & 52.8 & 61.9 & 61.6 & 55.6 & 56.2 \\
RTN + \ScaleAlgorithmShort & I & $\mathbf{X}\text{-}\tilde{\mathbf{X}}$ & 33.1 & 33.4 & 31.3 & 32.7 & 58.9 & 58.4 & 54.6 & 54.8 & 60.7 & 60.8 & 55.8 & 54.9 \\
RTN + \ScaleAlgorithmShort & S & $\mathbf{X}\text{-}\tilde{\mathbf{X}}$ & 42.3 & 40.0 & 39.8 & 38.7 & 59.1 & 57.1 & 56.3 & 53.3 & 61.3 & 61.3 & 56.5 & 56.1 \\
\midrule
AWQ &  &  & 32.1 & 32.2 & 31.6 & 32.6 & 55.5 & 54.9 & 52.5 & 51.5 & 61.0 & 60.7 & 54.8 & 54.7 \\
\midrule
GPTQ + absmax &  &  & 39.5 & 37.9 & 34.1 & 33.1 & 57.8 & 57.6 & 53.6 & 51.7 & 60.2 & 60.5 & 56.2 & \textbf{57.1} \\
GPTQ + \DataFreeAlgorithm & & $\mathbf{I}\text{-}\mathbf{I}$ & 40.3 & 37.9 & 40.9 & 34.8 & 59.5 & 57.5 & 55.4 & 54.1 & 61.8 & \textbf{62.3} & 55.6 & 54.9 \\
GPTQ + \ScaleAlgorithmShort & I & $\tilde{\mathbf{X}}\text{-}\tilde{\mathbf{X}}$ & 37.6 & 37.6 & 41.1 & \textbf{39.2} & 57.8 & 57.7 & 54.5 & 54.6 & 61.7 & 61.1 & 56.5 & 56.6  \\
GPTQ + \ScaleAlgorithmShort & S & $\tilde{\mathbf{X}}\text{-}\tilde{\mathbf{X}}$ & 40.3 & 38.4 & 39.6 & 34.5 & 59.4 & 57.2 & 52.9 & 51.3 & 61.7 & 61.2 & 55.6 & 55.9 \\
GPTQ $\odot$ \ScaleAlgorithmShort & I & $\tilde{\mathbf{X}}\text{-}\tilde{\mathbf{X}}$ & 39.7 & 37.5 & 38.2 & 38.3 & 58.1 & 57.7 & 55.3 & 55.5 & 61.6 & 61.1 & 55.9 & 55.6 \\
GPTQ $\odot$ \ScaleAlgorithmShort & S & $\tilde{\mathbf{X}}\text{-}\tilde{\mathbf{X}}$ & 40.9 & 38.5 & 40.0 & 34.9 & 57.9 & 55.9 & 54.1 & \textbf{56.8} & 61.1 & 61.0 & 56.1 & 56.0 \\
GPTQ $\boxdot$ \ScaleAlgorithmShort & I & $\tilde{\mathbf{X}}\text{-}\tilde{\mathbf{X}}$ & 45.6 & 40.4 & 42.2 & 37.2 & 58.8 & 58.2 & 54.2 & 55.9 & \textbf{62.2} & 61.0 & 56.2 & 56.9  \\
GPTQ $\boxdot$ \ScaleAlgorithmShort & S & $\tilde{\mathbf{X}}\text{-}\tilde{\mathbf{X}}$ & 44.0 & 41.1 & 39.0 & 37.5 & 59.5 & 58.8 & 54.5 & 55.4 & \textbf{62.2} & 61.2 & 56.1 & 56.4 \\
\midrule
Qronos + absmax &  &  & 39.3 & 39.3 & 32.6 & 31.9 & 57.2 & \textbf{58.9} & 52.4 & 51.0 & 61.1 & 61.6 & 56.6 & 55.2 \\
Qronos + \DataFreeAlgorithm &  & $\mathbf{I}\text{-}\mathbf{I}$ & 44.5 & 41.5 & 31.9 & 35.5 & 58.5 & 57.0 & \textbf{56.8} & 51.3 & 61.9 & 61.5 & 56.0 & 55.8 \\
Qronos + \ScaleAlgorithmShort & I & $\mathbf{X}\text{-}\tilde{\mathbf{X}}$ & 38.1 & 36.7 & 34.2 & 33.2 & 59.3 & 58.0 & 54.9 & 54.0 & 60.2 & 60.6 & 56.3 & 56.2 \\
Qronos + \ScaleAlgorithmShort & S & $\mathbf{X}\text{-}\tilde{\mathbf{X}}$ & 43.0 & 43.0 & 40.1 & 33.5 & 58.5 & 57.5 & 53.8 & 52.9 & 61.2 & 60.9 & 56.4 & 56.5 \\
Qronos $\odot$ \ScaleAlgorithmShort & I & $\mathbf{X}\text{-}\tilde{\mathbf{X}}$ & 43.9 & 42.4 & 34.6 & 31.4 & 59.4 & 57.8 & 54.4 & 55.2 & 60.7 & 61.3 & \textbf{57.4} & 56.4 \\
Qronos $\odot$ \ScaleAlgorithmShort & S & $\mathbf{X}\text{-}\tilde{\mathbf{X}}$ & 43.8 & 41.7 & 38.3 & 31.6 & 58.5 & 57.9 & 54.5 & 51.9 & 61.7 & 60.8 & 56.8 & 55.6 \\
Qronos $\boxdot$ \ScaleAlgorithmShort & I & $\mathbf{X}\text{-}\tilde{\mathbf{X}}$ & 46.4 & \textbf{43.7} & \textbf{44.3} & 33.8 & \textbf{59.8} & 57.2 & 55.3 & 50.1 & 61.9 & 61.8 & 55.3 & 55.0  \\
Qronos $\boxdot$ \ScaleAlgorithmShort & S & $\mathbf{X}\text{-}\tilde{\mathbf{X}}$ & \textbf{46.5} & 41.9 & 41.1 & 34.7 & 58.1 & 57.9 & 52.4 & 54.1 & 61.1 & 61.7 & 55.9 & 56.3 \\
\bottomrule
\end{tabular}
}
\end{table}

\begin{table}[t]
\centering
\caption{
Group-wise quantization results on Llama 3.1-8B and Qwen-2.5-7B Instruct with group size 16, reporting WikiText2 perplexity ($\downarrow$) and average zero-shot accuracy ($\uparrow$). The ``Obj.'' column denotes the scale optimization objective (see Equation~\ref{eq:layer-obj}): $\mathbf{X}\text{-}\mathbf{X}$ (float activations), $\tilde{\mathbf{X}}\text{-}\tilde{\mathbf{X}}$ (self-activation), or $\mathbf{X}\text{-}\tilde{\mathbf{X}}$ (cross-activation); blank where not applicable (absmax, AWQ). In the ``Var.'' column, $\mathrm{I}/\mathrm{S}$ denotes independent/sequential group-wise heuristics (see Section~\ref{sec:group-wise}).}
\label{tab:group_combined_big_g16}
\resizebox{\textwidth}{!}{
\begin{tabular}{lll cc cc cc cc cc cc}
\toprule
 & & & \multicolumn{4}{c}{\textbf{2-bit}} & \multicolumn{4}{c}{\textbf{3-bit}} & \multicolumn{4}{c}{\textbf{4-bit}} \\
\cmidrule(lr){4-7} \cmidrule(lr){8-11} \cmidrule(lr){12-15}
 & & & \multicolumn{2}{c}{WikiText2 ($\downarrow$)} & \multicolumn{2}{c}{0-shot ($\uparrow$)} & \multicolumn{2}{c}{WikiText2 ($\downarrow$)} & \multicolumn{2}{c}{0-shot ($\uparrow$)} & \multicolumn{2}{c}{WikiText2 ($\downarrow$)} & \multicolumn{2}{c}{0-shot ($\uparrow$)} \\
\cmidrule(lr){4-5} \cmidrule(lr){6-7} \cmidrule(lr){8-9} \cmidrule(lr){10-11} \cmidrule(lr){12-13} \cmidrule(lr){14-15}
 & Var. & Obj. & 8B & 7B & 8B & 7B & 8B & 7B & 8B & 7B & 8B & 7B & 8B & 7B \\
\midrule
BF16 & & & 5.89 & 6.92 & 68.6 & 57.7 & 5.89 & 6.92 & 68.6 & 57.7 & 5.89 & 6.92 & 68.6 & 57.7 \\
\midrule
RTN + absmax &  &  & 2e6 & 6e6 & 31.9 & 32.3 & 2e6 & 4e6 & 33.4 & 32.9 & 6.38 & 7.34 & 67.0 & 58.3 \\
RTN + \DataFreeAlgorithm &  & $\mathbf{I}\text{-}\mathbf{I}$ & 2e6 & 6e6 & 32.1 & 32.3 & 7.67 & 8.14 & 62.7 & 54.9 & 6.17 & 7.13 & 67.9 & 57.8 \\
RTN + \ScaleAlgorithmShort & I & $\tilde{\mathbf{X}}\text{-}\tilde{\mathbf{X}}$ & 1e6 & 980 & 32.0 & 31.8 & 6.84 & 7.44 & 65.3 & 56.8 & 6.05 & 7.00 & 68.1 & 58.0 \\
RTN + \ScaleAlgorithmShort & S & $\tilde{\mathbf{X}}\text{-}\tilde{\mathbf{X}}$ & 20.4 & 12.4 & 35.1 & 50.1 & 6.60 & 7.26 & 65.3 & 56.3 & 6.02 & 6.98 & 68.4 & 57.3 \\
RTN + \ScaleAlgorithmShort & I & $\mathbf{X}\text{-}\tilde{\mathbf{X}}$ & 3e5 & 308 & 31.9 & 32.6 & 6.74 & 7.40 & 65.0 & 55.7 & 6.05 & 7.03 & 67.9 & 57.9 \\
RTN + \ScaleAlgorithmShort & S & $\mathbf{X}\text{-}\tilde{\mathbf{X}}$ & 14.6 & 10.8 & 43.7 & 50.6 & 6.49 & 7.26 & 65.1 & 55.8 & 6.00 & 6.98 & 68.2 & 57.4 \\
\midrule
AWQ &  &  & 2e6 & 6e6 & 32.0 & 32.2 & 2e6 & 3e6 & 42.8 & 32.0 & 6.29 & 7.30 & 66.6 & 58.4 \\
\midrule
GPTQ + absmax &  &  & 22.9 & 12.7 & 43.2 & 36.0 & 6.79 & 7.44 & 65.8 & 56.5 & 6.08 & 7.04 & 67.8 & \textbf{58.7} \\
GPTQ + \DataFreeAlgorithm &  & $\mathbf{I}\text{-}\mathbf{I}$ & 20.4 & 10.8 & 38.0 & 46.9 & 6.59 & 7.27 & 66.2 & 56.8 & 6.02 & 6.98 & 67.6 & 57.1 \\
GPTQ + \ScaleAlgorithmShort & I & $\tilde{\mathbf{X}}\text{-}\tilde{\mathbf{X}}$ & 20.2 & 11.3 & 45.1 & 49.6 & 6.63 & 7.25 & 66.6 & 56.4 & 6.02 & 6.98 & 67.4 & 58.2 \\
GPTQ + \ScaleAlgorithmShort & S & $\tilde{\mathbf{X}}\text{-}\tilde{\mathbf{X}}$ & 14.8 & 10.7 & 46.2 & \textbf{53.4} & 6.63 & 7.25 & 66.5 & 56.8 & 6.04 & 6.98 & 67.7 & 57.2 \\
GPTQ $\odot$ \ScaleAlgorithmShort & I & $\tilde{\mathbf{X}}\text{-}\tilde{\mathbf{X}}$ & 17.5 & 11.6 & 46.3 & 32.4 & 6.59 & 7.25 & \textbf{66.9} & 55.4 & 6.02 & 6.98 & 67.7 & 58.0 \\
GPTQ $\odot$ \ScaleAlgorithmShort & S & $\tilde{\mathbf{X}}\text{-}\tilde{\mathbf{X}}$ & 15.7 & 10.6 & 48.2 & 52.3 & 6.64 & 7.27 & 65.8 & 56.2 & 6.03 & 6.98 & \textbf{68.7} & 56.3  \\
GPTQ $\boxdot$ \ScaleAlgorithmShort & I & $\tilde{\mathbf{X}}\text{-}\tilde{\mathbf{X}}$ & 15.7 & 10.1 & 38.3 & 52.9 & 6.51 & 7.22 & 66.1 & \textbf{56.9} & 6.01 & \textbf{6.97} & 68.4 & 57.7  \\
GPTQ $\boxdot$ \ScaleAlgorithmShort & S & $\tilde{\mathbf{X}}\text{-}\tilde{\mathbf{X}}$ & 15.6 & 10.1 & 44.5 & 52.3 & 6.52 & 7.23 & 65.8 & 55.7 & 6.01 & \textbf{6.97} & 68.4 & 57.4  \\
\midrule
Qronos + absmax &  &  & 29.6 & 2e4 & 34.7 & 33.6 & 6.68 & 7.67 & 63.5 & 56.1 & 6.06 & 7.06 & 67.4 & 57.6 \\
Qronos + \DataFreeAlgorithm &  & $\mathbf{I}\text{-}\mathbf{I}$ & 12.1 & 3e3 & 43.7 & 32.3 & 6.40 & 7.26 & 66.4 & 55.2 & 5.99 & \textbf{6.97} & 68.1 & 57.3  \\
Qronos + \ScaleAlgorithmShort & I & $\mathbf{X}\text{-}\tilde{\mathbf{X}}$ & 21.4 & 17.6 & 37.1 & 37.9 & 6.42 & 7.22 & 66.0 & 56.8 & 5.99 & \textbf{6.97} & 67.4 & 57.7  \\
Qronos + \ScaleAlgorithmShort & S & $\mathbf{X}\text{-}\tilde{\mathbf{X}}$ & 11.4 & 10.5 & 48.2 & 50.7 & 6.45 & 7.23 & 65.0 & 54.7 & 6.00 & 6.98 & 68.0 & 57.3 \\
Qronos $\odot$ \ScaleAlgorithmShort & I & $\mathbf{X}\text{-}\tilde{\mathbf{X}}$ & 13.8 & 9.92 & 43.6 & 32.6 & 6.41 & 7.21 & 66.1 & 56.2 & 5.99 & \textbf{6.97} & 67.9 & 57.8  \\
Qronos $\odot$ \ScaleAlgorithmShort & S & $\mathbf{X}\text{-}\tilde{\mathbf{X}}$ & 11.9 & 10.4 & 47.8 & 52.4 & 6.46 & 7.22 & 65.2 & 55.3 & 6.00 & 6.98 & 67.0 & 57.9 \\
Qronos $\boxdot$ \ScaleAlgorithmShort & I & $\mathbf{X}\text{-}\tilde{\mathbf{X}}$ & \textbf{10.1} & \textbf{8.91} & 36.4 & 52.2 & \textbf{6.37} & \textbf{7.18} & 66.4 & \textbf{56.9} & \textbf{5.98} & \textbf{6.97} & 67.3 & 57.2    \\
Qronos $\boxdot$ \ScaleAlgorithmShort & S & $\mathbf{X}\text{-}\tilde{\mathbf{X}}$ & 10.9 & 25.5 & \textbf{48.5} & 39.6 & 6.40 & 8.76 & 65.4 & 51.0 & 5.99 & 7.39 & 67.8 & 56.5  \\
\bottomrule
\end{tabular}
}
\end{table}

\FloatBarrier

\clearpage

\end{document}